\documentclass{article}

\usepackage{microtype}
\usepackage{graphicx}
\usepackage{subfigure}
\usepackage{booktabs} 
\usepackage{url}
\usepackage[para]{footmisc}
\usepackage[utf8]{inputenc} 
\usepackage[T1]{fontenc}
\usepackage{multirow}
\usepackage{amsmath,amssymb,amsthm,dsfont}
\usepackage{bm}
\usepackage{natbib}
\usepackage{graphicx,here,comment}
\usepackage{algorithm,algorithmic}
\usepackage{ mathrsfs }

\usepackage{hyperref}

\usepackage[accepted]{icml2018}

\usepackage{array}
\usepackage{eqparbox}
\renewcommand\algorithmiccomment[1]{%
	\hfill\#\ \eqparbox{COMMENT}{#1}%
}

\usepackage{etoolbox}  
\makeatletter
\patchcmd{\algorithmic}{\addtolength{\ALC@tlm}{\leftmargin} }{\addtolength{\ALC@tlm}{\leftmargin}}{}{}
\makeatother

\newtheorem{theorem}{Theorem}

\newtheorem{lemma}{Lemma}

\theoremstyle{definition}

\newtheorem{claim}{Claim}

\newcommand{\mA}{\mathcal{A}}
\newcommand{\mB}{\mathcal{B}}
\newcommand{\mC}{\mathcal{C}}

\newcommand{\mS}{\mathcal{S}}

\newcommand{\mM}{\mathcal{M}}
\newcommand{\mX}{\mathcal{X}}

\newcommand{\mR}{\mathcal{R}}

\newcommand{\mL}{\mathcal{L}}
\newcommand{\mN}{\mathcal{N}}
\newcommand{\mU}{\mathcal{U}}

\newcommand{\mE}{\mathcal{E}}

\renewcommand{\tilde}{\widetilde}

\newcommand{\argmin}{\operatornamewithlimits{argmin}}
\newcommand{\argmaxI}{\mathop{\mathrm{argmax}}\nolimits} 

\DeclareMathOperator{\tr}{Tr}
\DeclareMathOperator{\diag}{diag}

\usepackage{color}

\makeatletter
\renewcommand*\env@matrix[1][*\c@MaxMatrixCols c]{%
	\hskip -\arraycolsep
	\let\@ifnextchar\new@ifnextchar
	\array{#1}}
\makeatother

\newcommand{\vertiii}[1]{{\left\vert\kern-0.25ex\left\vert\kern-0.25ex\left\vert #1 
		\right\vert\kern-0.25ex\right\vert\kern-0.25ex\right\vert}}

\newcommand*\conj[1]{\bar{#1}}
\newcolumntype{L}{>{\arraybackslash}m{5cm}}

\icmltitlerunning{Solving Partial Assignment Problems using Random Clique Complexes}

\begin{document}

\twocolumn[
\icmltitle{Solving Partial Assignment Problems using Random Clique Complexes}



\icmlsetsymbol{equal}{*}

\begin{icmlauthorlist}
\icmlauthor{Charu Sharma}{iith}
\icmlauthor{Deepak Nathani}{iith}
\icmlauthor{Manohar Kaul}{iith}

\end{icmlauthorlist}
\icmlaffiliation{iith}{Department of Computer Science \& Engineering, Indian Institute of Technology Hyderabad, Hyderabad, India}

\icmlcorrespondingauthor{Charu Sharma}{charusharma1991@gmail.com}
\icmlcorrespondingauthor{Deepak Nathani}{deepakn1019@gmail.com}
\icmlcorrespondingauthor{Manohar Kaul}{mkaul@iith.ac.in}

\icmlkeywords{Machine Learning, ICML}

\vskip 0.3in
]



\printAffiliationsAndNotice{}  

\begin{abstract}
We present an alternate formulation of the \emph{partial assignment problem} as matching \emph{random clique complexes}, that are higher-order analogues of random graphs, designed to provide a set of invariants that better detect higher-order structure.
The proposed method creates random clique adjacency matrices for each $k$-skeleton of the random clique complexes and matches them, taking into account each point as the affine combination of its geometric neighborhood. We justify our solution theoretically, by analyzing the runtime and storage complexity of our algorithm along with the asymptotic behavior of the quadratic assignment problem (QAP) that is associated with the underlying \emph{random clique adjacency matrices}. 
Experiments on both synthetic and real-world datasets, containing severe occlusions and distortions, provide insight into the accuracy, efficiency, and robustness of our approach. We outperform diverse matching algorithms by a significant margin.
\end{abstract}

\section{Introduction}
The \emph{assignment problem} finds an assignment, or matching, between two finite sets
$U$ and $V$, each of cardinality $n$, such that the total cost of all matched pairs is minimized. The assignment problem can also be generalized to finding matchings between more than two sets. 
This is a fundamental problem in computer science and has been motivated by a wide gamut of research areas spanning diverse areas such as structural biology~\cite{singer2011three}, protein structure comparisons in bioinformatics~\cite{zaslavskiy2009global}, and computer vision~\cite{conte2004thirty}. Computer vision especially boasts a broad range of applications that  include object matching, image registration~\cite{shen2002hammer}, stereo matching~\cite{goesele2007multi}, shape matching~\cite{petterson2009exponential,berg2005shape}, structure from motion (SfM)~\cite{szeliski2010computer}, and object detection~\cite{jiang2011linear}, to name a few.

Various assignment approaches can 
broadly be classified as those that find a bijective assignment in the form of a permutation matrix by posing the problem as a \emph{linear assignment problem (LP)} versus ones that solve a \emph{quadratic assignment problem (QAP)} via \emph{graph matching}, where each graph's nodes represent the objects and the edges encode their corresponding distances; the goal of QAP then is to find node-wise correspondences between the graphs so that the overall discrepancy between their corresponding edge-wise counterparts is minimized and the overall \emph{relational structure} is best preserved.

\emph{Partial assignment} implies that only subsets of $U$ and $V$ can actually be assigned to each other successfully. 
This phenomenon is of particular interest to applications where either objects are absent due to incomplete observations, undergo deformations, and/or the objects in question cannot clearly be disambiguated because the objects in question along with their related objects are embedded in clutter. This variant of the assignment problem is widely accepted as a formidable challenge.

\begin{figure}[t!]
	\centering
	\includegraphics[width=\columnwidth]{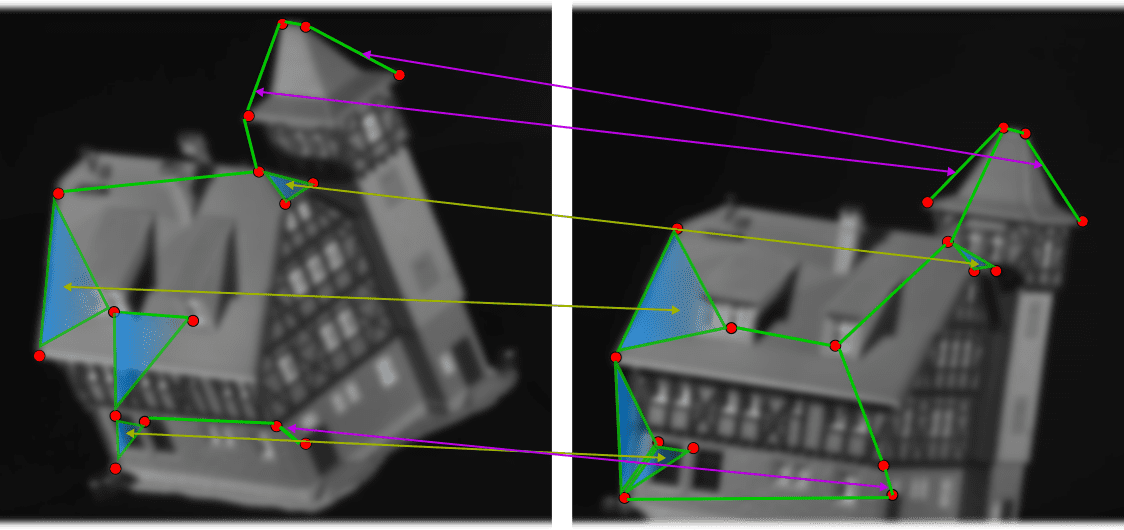}
	\caption{Matching cliques of two houses with $1$-cliques (red), $2$-cliques (green), $3$-cliques (blue). Violet and yellow lines show the matchings of $d$-cliques where $d=2,3$ respectively.}
	\label{fig:drawing}
\end{figure}

Although graph matching methods were found to be instrumental, they too perform poorly when faced with non-similar geometric transformations or transformations that produce degenerate triangulations. This is attributed to assigning weights to only node and edge assignments, while ignoring the interplay of higher-order connections/relations. For example, using triplet weights can alleviate the above mentioned problem to a very large extent by defining a measure invariant to scale and other transformations~\cite{Chertok2010}.

Motivated by the aforementioned observations and inspired by Kahle~\cite{kahle2006}'s work on combinatorial topological models like the \emph{random clique complex}, we focus our attention 
to matching higher-order components between two sets of points in the setting of some points \emph{missing completely at random}. We pose our assignment problem as finding a matching between two sets of points, each represented as a random clique complex, which is a higher-order analogue of random graphs. Figure~\ref{fig:drawing} illustrates such a matching of cliques of corresponding dimensionality, between two different scenes (taken from different camera angles) of the same house. 
Given an Erdős-Rényi (ER) graph, its \emph{clique complex} is the simplicial complex with all complete subgraphs (i.e., cliques) as its \emph{faces}. The Erdős-Rényi graph forms the $1$-skeleton of the random clique complex, where the cliques have at most a dimension of $2$, i.e., edges in the graph.
The clique topology of a random adjacency matrix is analogous to its eigenvalue spectrum, as it provides a set of invariants that help detect structure~\cite{giusti2015}. This probabilistic and combinatorial framework of random clique complexes allows us to further study the assignment problem under various assumptions of the underlying distribution of the matrix entry distributions, 
its robustness to missing values, and its asymptotic behavior for large-scale cases.

\textbf{Contributions:} We present the following contributions.
\begin{enumerate}
	\item To the best of our knowledge, our proposed approach is a first attempt to formulate 
	\emph{higher-order matching} between two sets of points, given partial or incomplete information, as a matching between two \emph{random clique complexes}. We also propose an efficient matching algorithm and study both its time and storage complexity.
	
	\item (i) We provide new bounds on the concentration inequality of eigenvalues of the QAP trace formulation for random symmetric matrices, (ii) we give tighter concentration inequality bounds on the largest eigenvalue for the Lawler QAP formulation on random matrices, in the context of \emph{affinity matrices} that are used by some earlier works.  Furthermore, we theoretically analyze and discuss the robustness of affinity-matrix based schemes to missing points, and
	(iii) we perform asymptotic analysis on the worst to best case ratio of a QAP solution for our higher-dimensional clique adjacency matrices in the \emph{clique percolation regime}~\cite{Bollobas2009}, 
	where the entries follow a \emph{Poisson} distribution.
	  
	\item Finally, we present a comprehensive empirical study that compares our method's matching accuracy to that of a diverse set of matching approaches~\cite{ZhouD16,zhou2013deformable,cho2010reweighted,feizi2016spectral,Leordeanu2005,cour2007balanced,pachauri2013solving,gold1996graduated,Kuhn1955,leordeanu2009integer,zass2008probabilistic,li2013object,duchenne2011tensor}. 
	We conducted our experiments on both synthetic and well-known hard real-world datasets that span across affine/non-affine transformations, severe occlusions, and clutter. Our study reveals much better accuracy for the popular datasets against several of the state-of-the-art matching methods.
\end{enumerate}

%
%
\section{Matching Random Clique Complexes}
\label{sec:match_rcc}
We consider the problem of capturing higher-order feature groups among landmark points in an image by representing them as a \emph{random clique complex} (RCC) and then using these RCCs to match two sets of groupings from two different images. We begin this section by describing the construction of a random clique complex, followed by our proposed method of matching two RCCs, and we finally analyze the runtime and storage complexity of our algorithm. 

\subsection{Structure of a Random Clique Complex}
\label{subsec:rcc}
We begin with general definitions pertaining to the structure of simplicial complexes and then accordingly adapt these definitions to our domain of random graphs to build random clique complexes.

Let $G(n,p)$ be an Erdős-Rényi graph with a set of $n$ vertices denoted by $V$, whose edges $\{v,v'\} \in {V \choose 2}$, are i.i.d Bernoulli($p$) distributed. Recall, that a $k$-clique in $G(n,p)$
is a \emph{complete subgraph} that comprises of $k$ vertices and ${k \choose 2}$ edges.
Here onwards, for ease of notation, we will denote $G(n,p)$ as $G$. Given any affinely independent set $V=\{v_i\}_{i=0}^k$ of $(k+1)$ points in $\mathds{R}^n$, the $k$-simplex $\sigma^{(k)}$ is the \emph{convex hull} of $V$,
i.e., it is the set of all points of the form $w_0 v_0 + \dots + w_k v_k$, where 
$\sum_{i=0}^{k} w_i = 1$ and $w_i \geq 0$ for all $i$. 
If we imagine the vertices of $G$ embedded generically in $\mathds{R}^n$, then each $(k+1)$-clique consisting of $k+1$ vertices is represented by a $k$-dimensional simplex $\sigma^{(k)}$ in our random clique complex. For example, a $2$-clique (edge) and a $3$-clique (triangle) in $G$ is represented as $\sigma^{(1)}$ and $\sigma^{(2)}$, respectively.

Given $0\leq i \leq k$, the $i$-th face $f_i$ of $\sigma^{(k)}$ is the subspace of points that satisfy $w_i=0$; it is the $(k-1)$-simplex $\sigma^{(k-1)}$ whose vertices are all those of $\sigma^{(k)}$, except the $i$-th vertex. In other words, when 
$\sigma^{(k)}$ is a clique of $G$, then all its subsets are also cliques and hence 
considered faces of $\sigma^{(k)}$. For example, a $3$-clique (triangle) has three $2$-cliques (edges) in it.

With the aforementioned definitions in mind, we define our random clique complex 
$\mX(G)$ as the set of all cliques in $G$ such that
$\mX(G) = \{ \sigma \in [n] \mid \sigma \text{ is a clique of }G  \}$. We denote a set of $(k+1)$-cliques as $\mX_k(G)$. Additionally, $\mX(G)$ also satisfies the following conditions of a simplicial complex: (i) Any face in $\mX(G)$ is also a simplex in $\mX(G)$ and (ii) the intersection of any two simplexes $\sigma_i, \sigma_j$ is a 
face (lower dimensional clique) of both $\sigma_i$ and $\sigma_j$.

The faces of $\sigma^{(k)}$ are \emph{copies} of $\sigma^{(j)}$ for $j<k$, which are glued together inductively. The $k$-skeleton of $\mX(G)$, for $k \in \mathds{N}$, is defined as the following quotient space
\begin{align*}
\mX^{(k)}(G) :=  \left( \mX^{(k-1)}(G) \cup \coprod_{\sigma: \textit{dim } \sigma=k}  \sigma^{(k)}  \right) \Bigg/ \sim
\end{align*}

where $\sim$ is the equivalence relation that identifies faces of $\sigma^{(k)}$
to the corresponding faces of $\sigma \in \mX^{(j)}(G)$ where $j<k$. 
Finally, $\mX(G)=\cup_{k=0}^\infty \mX^{(k)}(G)$.

\textbf{$k$-skeleton as adjacency matrix}: Given a random graph $G$ and its $k$-skeleton 
$\mX^{(k)}(G)$ that contains all its $(k+1)$-cliques, we follow the idea from Bollob{\'{a}}s et. al.~\cite{Bollobas2009}, to 
represent $\mX^{(k)}(G)$ as an adjacency matrix $G^{(k,l)}$ whose vertex set is the set of of all 
$(k+1)$- cliques in $G$ and in which two vertices (i.e., $(k+1)$-cliques) are \emph{adjacent} when they share a common face that has a minimum of $l$ vertices, where $k \geq 1$ and $1 \leq l \leq k$. Such an adjacency matrix is built for each $k$-skeleton and therefore $\mX(G)$ is expressed as a set of matrices $\{G^{(k,l)}  \}_{k=0}^h$, where $(k+1)$ is the dimension of the cliques.

\begin{algorithm}[tbp]
	\caption{Matching Random Clique Complexes}
	\label{alg:matching}
	\begin{flushleft}
		\textbf{Input:} $\mX(G) = \{  G^{(k,l)}   \}_{k=0}^h$ and 
		$\mX(G') = \{ G'^{(k,l)}   \}_{k=0}^h$
	\end{flushleft}
	\begin{algorithmic}[1]
		\FOR{ $k = h \dots 0$}
		\STATE Let $M, M'$ be the total number of $(k+1)$-cliques in $G^{(k,l)}$ 
		and $G'^{(k,l)}$, respectively
		\STATE $\mL := \{   c_i^{(k)}  \}_{i=0}^{M-1} $  \COMMENT {list of barycenters}
		\FOR{ $i = 0 \dots M-1$} 
		\STATE $\mN_i := \mN_i \cup \left\{  g^{(k,l)}_{(x,:)} \mid x=i,   g^{(k,l)}_{(x,y)} \neq 0   \right\}  $
		\STATE $\mN := \mN \cup \{  \mN_i \} $ \COMMENT {clique neighborhoods}
		\ENDFOR
		\FOR{ $i = 0 \dots M-1$} 
		\STATE $\alpha_i := [ \alpha_1, \dots, \alpha_{M-1}]^T$
		\STATE $\alpha := \alpha \cup \{ \alpha_i \}$  \COMMENT{affine weight vectors}
		\ENDFOR
		\STATE Repeat steps $3$--$11$ on $G'^{(k,l)}$ for $\mL',\mN'$ and $\alpha'$.
		\STATE Build cost matrix $C^{(k)}$ from weights vectors $\alpha, \alpha'$
		\STATE $X^{*}_k := $ \textbf{Kuhn-Munkres }($G^{(k,l)},G'^{(k,l)},C^{(k)} $)	
		\ENDFOR
	\end{algorithmic}
	\begin{flushleft}
		\textbf{Return:}  $ \{ X^{*}_0, \dots, X^{*}_h \}$ \text{ \# set of permutation matrices}
	\end{flushleft}
\end{algorithm}
	
\subsection{Problem Setup}
\label{subsec:setup}
The problem of matching random clique complexes each of dimension $h$, is the estimation of a set of 
optimal bijective maps of the form $\mM_i: \mX^{(i)}(G) \rightarrow \mX^{(i)}(G')$, for all $i \leq h$, subject to assignment constraints.
This can be formulated as a \emph{constrained quadratic assignment problem}, which can later be relaxed to a linear programming optimization problem.

Given two $h$-dimensional random clique complexes $\mX(G) = \{  G^{(k,l)}   \}_{k=0}^h$ and 
$\mX(G') = \{ G'^{(k,l)}   \}_{k=0}^h$, let $X=\{  X_0, \dots, X_h   \} \in \Pi$ be a set of permutation matrices
such that $X_k$ encodes assignments/matchings from $G^{(k,l)}$ to $G'^{(k,l)}$. 
The combinatorial matching requires the optimal set of permutation matrices that best \emph{align} 
$\mX(G)$ and $\mX(G')$. More formally, this can be expressed as the following constrained optimization problem
\begin{equation}
\label{eq:opt}
\begin{aligned}
& \argmin_{X_0,\dots,X_h}
& &   \sum_{k=0}^{h} \lVert  G^{(k,l)} X_k - X_k G'^{(k,l)}    \rVert_F^2 \\
& \text{subject to}
& & \forall k \leq h, \mathds{1}^T X_k = \mathds{1},  X_k^T \mathds{1} = \mathds{1}
\end{aligned}
\end{equation}

\subsection{Our Algorithm}
\label{subsec:algo}

At a high level, our goal is to minimize $\lVert \mX(G) - \mX(G') \rVert_{\mC}$, where $\mC$ is a combinatorial distance between two random clique complexes. Traditional metrics like Hausdorff distance are not suitable here because random clique complexes are combinatorial topological spaces.
Recall that $\mX(G)$ is comprised of a family of $k$-skeletons $\{  \mX^{(k)}(G)  \}_{k=0}^h$, where each $k$-skeleton contains cliques whose dimension is at most $k+1$ and $\mX^{(k)}(G)$ has a maximum dimension $h$.
The solution of the optimization problem outlined in Equation~(\ref{eq:opt}) aims to find a set of permutation matrices $\{X_1, \dots, X_h \}$ that minimizes the overall number of \emph{misalignments} between equi-dimensional faces of $\mX(G)$ and $\mX(G')$, i.e., cliques belonging to the corresponding $k$-skeletons, and thus producing the optimal least cost assignment between $\mX(G)$ and $\mX(G')$. 

Algorithm~\ref{alg:matching} presents our method to solve the combinatorial optimization problem (Equation~(\ref{eq:opt})). In decreasing order of clique dimensionality, for a fixed dimension $k$ and given the adjacency matrices $G^{(k,l)}$ and $G'^{(k,l)}$ for $k$-skeletons $\mX^{(k)}(G)$ and $\mX^{(k)}(G')$, respectively. In every iteration, our objective is to solve $\argmin_{X_k} \lVert G^{(k,l)}X_k - X_k G'^{(k,l)}  \rVert_F^2$ to find the optimal permutation $X^{*}_k$.
We assume the barycenters of every clique is pre-computed (Step~$3$). 
Next, the \emph{neighborhood} $\mN_i$ of the $i$-th clique is computed as the set of entries with $1$s in the $i$-th row of $G^{(k,l)}$  (Step~$5$). We denote the collection of every clique's neighborhood as $\mN$ (Step~$6$). 
An important objective of our method is to capture the \emph{geometric properties} of the neighborhood of every clique. We achieve this by characterizing the $i$-th clique's barycenter $c_i^{(k)}$ as an \emph{affine combination} of the barycenters (in all dimensions) associated with the cliques in its corresponding neighborhood $\mN_i$.
Given an arbitrary clique's barycenter $c_i^{(k)}$, let $\{ x_1^{(k)}  , \dots,  x_n^{(k)} \}$ denote the barycenters of its $n$ adjacent cliques. Then, $c_i^{(k)}$ expressed as $\sum_{i=1}^{n} \alpha_i x_i^{(k)}$ is an affine combination of the $x_i^{(k)}$s, if $\sum_{i=0}^{n} \alpha_i = 1$, i.e., the weights $\alpha_i$ sum to $1$.
Among all possible affine representations of $c_i^{(k)}$ we chose to use \emph{least squares} to guarantee minimal error under L$2$-norm, and furthermore it assigns non-zero weights to each of its adjacent clique barycenters, thereby capturing the local geometric properties in its neighborhood. 
The weight vector $\alpha_i$ is then calculated for each clique  (Step~$9$) and $\alpha$ denotes a collection of such weight vectors (Step~$10$). Next, a cost matrix is built by computing the L$2$-norm distance between weight vectors $\alpha$ and $\alpha'$ (Step~$13$). Finally, the \emph{Kuhn-Munkres}~\cite{Kuhn1955} algorithm is invoked with both the adjacency matrices and the cost matrix, which arrives at the optimal assignment (Step~$14$ ). At the end of all iterations, our method returns a set of optimal assignments for matches between each $k$-skeleton for every dimension below $h$ and the algorithm terminates. We refer the reader to our supplementary section for a working example.

\subsection{Complexity Analysis}
\label{subsec:complexity}

To begin our analysis, we must first ascertain the dimensionality of $G^{(k,l)}$, which is governed by the total number of $(k+1)$-cliques that exist in the underlying random graph $G$. It is important to note that there doesn't exist any closed form solution to counting the number of cliques of a given dimension in $G$.

We consider the distribution of a random variable $X_n(k)$ counting the number of $(k+1)$-cliques in a realization of $G$. We show in Appendix A of our supplementary material that this count is upper bounded by 
$(en/k)^k$, where $e$ is \emph{Euler's number}. 
This can be expressed in asymptotic notation as $O(n^k)$.
As dimensionality increases, there occurs an explosion in the number of cliques.
Fortunately, $G^{(k,l)}$ is a sparse matrix and its \emph{effective} dimensionality measured by the number of non-zero rows, i.e., the number of cliques with non-empty neighborhoods, is of order $O(nnz(G^{(k,l)}))$. Therefore, we set out to count the number of non-zero entries in $G^{(k,l)}$.

We use a seminal result by Bollob{\'{a}}s~\cite{Bollobas2009}, where they identify a threshold probability for \emph{percolation} of cliques in $G$ for all fixed $k$ and $l$, which is given by 
$p = \Theta \left(n ^{\frac{-2}{k+l-1}} \right)$. Moreover, they proved that for $p$ around this threshold, the number of cliques asymptotically converge to a \emph{Poisson} distribution. Exceeding this threshold results in formation of \emph{giant connected clique clusters}, which  causes an explosion in the number of possible cliques. 

Recall from our definition of $G^{(k,l)}$, that two cliques are adjacent if they share at least $l$ vertices. In order to analyze this further, we imagine an entry in $G^{(k,l)}$ occurs when we can migrate a $(k+1)$-clique from its original position to an adjacent clique by relocating exactly $(k+1-l)$ vertices and leaving the remaining $l$ vertices intact. The expected number of such relocations is given by $\left( {k+1 \choose l} -1 \right)  {n \choose k+1-l}  p^{  \left(  {k+1 \choose 2} - {l \choose 2}  \right)}$,
where the first term denotes the number of possible vertices in a $(k+1)$-clique 
that can be chosen for relocation, the second term counts the number of new adjacent positions a clique can relocate to, and the final term decides the probability of relocations that are correct and acceptable.
In our case, we define cliques to be adjacent to one another when they share at least $l=k$ vertices. This is done in order to keep the number of adjacent cliques to a manageable size during experiments.
Setting $l=k$, gives $kn p^k$ expected relocations, which in turn estimates $nnz(G^{(k,l)})$.

Note that for every iteration in Algorithm~\ref{alg:matching}, the dominating cost is that of running the Kuhn-Munkres matching algorithm in Step~$14$, which has a \emph{cubic} cost in $nnz(G^{(k,l)})$.
Let $\mC_{max}$ denote an upper bound on all the number of non-zero entries in $ \{  G^{(k,l)}  \}_{i=0}^h $.
Then, every iteration has a runtime $O(\mC_{max}^3)$ and therefore after $h$ iterations the final cost is $O(h\mC_{max}^3)$. Observe that as the dimensionality of the cliques increases in every iteration, 
$p^k$ decays very sharply and hence drastically reduces $nnz(G^{(k,l)})$, which in turn reduces the overall matching cost. Finally, the storage complexity can simply be given as $O(h\mC_{max})$.

\section{Theoretical Analysis of QAP}
\label{sec:theory}
In this section, we present three related results in the context of matching random matrices, namely: (i) concentration inequality of eigenvalue bounds on the \emph{QAP trace formulation} for random symmetric matrices, (ii) tighter concentration inequality of eigenvalue bounds on the \emph{Lawler QAP formulation} on random symmetric matrices in the context of works that use \emph{affinity matrices}, and (iii) provide an asymptotic analysis on the worst to best case ratio of a QAP for higher-dimensional clique adjacency matrices. For ease of notation, we will refer to the random clique adjacency matrices simply as $A$ and $B$.

\subsection{Eigenvalue Bounds of Trace QAP Formulation on Random Matrices}

Let $A=(a_{vv'})$, $B=(b_{vv'}) \in \mathds{R}^{n \times n}$ be random real-symmetric matrices. Let $X=(x_{ij}) \in \mathds{R}^{n \times n}$ be a \emph{permutation matrix}. 
Then, the trace formulation of a QAP is given by
\[
\text{minimize } \tr(AXBX^T) \text{ s.t. } X \in \Pi_X
\]
where $\Pi_X$ is the set of permutation matrices.

Let 
$\lambda_1(A) \leq \lambda_2(A) \leq \dots \leq \lambda_n(A)$ and 
$\lambda_1(B) \geq \lambda_2(B) \geq \dots \geq \lambda_n(B)$\footnote{The two sets of eigenvalues differ in ordering.}
be the eigenvalues of $A$ and $B$, respectively. 
Let the corresponding eigen-decompositions of matrices $A$ and $B$, be given by 
$A=Q_A \Lambda_A Q_A^T$ and $B=Q_B \Lambda_B Q_B^T$, where 
$\Lambda_A = \diag(\lambda_1(A),\dots,\lambda_n(A))$ and 
$\Lambda_B = \diag(\lambda_1(B),\dots,\lambda_n(B))$ with their corresponding 
orthogonal eigenvector matrices $Q_A$ and $Q_B$.
Finke et. al.~\cite{martello1987} gave the following eigenvalue bounds.
	\begin{theorem}
	\label{thm:1}
	Let $A$ and $B$ be symmetric matrices. Then for all $X \in X_\Pi$, 
	\begin{enumerate}
		\item 	$\tr(AXBX^T) = \lambda(A)^T Q^{(X)} \lambda(B)$, \\
		where $Q^{(X)} = \langle Q_A^{(i)} , X Q_B^{(j)} \rangle^2 $ with 
		vectors of eigenvalues given by $\lambda(A) = (\lambda_i(A))$ and $\lambda(B) = (\lambda_i(B))$. $Q_A^{(i)}$ and $Q_B^{(i)}$ denote the $i$-th eigenvectors of
		$A$ and $B$, respectively;
		\item 
		$ \mL \leq \tr(AXBX^T) \leq \mU$, where \\
		$\mL=\sum_{i=1}^{n} \lambda_i(A) \lambda_i(B)$, and\\
		$\mU=\sum_{i=1}^{n} \lambda_{n-i+1}(A) \lambda_i(B)$
	\end{enumerate}
\end{theorem}

It was also noticed by Finke~\cite{martello1987} that these bounds can further be 
tightened by reducing the \emph{spreads} of matrices $A$ and $B$, where the spread of a matrix $A$, denoted by $\mathfrak{S}(A)$, is given by $\mathfrak{S}(A) = \max_i {\lambda_i(A)} - \min_i {\lambda_i(A)}$.
There is no formula to compute the spread of a matrix directly, 
so Finke et. al.~\cite{martello1987} suggested a 
\emph{reduction} method to further sharpen the bound by replacing matrices $A$ and $B$ by \emph{smaller spread} symmetric matrices $\tilde{A}$ and $\tilde{B}$.
The reductions are achieved as
$\tilde{A} = A - M_A - M_A^T - \mathfrak{D}_A$ and 
$\tilde{B} = B - M_B - M_B^T - \mathfrak{D}_B$  	
, where $M_A$, $M_B$ are matrices with constant columns and 
$\mathfrak{D}_A$, $\mathfrak{D}_B$ are diagonal matrices, whose values 
are chosen appropriately in order to tighten the bounds on spreads $\mathfrak{S}(\tilde{A})$
and $\mathfrak{S}(\tilde{B})$.

\textbf{Our bounds on Random Matrices}: We propose new measure concentration 
inequalities on the spread of a random matrix, by redefining the spread in an alternate fashion that is more amenable to our analysis. 
Consider our reduced random symmetric matrix $\tilde{A} \in \mathds{R}^{n \times n}$, with eigenvalues 
$\lambda_1(\tilde{A}) \leq \dots \leq \lambda_n(\tilde{A})$, we define the \emph{gap} (spacing) between its consecutive eigenvalues as $\delta_i(\tilde{A}) := |\lambda_{i+1}(\tilde{A}) - \lambda_i(\tilde{A})|$ for $1 \leq i \leq n-1$. Then, the spread $\mathfrak{S}(\tilde{A})$ for the reduced matrix $\tilde{A}$ can be redefined as:
$\mathfrak{S}(\tilde{A}) = \sum_{i=1}^{n} \delta_i(\tilde{A})$

We begin by upper bounding $\delta_i(\tilde{A})$ using the following lemma~\ref{lemma:2} (proof in supplementary notes).
\begin{lemma}
	\label{lemma:2}
	Let $\vertiii{.}$ denote an algebraic matrix norm on a space of real $n \times n$ matrices $\mM_n$, then for any $A \in \mM_n$, 
	\begin{align*}
	\delta_i(A) \leq 2 \vertiii{A} 
	\end{align*}
\end{lemma}

To the best of the author's knowledge there does not exist a known distribution of eigenvalue gaps for a symmetric random matrix. We now attempt to give concentration inequalities for the tail probabilities of the sum of eigenvalue gaps, i.e., the spread.
For our i.i.d. random matrix $A \in \mathds{R}^{n \times n}$, 
consider the sequence of independent eigenvalue gaps 
$\delta_1(A), \dots ,\delta_n(A)$, where each $\delta_i(A)$ is upper bounded by $2 \vertiii{A}$,
as shown in Lemma~\ref{lemma:2}. Let us denote their sum as 
$\mS_n(A) := \delta_1(A) + \dots + \delta_n(A)$.
As $\delta_1(A), \dots ,\delta_n(A)$ are independent scalar random variables with 
$\delta_i(A) \leq \vertiii{A}$ a.s, with mean $\mu_i(A)$ and variance
$\sigma_i^2(A)$. Then, using Chernoff bounds, for any $ \epsilon>0$, we have
\begin{align*}
\mathbb{P}(| \mS_n(A) - \mu | \geq \epsilon \sigma) 
\leq K  \max \left( e^{(-p \epsilon^2  ) } , e^{  ( -p \epsilon \sigma / 2 \vertiii{A} ) }  \right)
\end{align*} 
for some absolute constants $K,p >0$.
The Chernoff inequality above, shows that $\mS_n(A)$ is sharply concentrated in the range
$n \mu + O(\sigma\sqrt{n})$, when $\epsilon$ is not too large.

\subsection{Eigenvalue bounds on Lawler's QAP on Random Affinity Matrices}
In literature, many graph matching algorithms use Lawler's QAP formulation. 
Recall, $A=(a_{ij}) ,B=(b_{uv}) \in \mathds{R}^{n \times n}$. 
Let $\Omega(a_{i,j}, b_{u,v})$ denote the \emph{pairwise affinity score} of 
assigning the $(i,j)$-th entry in $A$ to the $(u,v)$-th entry in $B$, implying that 
node $a_i$ is matched to node $b_u$ and node $a_j$ to node $b_v$, simultaneously.
Then, the \emph{affinity matrix} $\mathscr{A} \in \mathds{R}^{n^2 \times n^2}$ is given by 
$\mathscr{A}[(i-1)n + u, ( j-1  )n + v]=  \Omega(a_{i,j}, b_{u,v})$ and the optimal assignment to Lawler's QAP is the one that maximizes the sum total pairwise affinity scores. 
Leordeanu et. al.~\cite{Leordeanu2005} show via a spectral relaxation 
that Lawler's WAP reduces to solving 
$w^* = \argmaxI_w \frac{w^T Aw}{w^Tw} $, $w \in \mathds{R}^{n^2}$. This is solved by finding the leading eigenvalue $\lambda_1(\mathscr{A} )$.

As illustrated in~\cite{Alon2002}, we also use \emph{Talagrand's concentration inequality}~\cite{Talagrand1995}. We provide a tighter bound in the case of our affinity matrix using \emph{Rayleigh's quotient}. 
\begin{theorem}
	\label{thm:bound}
	For a random affinity matrix $\mathscr{A} \in R^{m \times m}$ and for a positive constant $t$,
	$\mathbb{P}[ |\lambda_1( \mathscr{A}) - \mM| \geq t] \leq 4e^{-t^2/8}$,
	where $\mM$ is the median of $\lambda_1( \mathscr{A})$.\qed
\end{theorem}

\textbf{Discussion:} We further investigate the robustness of affinity-matrix based graph matching solutions when dealing with \emph{missing} or \emph{incomplete} data. We show the sharpness of our result in Theorem~\ref{thm:bound} on the affinity matrix, similar to~\cite{Alon2002}, who analyze their results using \emph{fat} matrices as an example. 
Consider an affinity matrix $\mathscr{A}=(a_{ij}) \in \mR^{m \times m}$, whose entries are i.i.d. 
Bernoulli distributed. Simulating missing affinity scores due to missing edge assignments, we set $a_{ij}=1$ with probability $1/4$ and $a_{ij}=0$ with probability $3/4$. Notice that our random affinity matrix represents the Erdős-Rényi graph $G(m,1/4)$.  
As shown in~\cite{Alon2002}, the median and expectation of $\lambda_1(A)$ differ by a constant factor. Let $G=(V,E)$ denote a general undirected graph, where the degree of each vertex $v \in V$ is given by $d_G: V \rightarrow \mathds{Z}$. Then, the \emph{average degree} of $G$ is given by $\bar{d} = \sum_{v \in V} d_G(v) / |V|$ and its \emph{maximum degree} is 
$\Delta = \max_{v \in V} d_G(v)$. It is then well known that $\bar{d} \leq \lambda_1(A) \leq \Delta$, i.e., the largest eigenvalue of a graph is squeezed between its average and maximum degree. 

Let $|E|$ denote the total number of edges in $G(m,1/4)$, then the \emph{average degree} of $G(m,1/4)$ is given by $2 |E|/n$, where $ |E| = (Bin( {m \choose 2}, 1/4 )$ and the standard deviation of $|E|$ is $ \sqrt{ {m \choose 2} (1/4) (3/4) } = \Theta(m)$. 
For large $m$, our binomial distribution converges to a normal distribution. Therefore, we calculate the probability for the total number of edges $|E|$ to deviate from its expectation by $t$ standard deviations as $e^{-\Theta(t^2)}$. Furthermore, we know that if $|E|$ exceeds its expectation by $\Theta(tn)$, then the average degree $\bar{d}$ must also correspondingly exceed its expectation by $\Theta(t)$. Therefore, the probability of the average degree $\bar{d}$ exceeding its expectation by $t$ standard deviations is \emph{at least} $e^{- \Theta(t^2)}$. Given that $\bar{d} \leq \lambda_1(A)$, it follows that $\lambda_1$ exceeding its expectation is also lower bounded by the same $e^{- \Theta(t^2)}$. The bounds achieved are \emph{tight up to a constant factor} in the exponent. Our experimental results on \emph{Factorized Graph Matching (FGM)} by Zhou et. al.~\cite{ZhouD16} and \emph{Re-weighted Random Walk Matching (RRWM)}~\cite{cho2010reweighted}  also support the finding that affinity matrix based matching solutions are more robust to missing edges due to occlusions in data. 

\subsection{Asymptotic Analysis of Higher-Order Clique Assignment}	
Following along the same lines as Finke et. al.~\cite{martello1987}, we study the asymptotic behavior of the worst to most optimal ratio and present it as the following theorem. 
\begin{theorem}
	Given random clique adjacency matrices $A^{k,l}_n, B^{k,l}_n \sim Pois(\lambda) $ and their associated cost matrix $\mathcal{C}_{vv'} \sim Pois(\lambda)$. We denote by $\lambda_e := \mathbb{E}(\mathcal{C}_{vv'}) $ and 
	$\lambda_v := \mathbb{V}(\mathcal{C}_{vv'})$ the expectation and variance of our Poisson distributed cost function. 
	For $\epsilon>0$ and $p = \Theta \left( n^{\frac{-2}{(k+l-1)}} \right)$, we have the following bound on the ratio of the worse to the best solution as
	\begin{align*}
	&\mathbb{P}\left\lbrace 
	\frac{\max\limits_{\pi \in \Pi}  \sum\limits_{vv'} \mathcal{C}_{vv'}}  
	{\min\limits_{\pi \in \Pi}  \sum\limits_{vv'} \mathcal{C}_{vv'}}
	\leq 1 + \epsilon	 
	\right\rbrace \\
	&\geq 1 - 
	2|\Pi|\exp \left( -2|S_\pi| \left( \frac{\epsilon'\sqrt{\lambda_v}}
	{\epsilon' + 2\lambda_v}   \right)^2 \right) =: \psi(n,\epsilon)
	\end{align*}
	where, $|\Pi| = n!$, $|S_\pi|= {{n+1} \choose 2}$, $\lim_{n \to \infty} \psi(n,\epsilon) = 1$\qed
\end{theorem}

\begin{figure*}[t!]
	\centering
	\subfigure[]{%
		\label{fig:zero11}%
		\includegraphics[width=27mm]{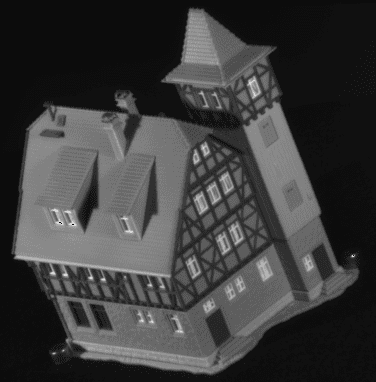}}%
	\qquad
	\subfigure[]{%
		\label{fig:first11}%
		\includegraphics[width=27mm]{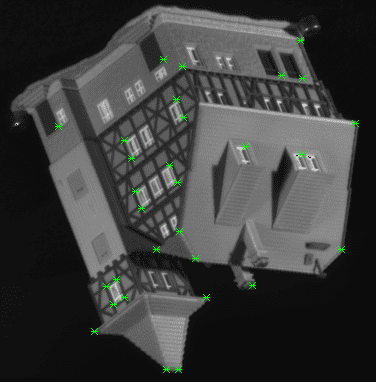}}%
	\qquad
	\subfigure[]{%
		\label{fig:second11}%
		\includegraphics[width=27mm]{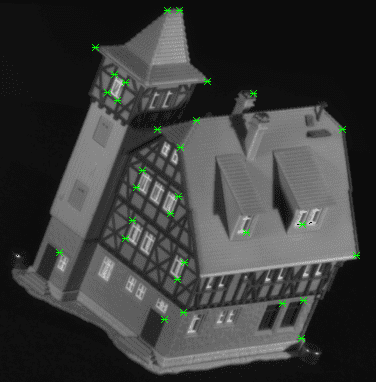}}%
	\qquad
	\subfigure[]{%
		\label{fig:third11}%
		\includegraphics[width=27mm,height=27mm]{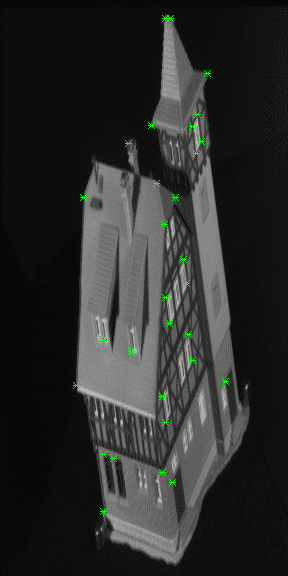}}%
	\qquad
	\subfigure[]{%
		\label{fig:fourth11}%
		\includegraphics[width=27mm,height=27mm]{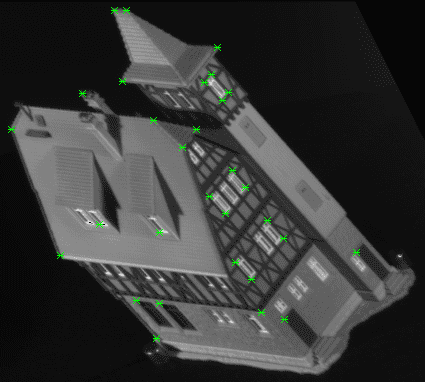}}%
	\vspace{-5mm}
	\caption{(a) Original House Frame, (b)--(e) Four transformations on house frame: Rotation, Reflection, Scaling and Shear (green markers show true matching case with the original frame (a)).}
	\label{fig:HouseTransform}
\end{figure*}
\begin{table*}[h]
	{\tiny
		\hfill{}
		\begin{tabular}{l||l|l|l|l|l|l|l|l|l|l}
			\hline
			\multicolumn{1}{l||}{\textbf{Algorithms}} & \multicolumn{2}{|l|}{\textbf{$20^{\circ}$} Rotation} & \multicolumn{2}{|l|}{\textbf{$60^{\circ}$} Rotation}&  \multicolumn{2}{|l|}{\textbf{Reflection}} & \multicolumn{2}{|l|}{\textbf{Scaling}} & \multicolumn{2}{|l}{\textbf{Shear}}\\
			
			\cline{1-11}
			& 20\% & 40\% & 20\% & 40\% & 20\% & 40\% & 20\% & 40\% & 20\% & 40\% \\
			\cline{2-11}
			\hline
			\hline
			OurMethod & \textbf{0.01 $\pm$ 0.0} & \textbf{0.01 $\pm$ 0.0} & \textbf{0.05 $\pm$ 0.0} & \textbf{0.03 $\pm$ 0.0} & \textbf{0.0 $\pm$ 0.0} & \textbf{0.0 $\pm$ 0.0} & \textbf{0.1 $\pm$ 0.0} & \textbf{0.2 $\pm$ 0.1} & \textbf{0.0 $\pm$ 0.0} & \textbf{0.0 $\pm$ 0.0}\\
			EigenAlign & 63.97 $\pm$ 1.0 & 70.19 $\pm$ 1.0 & 65.39 $\pm$ 1.5 & 70.88 $\pm$ 1.9 & 62.5 $\pm$ 0.4 & 64.97 $\pm$ 0.2 & 62.37 $\pm$ 0.6 & 65.42 $\pm$ 1.1 & 61.04 $\pm$ 0.2 & 62.36 $\pm$ 0.7\\
			FGM & 3.6 $\pm$ 0.5 & 7.4 $\pm$ 0.5 & 18.0 $\pm$ 0.0 & 36.4 $\pm$ 0.5 & \textbf{0.0 $\pm$ 0.0} & \textbf{0.0 $\pm$ 0.0} & 2.20 $\pm$ 1.0 & 3.4 $\pm$ 0.5 & \textbf{0.0 $\pm$ 0.0} & \textbf{0.0 $\pm$ 0.0}\\
			LAI-LP & 40.90 $\pm$ 0.9 & 43.07 $\pm$ 1.2 & 49.9 $\pm$ 1.2 & 61.16 $\pm$ 0.7 & 49.07 $\pm$ 0.9 & 59.41 $\pm$ 0.8 & 43.12 $\pm$ 5.5 & 45.91 $\pm$ 3.2 & 39.04 $\pm$ 0.7 & 38.27 $\pm$ 0.6\\
			PermSync & 13.36 $\pm$ 1.1 & 16.48 $\pm$ 0.4 & 25.48 $\pm$ 1.4 & 41.59 $\pm$ 0.7 & 26.24 $\pm$ 3.4 & 42.31 $\pm$ 5.1 & 12.15 $\pm$ 0.7 & 13.7 $\pm$ 0.6 & 10.19 $\pm$ 0.8 & 6.23 $\pm$ 3.2\\
			RRWM & 2.0 $\pm$ 0.0 & 4.0 $\pm$ 0.0 & 14.0 $\pm$ 0.0 & 27.0 $\pm$ 0.0 & \textbf{0.0 $\pm$ 0.0} & \textbf{0.0 $\pm$ 0.0} & 5.0 $\pm$ 2.0 & 10.0 $\pm$ 1.5 & \textbf{0.0 $\pm$ 0.0} & \textbf{0.0 $\pm$ 0.0}\\
			Tensor & 6.88 $\pm$ 0.0 & 14.39 $\pm$ 0.8 & 19.24 $\pm$ 0.6 & 37.95 $\pm$ 0.8 & 17.1 $\pm$ 0.5 & 34.69 $\pm$ 0.4 & 2.91 $\pm$ 0.6 & 6.19 $\pm$ 0.7 & \textbf{0.0 $\pm$ 0.0} & \textbf{0.0 $\pm$ 0.0}\\
			IPFP & 6.8 $\pm$ 1.5 & 11.6 $\pm$ 0.5 & 18.2 $\pm$ 0.5 & 35.0 $\pm$ 0.0 & 1.0 $\pm$ 0.0 & 1.0 $\pm$ 0.0 & 7.4 $\pm$ 2.0 & 15.6 $\pm$ 1.5 & 0.8 $\pm$ 0.5 & 0.6 $\pm$ 0.5\\
			PM & 43.2 $\pm$ 1.0 & 49.4 $\pm$ 1.5 & 46.4 $\pm$ 1.0 & 56.8 $\pm$ 0.5 & 37.0 $\pm$ 0.0 & 37.0 $\pm$ 0.0 & 45.6 $\pm$ 0.5 & 53.4 $\pm$ 0.5 & 37.8 $\pm$ 0.5 & 38.4 $\pm$ 0.5\\
			SMAC & 12.6 $\pm$ 0.5 & 20.0 $\pm$ 0.0 & 19.0 $\pm$ 0.0 & 33.0 $\pm$ 0.0 & 4.0 $\pm$ 0.0 & 4.0 $\pm$ 0.0 & 13.8 $\pm$ 2.5 & 23.4 $\pm$ 0.5 & 4.0 $\pm$ 0.0 & 4.2 $\pm$ 0.5\\
			SM & 32.6 $\pm$ 0.5 & 39.4 $\pm$ 1.5 & 37.8 $\pm$ 1.0 & 49.6 $\pm$ 1.0 & 25.0 $\pm$ 0.0 & 25.0 $\pm$ 0.0 & 33.0 $\pm$ 1.5 & 40.4 $\pm$ 0.5 & 25.2 $\pm$ 0.5 & 24.8 $\pm$ 0.5\\
			GA & 34.0 $\pm$ 1.5 & 37.0 $\pm$ 1.0 & 38.4 $\pm$ 1.0 & 47.0 $\pm$ 0.0 & 30.0 $\pm$ 0.0 & 30.0 $\pm$ 0.0 & 37.6 $\pm$ 1.0 & 45.4 $\pm$ 0.5 & 29.0 $\pm$ 0.0 & 28.4 $\pm$ 0.5\\
			Munkres & 35.25 $\pm$ 1.5 & 36.47 $\pm$ 0.9 & 44.75 $\pm$ 0.8 & 55.13 $\pm$ 1.3 & 46.05 $\pm$ 1.2 & 57.25 $\pm$ 2.2 & 37.32 $\pm$ 1.5 & 37.81 $\pm$ 1.8 & 32.84 $\pm$ 0.3 & 31.27 $\pm$ 0.8\\			
			\hline
	\end{tabular}}
	\hfill{}
	\caption{Error (\%) of transformation on CMU House: inserted $20\%$ and $40\%$ impurity in CMU House frame sequence randomly for rotation ($20^{\circ}$, $60^{\circ}$), reflection, scaling and shear. Minimum error (\%) is shown in bold. Matching is computed for $111$ frames from the $1^{st}$ frame to the other $110$ frames. Our method shows best performance among all the methods.}
	\label{tb:Transformation}
\end{table*}

\section{Experiments}
\label{sec:exp}
Here, we study the robustness of various matching algorithms when affected by missing or incomplete information and transformations (both affine and non-affine) on synthetic and real-world datasets. For the sake of brevity, we report detailed dataset descriptions in our supplementary notes.
The graph matching algorithms can broadly be classified based on their use of
(i) \emph{affinity-matrix:} FGM~\cite{ZhouD16,zhou2013deformable}\footnote{\href{http://www.f-zhou.com/gm\_code.html}{FGM}}, RRWM~\cite{cho2010reweighted}, 
(ii) \emph{Eigenvalues:} EigenAlign~\cite{feizi2016spectral}\footnote{\href{https://github.com/SoheilFeizi/spectral-graph-alignment}{EigenAlign}}, SM~\cite{Leordeanu2005}, SMAC~\cite{cour2007balanced}, PermSync~\cite{pachauri2013solving}\footnote{\href{http://pages.cs.wisc.edu/~pachauri/perm-sync/}{PermSync}},
(iii) \emph{LP relaxation:} GA~\cite{gold1996graduated}, Kuhn-Munkres~\cite{Kuhn1955}, 
(iv) \emph{Integer QAP:} IPFP~\cite{leordeanu2009integer}, 
(v) \emph{Probabilistic matching:} PM~\cite{zass2008probabilistic}, 
(vi) \emph{Higher-order matching given complete data:} Tensor~\cite{duchenne2011tensor}\footnote{\href{http://www.cs.cmu.edu/~olivierd/}{Tensor}},
and (vii) \emph{Geometric and Feature matching:} LAI-LP~\cite{li2013object}\footnote{This algorithm serves as our \emph{naive baseline} as it directly uses neighborhood properties of the underlying graph (\href{http://www.ee.cuhk.edu.hk/~hsli/}{LAI-LP}).}. Our code\footnote{\href{https://github.com/charusharma1991/RandomCliqueComplexes_ICML2018}{Our Method}} is publicly available.

\subsection{Effect of Affine Transformations}
\textbf{Simulated Dataset:} We perform affine transformations on \href{http://vasc.ri.cmu.edu/idb/html/motion/house/index.html}{CMU House}, which is a sequence of $N$ frames extracted from a video. More specifically, we uniformly sample frames (at $20\%$ and $40\%$) and perform affine transformations on the selected frames to distort them. Figure~\ref{fig:HouseTransform} shows examples of affine transformations on house frame sequences. Table \ref{tb:Transformation} shows the comparative error in matching for all the algorithms. We now describe each affine transformation as performance metrics in our experiments.
\begin{center}
	\begin{table*}
		{\tiny
			\hfill{}
			\begin{tabular}{l||l|l|l|l|l|l}
				\hline
				\textbf{Methods} & \textbf{Car} & \textbf{Bike} & \textbf{Butterfly} & \textbf{Magazine} & \textbf{Building} & \textbf{Book}\\
				\hline
				\hline
				OurMethod & \textbf{4.14 $\pm$ 2.45}/ 7.13 & \textbf{3.15 $\pm$ 0.32}/ 6.96 & \textbf{3.89 $\pm$ 0.23}/ 14.76 & \textbf{0.48 $\pm$ 0.02}/ 43.99 & \textbf{4.17 $\pm$ 0.32}/ 12.65 & \textbf{22.20 $\pm$ 1.16}/ 14.86\\
				EigenAlign & 60.68 $\pm$ 0.29/ 19.37 & 57.44 $\pm$ 0.37/ 19.32 & 66.57 $\pm$ 0.0/ 26.13 & 43.23 $\pm$ 0.0/ 93.60 & 90.51 $\pm$ 0.0/ 2.64 & 98.41 $\pm$ 0.0/ 8.29\\
				FGM & 55.51 $\pm$ 0.0/ 1793.9 & 48.17 $\pm$ 0.0/ 2013.7 & 16.12 $\pm$ 0.0/ 674.94 & \textbf{0.0 $\pm$ 0.0}/ 777.55 & 74.87 $\pm$ 0.05/ 2530.5 & 97.54 $\pm$ 0.01/ 4293.9\\
				LAI-LP & 73.06 $\pm$ 0.23/ 152.47 & 42.00 $\pm$ 0.24/ 154.15 & 49.54 $\pm$ 0.0/ 161.06 & 88.73 $\pm$ 0.1/ 184.153 & 87.98 $\pm$ 0.0/ 33.06 & 96.38 $\pm$ 0.0/ 14.71\\
				PermSync & 10.63 $\pm$ 0.0/ 0.45 & 8.90 $\pm$ 0.0/ 0.46 & 46.93 $\pm$ 0.0/ 0.43 & 79.88 $\pm$ 0.0/ 1.08 & 64.00 $\pm$ 0.0/ 0.22 & 70.00 $\pm$ 0.0/ 0.48\\
				RRWM & 60.91 $\pm$ 0.0/ 4.96 & 54.53 $\pm$ 0.0/ 4.83 & 30.99 $\pm$ 0.0/ 8.53 & 1.98 $\pm$ 0.0/ 18.09 & 72.87 $\pm$ 0.01/ 7.98 & 87.04 $\pm$ 0.0/ 21.84\\
				Tensor & 24.37 $\pm$ 0.9/ 93.36 & 15.07 $\pm$ 1.0/ 93.97 & \textbf{1.07 $\pm$ 0.17}/ 107.93 & \textbf{0.0 $\pm$ 0.0}/ 182.07 & 43.24 $\pm$ 2.98/ 40.21 & 32.35 $\pm$ 0.15/ 40.41\\
				IPFP & 65.13 $\pm$ 0.0/ 6.35 & 60.81 $\pm$ 0.0/ 6.28 & 40.90 $\pm$ 0.0/ 8.43 & 3.94 $\pm$ 0.0/ 12.31 & 76.19 $\pm$ 0.0/ 4.65 & 87.74 $\pm$ 0.0/ 8.90\\
				PM & 74.63 $\pm$ 0.0/ 7.07 & 71.93 $\pm$ 0.0/ 4.90 & 70.27 $\pm$ 0.0/ 0.94 & 48.82 $\pm$ 0.0/ 1.69 & 83.79 $\pm$ 0.02/ 2.98 & 91.43 $\pm$ 0.24/ 0.44\\
				SMAC & 70.00 $\pm$ 0.0/ 5.75 & 67.36 $\pm$ 0.0/ 5.52 & 50.53 $\pm$ 0.0/ 4.15 & 5.52 $\pm$ 0.0/ 6.90 & 78.56 $\pm$ 0.22/ 1.94 & 87.88 $\pm$ 0.11/ 3.42 \\
				SM & 68.54 $\pm$ 0.0/ 3.34 & 67.18 $\pm$ 0.0/ 3.47 & 65.96 $\pm$ 0.0/ 3.32 & 34.16 $\pm$ 0.0/ 4.93 & 80.27 $\pm$ 0.07/ 1.88 & 88.66 $\pm$ 0.09/ 2.10\\
				GA & 65.06 $\pm$ 0.0/ 4.53 & 64.60 $\pm$ 0.0/ 4.61 & 61.58 $\pm$ 0.0/ 4.08 & 31.62 $\pm$ 0.0/ 5.83 & 77.20 $\pm$ 0.27/ 3.51 & 87.02 $\pm$ 0.16/ 32.28\\
				Munkres & 33.71 $\pm$ 0.0/ 1.52 & 29.99 $\pm$ 0.0/ 1.49 & 51.87 $\pm$ 0.0/ 1.39 & 79.69 $\pm$ 0.0/ 2.45 & 74.00 $\pm$ 0.0/ 0.73 & 92.00 $\pm$ 0.0/ 1.16\\
				
				\hline
		\end{tabular}}
		\hfill{}
		\caption{Car, Motorbike~\cite{cho2013learning}, Butterfly, Magazine~\cite{jiang2011linear}, Building and Books error (\%) for pairwise matchings. Computation time (in seconds) is mentioned after the "/" in the above Table.}
		\label{tb:n2}
	\end{table*}
\end{center}
\vspace{-0.8cm}
\textbf{Rotation:} 
Figure~\ref{fig:first11} shows a $180^{\circ}$ rotated version of the original house frame (Figure~\ref{fig:zero11}). Table \ref{tb:Transformation} shows errors in matching when 
$20$\% of the frames are rotated by both $20^{\circ}$ and $60^{\circ}$, respectively 
and when the same transformations are applied to $40$\% of the frames. As the percentage of transformed frames with greater degree increases, we note a substantial increase in error for other methods in comparison to our method's error increase. 

\textbf{Reflection:} 
The reflected version of a house frame is shown in Figure~\ref{fig:second11}. 
Table \ref{tb:Transformation} shows that affinity-based approaches also performed equally well for reflection of house frame sequences.

\textbf{Scaling:} 
Resizing an image both horizontally and vertically scales the image as is shown in Figure~\ref{fig:third11} . We fixed the scales to $0.5$, $0.75$, $1.25$, and $1.5$ randomly in both the directions in order to transform the images. Our method in Table \ref{tb:Transformation} produces much better matchings than the other methods.

\textbf{Shearing:} 
We randomly apply shearing on house in one of the directions with shear factor $0.5$ (shown in Figure~\ref{fig:fourth11}) and measured the performance shown in Table \ref{tb:Transformation}. In addition to our method, we find that affinity-based algorithms also produce robust matchings.
 
\subsection{Effect of Incomplete and Occluded Landmarks}
To understand the effect of occlusions, we took two real-world datasets, i.e., \emph{Books} and \emph{Building}~\cite{pachauri2013solving} with severe occlusions which are scenes of the same 3D object taken from arbitrary camera angles. These datasets have widely been used in \emph{Structure from Motion (SfM)} problems and are known to be difficult for matching. 
Focusing our attention to the last two columns of Table \ref{tb:n2}, it is evident that our method gives the best results.
\begin{figure}[t!]
	\centering
	\subfigure{%
		\includegraphics[trim={0 0 0 0.45cm},clip,width=70mm,height=20mm]{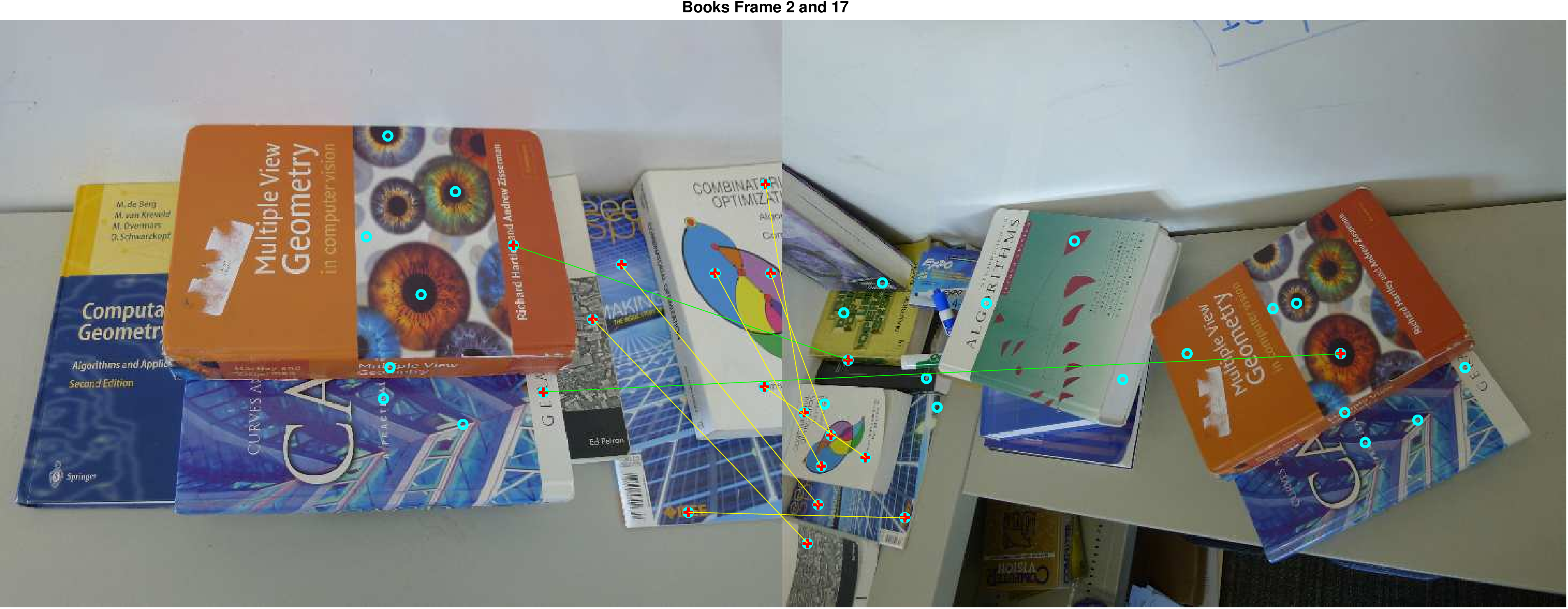}}%
	\qquad
	\subfigure{%
		\includegraphics[trim={0 0 0 0.45cm},clip,width=70mm,height=20mm]{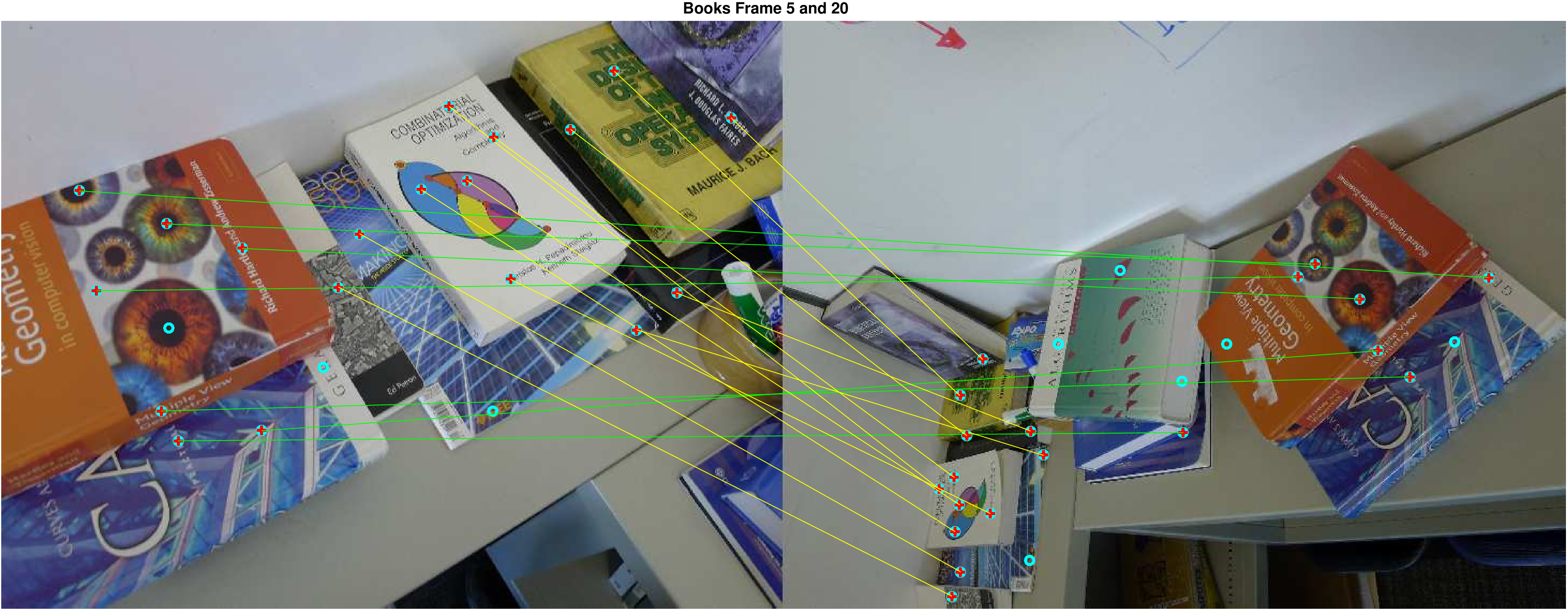}}%
	\caption{Two instances of matchings in Books dataset which is severely occluded. Yellow/green lines show correct/incorrect matches and isolated points show no matches.}
	\label{fig:books}
\end{figure}
Figure~\ref{fig:books} shows the Books dataset where books are placed on a table in various orientations with varying levels of occlusion, along with two sample matchings between different pairs of images. Note that in Figure~\ref{fig:books}, when a corresponding matching clique is not found in the other image, a match isn't \emph{forced} but rather there is no match reported, which doesn't degrade the matching accuracy. Matching as many random cliques, in order of decreasing dimensionality, as possible, manifests itself as an advantage over existing methods, especially when dealing with clutter and/or occlusions.
\begin{figure}[t]
	\centering     
	\includegraphics[width=0.54\linewidth]{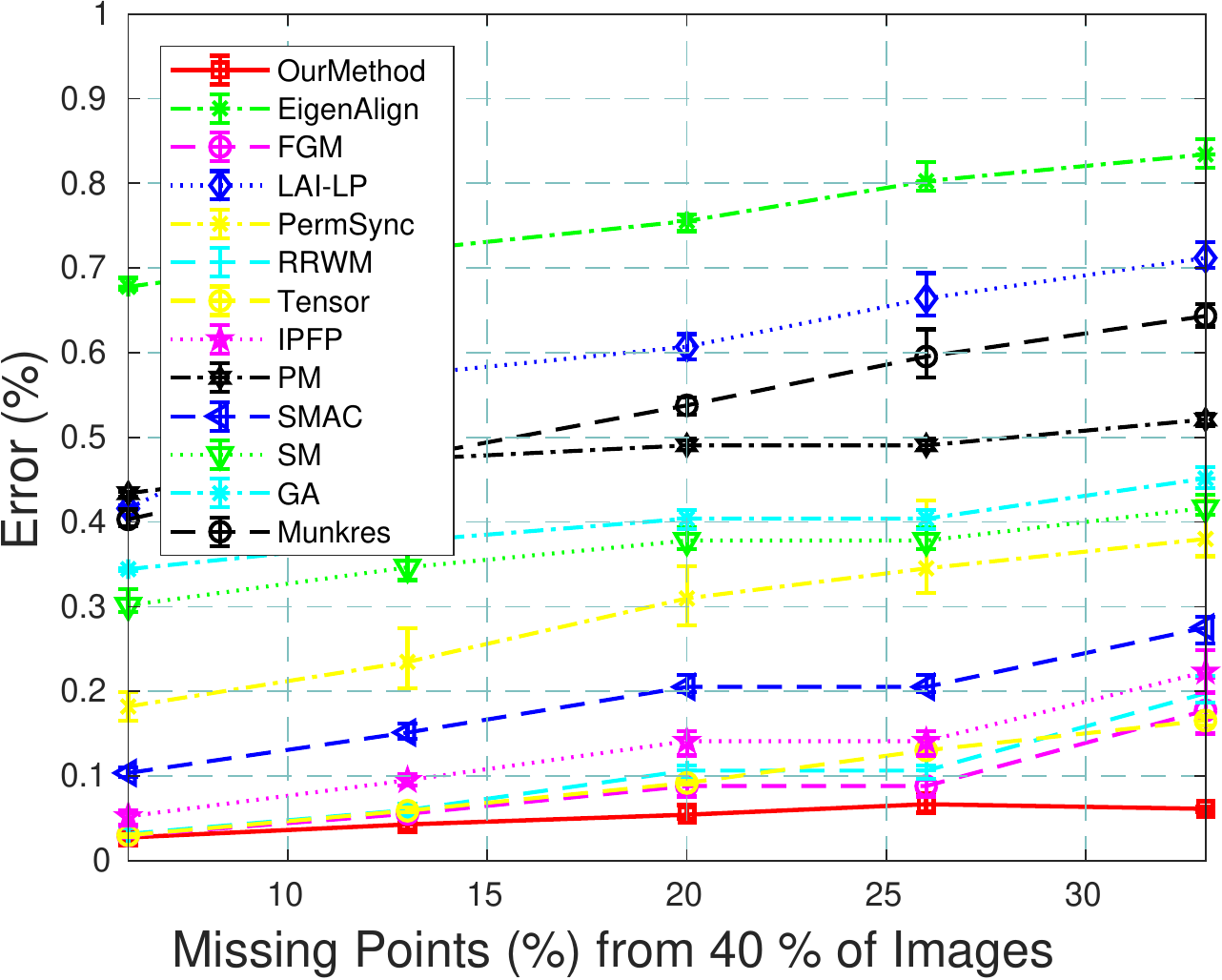}
	\caption{Error (\%) in matching when varying the number of missing landmarks in	$40\%$ of the images in the frame sequence.}
	\label{fig:occ}
\end{figure}
\textbf{Simulating missing points:}
In order to gain a deeper insight into the behavior of all the matching algorithms, we 
omit $2, 4, 6, 8$, and $10$ ($6.66\%$, $13.33\%$, $20\%$, $26.66\%$, and $33.33\%$) 
points out of total House landmark points (i.e., $30$ points) from $40\%$ (Figure~\ref{fig:occ})
 of frame sequences randomly. In general, all algorithms show an increase in error as more points are removed, but our method has a less gradual increase, while eigenvalue related methods show a rather steep increase in error. Our method is comparable to FGM and RRWM, but the gap in error increases with more missing points. We also observed that FGM incurs the longest runtime for matching in this scenario.

\subsection{Effect of Frame Separation}
Here, we pick two frames from a video for matching and vary the separation in their frame sequence number. The farther apart two frames are the more \emph{pronounced} is the effect we seek between the frame images. For example, as the frame separation increases, \href{http://vasc.ri.cmu.edu/idb/html/motion/hotel/index.html}{\emph{CMU Hotel}} undergoes a more severe 3D rotation, while \emph{Horse-Shear}~\cite{caetano2009learning} undergoes a larger degree of shear. 
\begin{figure}[t]
	\centering
	\parbox{3.6cm}{
		\includegraphics[width=3.6cm]{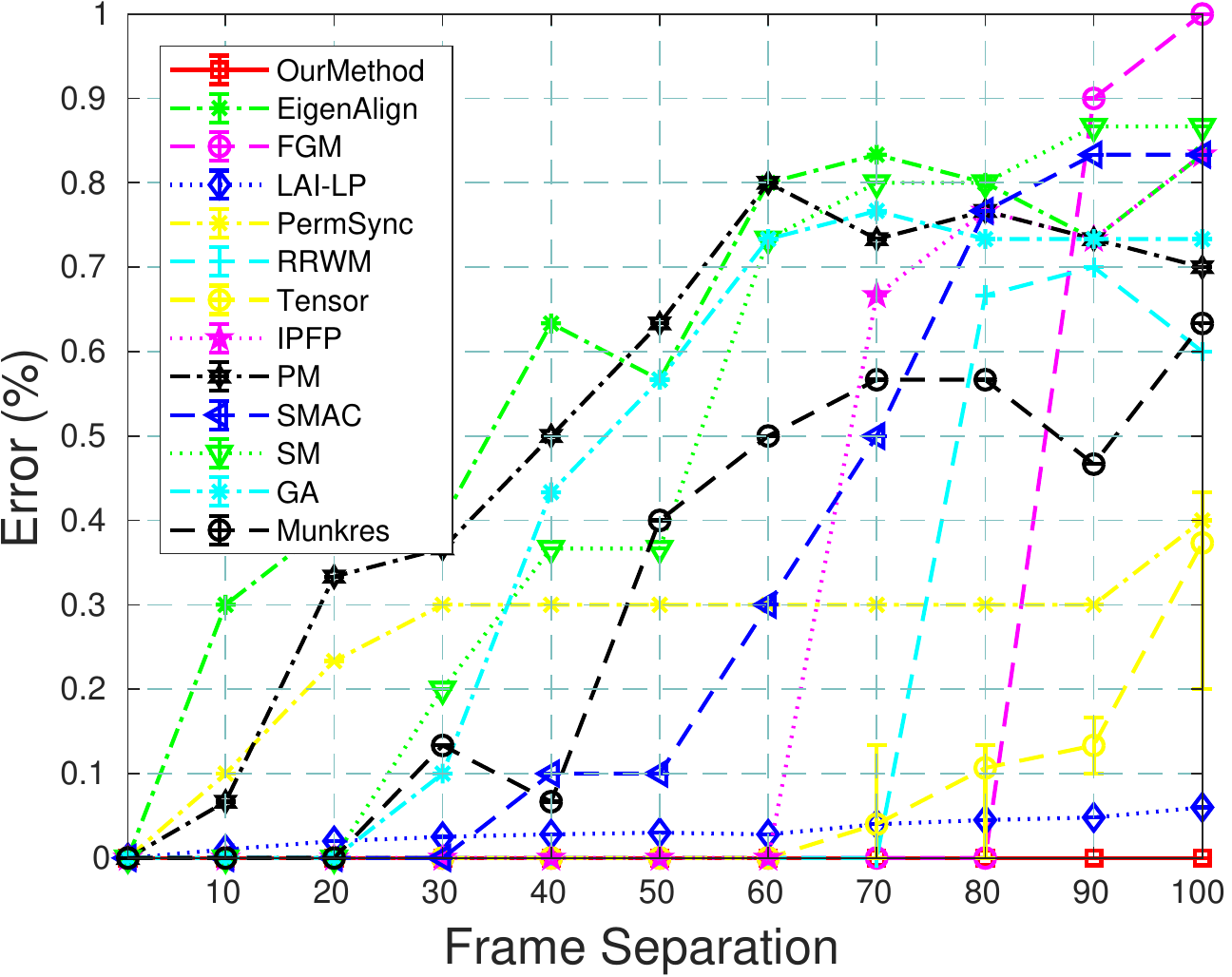}}
		\qquad
		\begin{minipage}{3.6cm}
			\includegraphics[width=3.6cm]{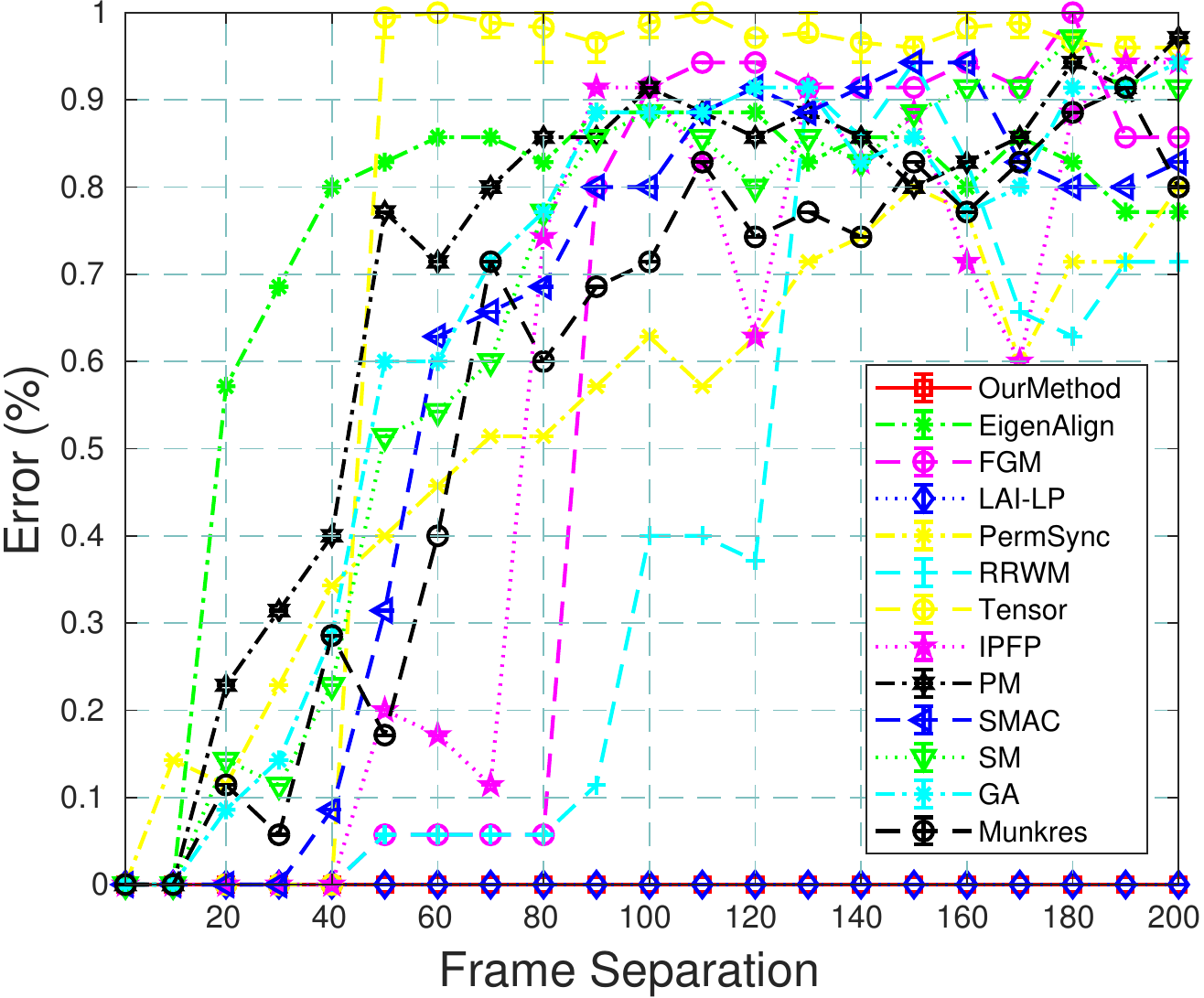}
		\end{minipage}
		\caption{Error (\%) in matching by various methods with different frame separation level for 
			\emph{CMU Hotel} (on left) and \emph{Horse Shear} (on right).}
			\label{fig:frame-seps}
\end{figure}
We set $p=0.7$ and $k=7$ as nearest neighbors to get the correct matchings.
In Figure~\ref{fig:frame-seps}, in both the left and right plots we notice that most methods show a very sharp rise in error, while our method is quite stable and reports a $0\%$ error. We observe that the naive baseline, LAI-LP also does well and doesn't exhibit steep changes in error with larger frame separation.

\textbf{Experimental Summary:}
In general, we find that the \emph{affinity matrix} based methods like FGM and RRWM are more robust to affine transformations than other competing algorithms. Our method performs the best as the weight vectors in our algorithm effectively capture even the \emph{higher-order geometric properties} of the neighborhood and nearly preserves them under affine transformations. The naive baseline, i.e., LAI-LP, does not perform as well because it also has a feature-based component like SIFT which is known to fail on some affine transformations.

\section{Conclusion}
To the best of our knowledge, we have presented the first approach towards partial higher-order matching by initially capturing higher-order structure as \emph{random clique complexes} and then proposing a corresponding matching algorithm. 
From a theoretical point of view, we studied matching as a QAP on random clique adjacency matrices that represented the $k$-skeleton of our random clique complexes and gave bounds on the concentration inequality of the spread of its eigenvalues. We also improved bounds on the largest eigenvalue of the Lawler QAP formulation, used by affinity-matrix based approaches. We discussed the robustness of such approaches to missing points and also showed the sharpness of our result. Furthermore, inspired by Finke et. al.~\cite{martello1987} we studied the asymptotic behavior of our higher dimensional clique adjacency matrices. A more detailed investigation of the distribution of eigenvalue gaps for such random matrices with Poisson distributed entries is left for future work. 

From an empirical perspective, we compared the matching accuracies of diverse algorithms on both synthetic and real-world datasets that were known to have severe occlusions and distortions, thus posing a daunting challenge to matching algorithms. We argue that our experiments show strong evidence that our approach outperforms all the state-of-the-art matching methods on a diverse range of datasets.  

\clearpage
\section*{Acknowledgements}
We thank our colleagues from the Mathematics Dept. at IIT-H (Sukumar Daniel, Narasimha Kumar, and Bhakti B. Manna) for their insight and expertise.
We would also like to thank all the reviewers for their feedback and suggestions. We are grateful to the authors of ~\cite{ZhouD16,zhou2013deformable,cho2010reweighted,feizi2016spectral,pachauri2013solving,li2013object,duchenne2011tensor} for providing their source codes and datasets.

\bibliography{mka}
\bibliographystyle{icml2018}

\clearpage
\appendix
\section{Proofs}
	
	\subsection{Upper Bound to Clique Size in a Random Graph}
	\label{ssec:bound}
	Let $G(n,p)$ denote the Erd\H{o}s-R\'{e}nyi random graph on $n$ vertices, i.e., 
	$G(n,p) = \{G_{ij}|1 \leq i <j\leq n\}$, where $G_{ij} \sim Ber(p)$ are i.i.d
	Bernoulli random variables. We denote the number of $k$-cliques in the realization
	of $G(n,p)$ as $X_n(k)$. By definition, a $k$-clique in a graph $G$ is a 
	subset $A$ of $k$ vertices, which induce a complete subgraph of $G$. Additionally,
	no other vertex in $G$ can be joined by edges to all vertices of $A$.
	Therefore, we can represent $X_n(k)$ as a sum of indicator random variables 
	$\mathds{1}_A$, where
	\begin{equation}
	\mathds{1}_A = 
	\begin{cases}
	1 & \text{if $A$ is a $k$-clique in $G(n,p)$}\\
	0 & \text{otherwise}
	\end{cases}
	\end{equation}

	It is clear that $X_n(k) = \sum_{|A|=k} \mathds{1}_A $. Hence, we get
	\begin{align*}
	\mathbb{E}(X_n(k)) = 
	\mathbb{E}( \sum_{|A|=k} \mathds{1}_A ) 
	= \sum_{|A|=k} \mathbb{E}(\mathds{1}_A)
	= {n \choose k}p^{k \choose 2}
	\end{align*}
	
	Using Stirling's formula, we upper bound $X_n(k)$ as $\left( \frac{en}{k} \right)^k$, 
	where $e$ is the \emph{Euler's number}.

	\subsection{Quadratic Assignment Problem}
	We begin by defining the general quadratic assignment problem (QAP) using the \emph{Koopman-Beckmann} version. Let $A=(a_{vv'})$, $B=(b_{vv'}) \in 
	\mathds{R}^{n \times n}$. Let $\Pi$ denote the set of all possible bijections (permutations) $\pi:N \rightarrow N$, where $N=\{1,2,\dots n\}$.    
	We define the QAP as:
	\begin{equation*}
	\begin{aligned}
	& \text{minimize}
	& & \sum_{v,v'} b_{vv'}a_{\pi(v)\pi(v')} \\
	& \text{subject to}
	& & \pi \in \Pi
	\end{aligned}
	\end{equation*}
	
	For now on, for ease of notation, we denote the cost function $b_{vv'}a_{\pi(v)\pi(v')}$ as $\mathcal{C}_{vv'}$.

	\subsection{Asymptotic Analysis of Higher-order Clique Assignment (Proof of Theorem $3$)}

		Given that the QAP is a combinatorial optimization problem, in the case of random symmetric matrices, 
		the subset of \emph{feasible solutions} $S_\pi$ is of the form:
		\begin{align*}
		S_\pi = \{ (\pi(v),\pi(v') \mid v<v', u,v = 1,\dots ,n \}
		\end{align*}
		where, $|S_\pi| = {n+1 \choose 2}$ and $|\Pi| = n!$.
		
		Recall that our cost function $\mathcal{C}_{vv'}$ has expectation $\lambda_e$
		and variance $\lambda_v$. For notational convenience, we set 
		$\epsilon' = \lambda_v - \epsilon$ .
		Then, there exists a bijection $\pi \in \Pi$, for which the following holds by the definition of variance
		\begin{align*}
		\mathbb{P} 
		\left\lbrace     
		\frac{1}{|S_\pi|} 
		\left| 
		\sum_{v,v'\in \pi} ( \mathcal{C}_{vv'} - \lambda_e  )   
		\right| 
		\geq \epsilon'
		\right\rbrace
		\end{align*}
		To proceed further with our proof, we make use of the following lemma by Renyi et. al.~\cite{Renyi1970}. 
		\begin{lemma}
			\label{lemma:1}
			Let $X_1,\dots,X_n$ be independent random variables with 
			$|X_k - \mathbb{E}(X_k)| \leq 1$, $k=1,\dots,n$. Denote
			\[
			D := \sqrt{ \sum_{k=1}^{n} \mathbb{V}(X_k) }
			\]
			and let $\mu$ be a positive real number with $\mu \leq D$. Then
			\begin{align*}
			\mathbb{P} \left\lbrace \left| \sum_{k=1}^{n} (X_k - \mathbb{E}(X_k)) \right| \geq \mu D \right\rbrace \leq 2 
			\exp\left(- \frac{\mu^2}{2(1 + \mu/2D)^2 }    \right)
			\end{align*}
		\end{lemma} \qed
		\vspace{1em}

		In order to apply Lemma~\ref{lemma:1}, we change the form of the inequality 
		as follows:
		\begin{align}
		&\mathbb{P} 
		\left\lbrace     
		\frac{1}{|S_\pi|} 
		\left| 
		\sum_{v,v'\in \pi} ( \mathcal{C}_{vv'} - \lambda_e  )   
		\right| 
		\geq \epsilon'
		\right\rbrace \label{eq2}\\
		\leq&
		\sum_{\pi \in \Pi} 
		\mathbb{P} 
		\left\lbrace     
		\frac{1}{|S_\pi|} 
		\left| 
		\sum_{v,v'\in \pi} ( \mathcal{C}_{vv'} - \lambda_e  )   
		\right| 
		\geq \epsilon'
		\right\rbrace \\	
		\leq&
		|\Pi|
		\mathbb{P} 
		\left\lbrace     
		\frac{1}{|S_\pi|} 
		\left| 
		\sum_{v,v'\in \pi} ( \mathcal{C}_{vv'} - \lambda_e  )   
		\right| 
		\geq \epsilon'
		\right\rbrace \label{eq4}								
		\end{align}
		Before applying Lemma~\ref{lemma:1}, we compute $D$ as, 
		\begin{align*}
		D = \sqrt{ \sum_{v,v'\in \pi} \lambda_v} = 
		\sqrt{|S_\pi | \lambda_v} 
		\end{align*}
		We can rewrite~(\ref{eq4}) as 
		\begin{align*}
		&|\Pi|
		\mathbb{P} 
		\left\lbrace     
		\left( \frac{\sqrt{\lambda_v}}{\sqrt{|S_\pi|}}\right)
		\left( \frac{1}{ \sqrt{|S_\pi|\lambda_v}  }\right)		
		\left| 
		\sum_{v,v'\in \pi} ( \mathcal{C}_{vv'} - \lambda_e  )   
		\right| 
		\geq \epsilon'
		\right\rbrace\\
		=&|\Pi|
		\mathbb{P} 
		\left\lbrace 
		\left| 
		\sum_{v,v'\in \pi} ( \mathcal{C}_{vv'} - \lambda_e  )   
		\right| 
		\geq 
		\underbrace{\left( \frac{\epsilon' \sqrt{|S_\pi|}} { \sqrt{\lambda_v} } \right)}_{\mu}
		\underbrace{\left( \sqrt{|S_\pi| \lambda_v} \right)}_{D}		
		\right\rbrace
		\end{align*}
		
		Now, we make use of Lemma~\ref{lemma:1} and get 
		\begin{align*}
		&|\Pi|
		\mathbb{P} 
		\left\lbrace 
		\left| 
		\sum_{v,v'\in \pi} ( \mathcal{C}_{vv'} - \lambda_e  )   
		\right| 
		\geq 
		\left( \frac{\epsilon' \sqrt{|S_\pi|}} { \sqrt{\lambda_v} } \right) 
		\left( \sqrt{|S_\pi| \lambda_v} \right)		
		\right\rbrace \\
		\leq&2|\Pi|	
		\exp 
		\left( 
		- \frac
		{ \left( \frac{\epsilon'\sqrt{|S_\pi|}}{\sqrt{\lambda_v}} \right)^2 }
		{ 2 \left( 1 + \frac{\epsilon' \sqrt{|S_\pi|}}{\sqrt{\lambda_v}} 
			\frac{1}{2\sqrt{|S_\pi|\lambda_v}}
			\right)^2 } 
		\right)	\\
		=&
		2|\Pi|\exp \left( -2|S_\pi| \left( \frac{\epsilon'\sqrt{\lambda_v}}
		{\epsilon' + 2\lambda_v}   \right)^2 \right)
		\end{align*}
		Equation~\ref{eq2} can now be written as
		\begin{align*}
		\mathbb{P} 
		\left\lbrace     
		\frac{1}{|S_\pi|} 
		\left| 
		\sum_{v,v'\in \pi} ( \mathcal{C}_{vv'} - \lambda_e  )   
		\right| 
		\leq \epsilon'
		\right\rbrace
		\geq 1 - \\
		2|\Pi|\exp \left( -2|S_\pi| \left( \frac{\epsilon'\sqrt{\lambda_v}}
		{\epsilon' + 2\lambda_v}   \right)^2 \right) 
		\end{align*}
		It can easily be verified that the expression in the R.H.S. of the above inequality tends to $1$ as 
		$n \to \infty$.
		
		We know that for the expression $		\left| 
		\sum_{v,v'\in \pi} ( \mathcal{C}_{vv'} - \lambda_e  )   
		\right| \leq \epsilon'|S_\pi|$, the following bounds hold.
		
		\begin{align*}
		|S_\pi|(\lambda_e - \epsilon')	
		\leq \sum_{v,v'\in \pi}  \mathcal{C}_{vv'} 
		\leq |S_\pi|(\lambda_e + \epsilon')	
		\end{align*}
		It follows that 
		\begin{align*}
		\frac{\max\limits_{\pi \in \Pi}  \sum\limits_{vv'} \mathcal{C}_{vv'}}  
		{\min\limits_{\pi \in \Pi}  \sum\limits_{vv'} \mathcal{C}_{vv'}}
		\leq 
		\frac{|S_\pi|(\lambda_e + \epsilon')}{|S_\pi|(\lambda_e - \epsilon')}
		\leq 1 + \epsilon
		\end{align*}		
		This completes the proof.\qed

	\subsection{Eigenvalue Bounds on Lawler's QAP Formulation on Random Matrices (Proof of Theorem 2)}
	As illustrated in~\cite{Alon2002}, we will make use of Talagrand's concentration inequality.
	We provide a tighter bound in the case of our affinity matrix using the Rayleigh's quotient. 
	
	\begin{theorem}\cite{Talagrand1995}
		\label{thm:tal}
		Let $\Omega = \prod_{i=1}^m \Omega_i $ be a product space of probability spaces.
		Let $\mA$ and $\mA_t$ be subsets of $\Omega$ and if 
		for each $y=(y_1,\dots,y_m) \in \mA_t$,
		there exists a real vector $\alpha = (\alpha_1, \dots ,\alpha_m)$, such that for every 
		$x=(x_1,\dots,x_k) \in \mA$, the following inequality holds
		\begin{align*}
		\sum_{i:x_i \neq y_i} |\alpha_i| \leq t \left(  \sum_{i=1}^{m} \alpha_i^2  \right)^{1/2}
		\end{align*} 
		Then, 
		\begin{align*}
		\mathbb{P}[\mA] \mathbb{P}[\conj \mA_t] \leq e^{-t^2/4}.
		\end{align*}
	\end{theorem}
	Here, $\mA_t$ denotes the set with \emph{Talagrand distance} at most $t$ from $\mA$ and 
	$\conj \mA_t$ denotes the \emph{complement} of set $\mA_t$.\qed
	
	\begin{theorem}
		\label{thm:bound}
		For a real symmetric matrix $A=(a_{ij}) \in R^{m \times m}$ and for positive constant $t$,
		\begin{align*}
		\mathbb{P}[ |\lambda_1(A) - \mM| \geq t] \leq 4e^{-t^2/8},
		\end{align*}
		where $\mM$ is the median of $\lambda_1(A)$.
	\end{theorem}
	\begin{proof}\footnote{Our proof technique follows the technique outlined in~\cite{Alon2002}}
		Given a real symmetric matrix $A=(a_{ij}) \in R^{m \times m}$ and a non-zero vetor $x$, the \emph{Rayleigh Quotient} $\mR(A,x)$ is defined as
		\begin{align*}
		\mR(A,x) = \frac{x^T A x}{x^T x}
		\end{align*}
		Given the eigenvalues of $A$ in decreasing order as $\lambda_1(A) \geq \dots \geq \lambda_m(A)$, we know that $\mR(A,x) \in [\lambda_m(A), \lambda_1(a)]$. 
		It is well known that $\mR(A,x)$ attains its maximum value at $\lambda_1(A)$
		when $x=v$, where $v$ is the eigenvector corresponding to $\lambda_1(A)$.
		Therefore, we have
		\begin{align}
		\label{eqn:ray}
		\lambda_1(A) = \mR(A,v) = \frac{v^T A v}{v^T v}
		\end{align}
		In our proof, we omit the constant factor $v^T v$ and normalize the eigenvector $v$, hence $\left\lVert v  \right\rVert =1$.
		
		Consider the product space $\Omega$ of entries $a_{ij}$, $1 \leq i \leq j \leq m$.
		Let $t,\mM$ be real numbers, where $t>0$ and $\mM$ is the median of $\lambda_1(A)$.
		Let $\mA$ be the set of matrices $A=(a_{ij}) \in \Omega$, for which 
		$\lambda_1(A) \leq \mM$. By definition, $\mathbb{P}[\mA] \geq 1/2$. Additionally, let 
		$\mB$ be the set of matrices $B=(b_{ij}) \in \Omega$, for which $\lambda_1(B) \geq \mM+t$. Using Rayleigh's equation~(\ref{eqn:ray}) for $\lambda_1(A)$, we rewrite it as a summation of diagonal and off-diagonal terms
		\begin{align*}
		\lambda_1(A) = \mR(A,v) = v^T A v = 
		\underbrace{\sum_{1\leq i <j \leq m} ( v^T_i v_j + v^T_j v_i )a_{ij}}_{	\text{off-diagonal}} + \\
		\underbrace{\sum_{i=1}^{m} v^T_i v_i a_{ii} }_{\text{diagonal}} \leq \mM
		\end{align*}
		and
		\begin{align*}
		\lambda_1(B) = \mR(B,v) = v^T B v = 
		\underbrace{\sum_{1\leq i <j \leq m} ( v^T_i v_j + v^T_j v_i )b_{ij}}_{	\text{off-diagonal}} + \\
		\underbrace{\sum_{i=1}^{m} v^T_i v_i b_{ii}}_{\text{diagonal}} \geq \mM+t
		\end{align*}
		In order to apply Talagrand's inequality (Theorem~\ref{thm:tal}), we set a real vector $\alpha = (\alpha_{ij})_{1 \leq i \leq j \leq m}$ as follows: 
		For off-diagonal $(1 \leq i < j \leq m)$ terms, we set
		\begin{align*}
		\alpha_{ij} = (v^T_i v_j + v^T_j v_i)
		\end{align*}
		For diagonal $(1 \leq i \leq m)$ terms, we set
		\begin{align*}
		\alpha_{ii} = v^T_i v_i 
		\end{align*}
		We proceed by first proving two claims that will be used in this proof.
		\begin{claim}
			\begin{align*}
			\sum_{1 \leq i \leq j \leq m} \alpha_{ij}^2 \leq 2		
			\end{align*}
		\end{claim}
		\begin{proof}
			By definition,
			\begin{align*}
			\sum_{1 \leq i \leq j \leq m} \alpha_{ij}^2 &= 
			\sum_{i=1}^{m} (v^T_i v_i)^2 + 
			\sum_{1 \leq i <j \leq m} (v^T_i v_j + v^T_j v_i)^2 \\
			&< 2 \left( \sum_{i=1}^{m} {v^T_i}^2 \right)
			\left( \sum_{i=1}^{m} {v_i}^2 \right)\\ 
			&= 2 \quad\text{(since $\left\lVert v  \right\rVert =1$)}
			\end{align*}
			This completes the proof.
		\end{proof}
		
		\begin{claim}
			For every $A \in \mA$,
			\begin{align*}
			\sum_{ 1 \leq i \leq j \leq m; a_{ij} \neq b_{ij} } |\alpha_{ij}| \geq t
			\end{align*}
		\end{claim}
		\begin{proof}
			Recall that for matrix $A \in \mA$, $v$ is the eigenvector with unit-norm corresponding to $\lambda_1(A)$.
			We know that,
			\begin{align*}
			v^TAv \leq \lambda_1(A) \leq \mM \quad\text{(from set $\mA$)}
			\end{align*}
			while,
			\begin{align*}
			v^TBv \geq \lambda_1(B) \geq \mM +t \quad\text{(from set $\conj \mA_t$)}
			\end{align*}
			We observe that the entries in affinity matrices $A$ and $B$, are \emph{affinity scores} in interval $[0,1]$. Therefore, we have 
			$| b_{ij} - a_{ij}| \leq 1$, for all $1 \leq i,j \leq m$.
			For ease of notation, let us denote by $P$, the set of ordered pairs $ij$
			with $1 \i, j \leq m$ where $a_{ij} \neq b_{ij}$. Then,
			\begin{align*}
			t &\leq v^T(B-A)v 
			= \sum_{i,j\in P} (b_{ij} - a_{ij})v^T_i v_j\\
			&\leq \sum_{i,j\in P} |v^T_i| |v_j|
			\leq \sum_{i,j\in P} |\alpha_{ij}|
			\end{align*}
			This completes the proof.
		\end{proof}
		By the above two claims, we get the following form:
		\begin{align*}
		\sum_{x_i \neq y_i} |\alpha_i| \geq t 
		>  \left( \frac{t}{\sqrt{2}} \right)
		\left( \sum_{1\leq i\leq j\leq m} \alpha_{ij}^2 \right)^{1/2}
		\end{align*}
		Applying Talagrand's inequality, we get
		\begin{align*}
		\mathbb{P}[ \lambda_1(A) \leq \mM] \mathbb{P}[ \lambda_1(B) \geq \mM+t]
		\leq e^{ \frac{-1}{4}  \left( \frac{t}{\sqrt{2}}\right)^2    }
		\leq e^{-t^2/8}
		\end{align*}
		Since $\mM$ is the median of $\lambda_1(A)$, by definition 
		$\mathbb{P}[ \lambda_1(A) \leq \mM] \geq 1/2$, then
		\begin{align}
		\label{p1}
		\mathbb{P}[ \lambda_1(A) \geq \mM +t] \leq 2e^{-t^2/8}
		\end{align}
		Accordingly, we also have that,
		\begin{align}
		\label{p2}
		\mathbb{P}[ \lambda_1(A) \leq \mM -t] \leq 2e^{-t^2/8}
		\end{align}
		Combining results~(\ref{p1}) and (\ref{p2}), we have
		\begin{align}
		\label{p3}
		\mathbb{P}[ |\lambda_1(A) - \mM| \geq t] \leq 4e^{-t^2/8}
		\end{align}
		This completes the proof.
	\end{proof}

	\subsection{Proof of Lemma 1}
		For ease of understanding, we drop the $(A)$ as it is obvious from context. Let $\lambda_i$ be the $i$-th eigenvalue of $A$, and let $x_i \neq 0$ be its corresponding eigenvector. From $A x_i = \lambda_i x_i$, we have
		\begin{align*}
		A X_i = \lambda_i X_i, \text{\hspace{1em}where } 
		X_i:= \begin{bmatrix}[c|c|c]
		x_i & \dots & x_i
		\end{bmatrix}
		\in \mM_n \setminus \{0\}
		\end{align*} 
		It follows,
		\begin{align*}
		|\lambda_i|  \vertiii{X_i} \vertiii{X_i}  = \vertiii{\lambda_i X_i}    
		=  \vertiii{AX_i} \leq \vertiii{A} \vertiii{X_i} .
		\end{align*} 
		
		As $\vertiii{X}$ is non-negative, we get $|\lambda_i| \leq \vertiii{A}$.
		Thus, every eigenvalue of $A$ is upper bounded by the matrix norm $\vertiii{A}$.
		Applying the triangle inequality, we get that $\delta_i(A) = |\lambda_{i+1}(A) - \lambda_i(A)| \leq 2 \vertiii{A}$, which completes the proof.\qed
	\section{Example}

	\begin{figure*}[h]
	\makebox[\linewidth]{
	  \subfigure[]{%
	  \label{fig:match3}%
	    \includegraphics[width=75mm,height=33mm]{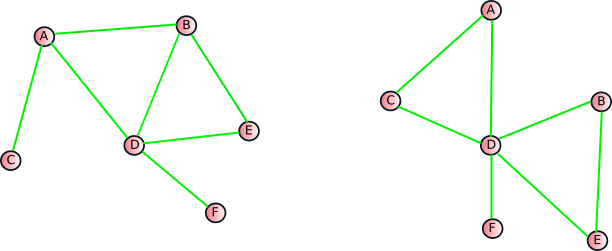}}%
	    
	  \qquad
	  \subfigure[]{%
	  \label{fig:match2}%
	    \includegraphics[width=75mm,height=33mm]{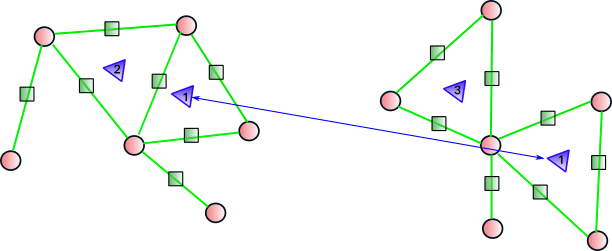}}%
	  \qquad
	  }
	\makebox[\linewidth]{
	\subfigure[]{%
	  \label{fig:match1}%
	    \includegraphics[width=75mm,height=33mm]{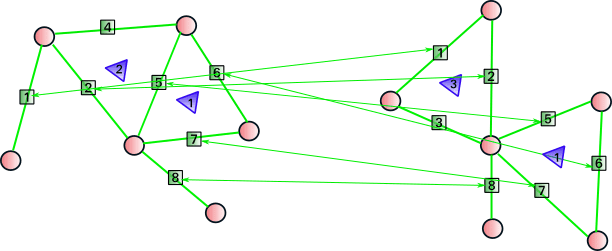}}%
	  \qquad
	  \subfigure[]{%
	  \label{fig:match0}%
	    \includegraphics[width=75mm,height=33mm]{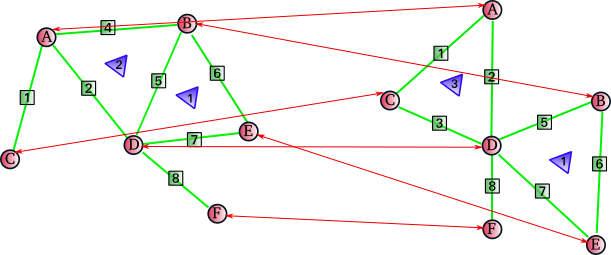}}%
	    }
	  \vspace{-5mm}
	  \caption{Matching corresponding 3-clique, 2-cliques and 1-cliques (b)-(d) respectively in a pair of (a) Erd\H{o}s-R\'{e}nyi graph.}
	  \label{fig:Example}
	\end{figure*}
\begin{table*}[h]
	{\scriptsize
		\hfill{}
		\begin{tabular}{|l||l|l|l|l|}
			\hline
			\multicolumn{1}{|l||}{\textbf{Cliques}} & \multicolumn{2}{|l|}{\textbf{Graph 1 ($G_1$)}} & \multicolumn{2}{|l|}{\textbf{Graph 2 ($G_2$)}}\\
			
			\cline{1-5}
			& Clique & Neighbours & Clique & Neighbours \\
			\cline{2-5}
			\hline
			\hline
			3-Cliques & \{A,B,D\} & \{A\}, \{B\}, \{D\}, \{A,B\}, \{A,D\}, \{B,D\}, \{B,D,E\} & \{A,C,D\} & \{A\}, \{C\}, \{D\}, \{A,C\}, \{A,D\}, \{C,D\}, \{B,D,E\} \\
			 & \{B,D,E\} & \{B\}, \{D\}, \{E\}, \{B,D\}, \{B,E\}, \{D,E\}, \{A,B,D\} & \{B,D,E\} & \{B\}, \{D\}, \{E\}, \{B,D\}, \{B,E\}, \{D,E\}, \{A,C,D\} \\
			\hline
			2-Cliques & \{A,B\} & \{A\}, \{B\}, \{A,B,D\}  & \{A,C\} & \{A\}, \{C\}, \{A,C,D\}\\
			& \{A,C\} & \{A\}, \{C\} & \{A,D\} & \{A\}, \{D\}, \{A,C,D\}\\
			& \{A,D\} & \{A\}, \{D\}, \{A,B,D\} & \{B,D\} & \{B\}, \{D\}, \{B,D,E\}\\
			& \{B,D\} & \{B\}, \{D\}, \{A,B,D\}, \{B,D,E\} & \{B,E\} & \{B\}, \{E\}, \{B,D,E\}\\
			& \{B,E\} & \{B\}, \{E\}, \{B,D,E\} & \{C,D\} & \{C\}, \{D\}, \{A,C,D\}\\
			& \{D,E\} & \{D\}, \{E\}, \{B,D,E\} & \{D,E\} & \{D\}, \{E\}, \{B,D,E\}\\
			& \{D,F\} & \{D\}, \{F\} & \{D,F\} & \{D\}, \{F\}\\
			\hline
			1-Cliques & \{A\} & \{A,B\}, \{A,C\}, \{A,D\}, \{A,B,D\} & \{A\} & \{A,C\}, \{A,D\}, \{A,C,D\}\\
			& \{B\} & \{A,B\}, \{B,D\}, \{B,E\}, \{A,B,D\}, \{B,D,E\} & \{B\} & \{B,D\}, \{B,E\}, \{B,D,E\}\\
			& \{C\} & \{A,C\} & \{C\} & \{A,C\}, \{C,D\}, \{A,C,D\} \\
			& \{D\} & \{A,D\}, \{B,D\}, \{D,E\}, \{D,F\}, \{A,B,D\}, & \{D\} & \{A,D\}, \{B,D\}, \{C,D\}, \{D,E\}, \{D,F\}, \{A,C,D\},\\
			& 	& \{B,D,E\} & 	&  \{B,D,E\}\\
			& \{E\} & \{B,E\}, \{D,E\}, \{B,D,E\} & \{E\} & \{B,E\}, \{D,E\}, \{B,D,E\}\\
			& \{F\} & \{D,F\} & \{F\} & \{D,F\}\\
			\hline
	\end{tabular}}
	\hfill{}
	\vspace{5mm}
	\caption{Neighbourhood of $3,2,1$-cliques of graphs $G_1$ and $G_2$ shown in Figure~\ref{fig:Example}.}
	\label{tb:exampleN}
\end{table*}

	\begin{table*}[h]
	{
		\hfill{}
		\begin{tabular}{|l||l|l|}
			\hline
			\multicolumn{1}{|l||}{\textbf{Cliques}} & \multicolumn{1}{|l|}{\textbf{Graph 1 ($G_1$)}} & \multicolumn{1}{|l|}{\textbf{Graph 2 ($G_2$)}}\\
			
			\cline{1-3}
			\hline
			\hline
			3-Cliques & \{B,D,E\} $\rightarrow$ 1 & \{B,D,E\} $\rightarrow$ 1\\
			\hline
			2-Cliques & \{A,C\} $\rightarrow$ 1  & \{A,C\} $\rightarrow$ 1\\
			& \{A,D\} $\rightarrow$ 2 & \{A,D\} $\rightarrow$ 2 \\
			& \{B,D\} $\rightarrow$ 5 & \{B,D\} $\rightarrow$ 5\\
			& \{B,E\} $\rightarrow$ 6 & \{B,E\} $\rightarrow$ 6\\
			& \{D,E\} $\rightarrow$ 7 & \{D,E\} $\rightarrow$ 7 \\
			& \{D,F\} $\rightarrow$ 8 & \{D,F\} $\rightarrow$ 8 \\
			\hline
			1-Cliques & \{A\}, \{B\}, \{C\} & \{A\}, \{B\}, \{C\} \\
			& \{D\}, \{E\}, \{F\} & \{D\}, \{E\}, \{F\}\\
			\hline
	\end{tabular}}
	\hfill{}
	\vspace{5mm}
	\caption{Matchings of $3,2,1$-cliques of graphs $G_1$ and $G_2$ shown in Figure~\ref{fig:Example}.}
	\label{tb:exampleM}
\end{table*}
We explained our method with the help of an example shown in Figure~\ref{fig:Example}, Table~\ref{tb:exampleN} and Table~\ref{tb:exampleM} for a better understanding. We consider two random graphs 
$G_1$ and $G_2$ with $6$ vertices each in Figure~\ref{fig:match3} for which we perform higher-order matching from $3$-cliques~\ref{fig:match2} to $1$-cliques~\ref{fig:match0}. 
For a higher-order matching, we take the neighbourhood of a barycenter of a clique as the barycenters of the other cliques it is connected to. Thus, we place additional nodes 
of different order in the neighbourhood of each clique in addition to the same order cliques. This information would help the cliques to have more accurate matches. The 
neighbours and matchings of $3$-cliques, $2$-cliques and $1$-cliques are mentioned in the Table~\ref{tb:exampleN} and Table~\ref{tb:exampleM} for both the graphs $G_1$ and $G_2$ respectively. And, the matchings shown in Figure
~\ref{fig:match2},~\ref{fig:match1} and \ref{fig:match0} are based on having the same labels for each barycenter in graph $G_1$ and $G_2$.

	\section{Experiments}
\subsection{Setup}
	We compare the performance of our proposed method with various other matching algorithms on synthetic and real world datasets. The real world datasets are categorized in 
	Table~\ref{tab:datasets}. Here, $N$ is the total number of samples with $n$ landmark points in each image to be matched. We represent random graphs on images in Figure~\ref{fig:randomGraphs} for better understanding and visualization of random graphs for our experiments. Matchings of two images for real world datasets (Table~\ref{tab:datasets}) are shown in Figure~\ref{fig:Matching}.
	\begin{table}[h]
		\centering
		\scriptsize
		\begin{tabular}{|l||l|l|}
			\hline
			\textbf{Groups} & \textbf{Dataset} & \textbf{$N \times n$} \\
			\hline
			\hline
			\multirow{2}{*}{Video Frames}
			& \href{http://vasc.ri.cmu.edu/idb/html/motion/house/index.html}{CMU House} & $111 \times 30$\\
			& \href{http://vasc.ri.cmu.edu/idb/html/motion/hotel/index.html}{CMU Hotel} & $101 \times 30$\\ 
			\hline
			\multirow{2}{*}{Affine}
			& Horse-Rot~\cite{caetano2009learning} & $200 \times 35$\\
			& Horse-Shear~\cite{caetano2009learning} & $200 \times 35$\\ 
			\hline
			\multirow{2}{*}{Occluded}
			& Books~\cite{pachauri2013solving} & $20 \times 34$\\
			& Building~\cite{pachauri2013solving} & $16 \times 28$\\ 
			\hline
			\multirow{2}{*}{Non-Affine}
			& Magazine~\cite{jiang2011linear} & $30 \times 30$\\
			& Butterfly~\cite{jiang2011linear} & $30 \times 19$\\ 
			\hline
			\multirow{2}{*}{Object Matching}
			& Car~\cite{cho2013learning} & $40 \times 10$\\
			& Bike~\cite{cho2013learning} & $40 \times 10$\\ 
			\hline
		\end{tabular}
		\vspace{5mm}
		\caption{Datasets used, where $N$ is the number of samples and $n$ is the dimensionality of each sample.}
		\label{tab:datasets}
	\end{table}
	\begin{figure*}[h]
	\makebox[\linewidth]{
	  \subfigure[]{%
	  \label{fig:random1}%
	    \includegraphics[trim={0 0 0 1.65cm},clip,width=52mm,height=20.5mm]{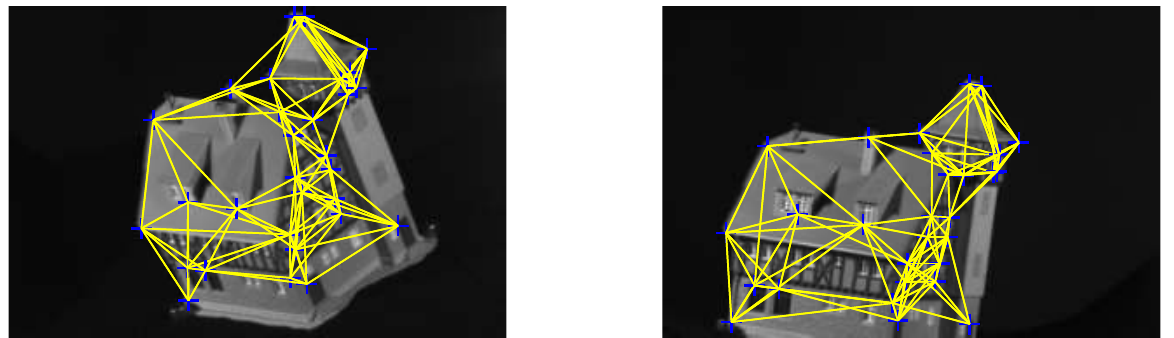}}%
	    
	  \qquad
	  \subfigure[]{%
	  \label{fig:random2}%
	    \includegraphics[trim={0 0 0 1.72cm},clip,width=52mm,height=20.5mm]{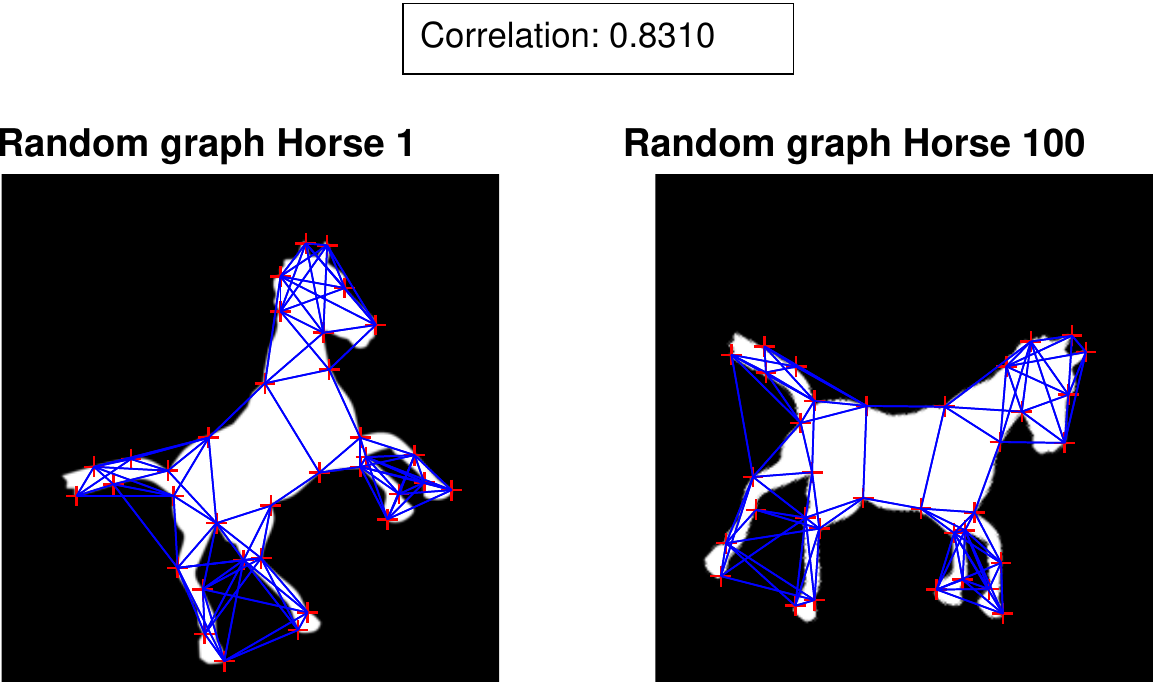}}%
	  \qquad
	  \subfigure[]{%
	  \label{fig:random3}%
	    \includegraphics[trim={0 0 0 1.65cm},clip,width=52mm,height=20.5mm]{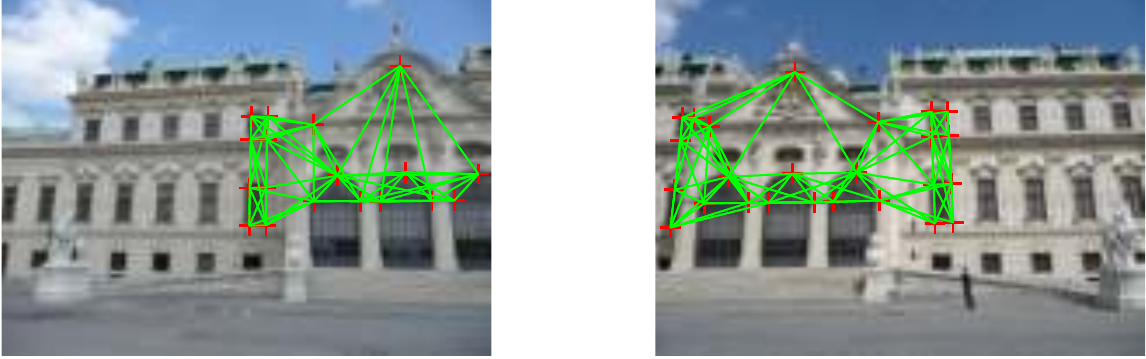}}%
	    }
	  \vspace{-5mm}
	  \caption{Pair of Erd\H{o}s-R\'{e}nyi graphs on \emph{CMU House}, \emph{Horse Rotate}, and \emph{Building} datasets.}
	  \label{fig:randomGraphs}
	\end{figure*}

    \begin{figure*}
\centering
\makebox[\linewidth]{
    \subfigure[]{%
    \label{fig:zero11}%
    \includegraphics[width=27mm,height=27mm]{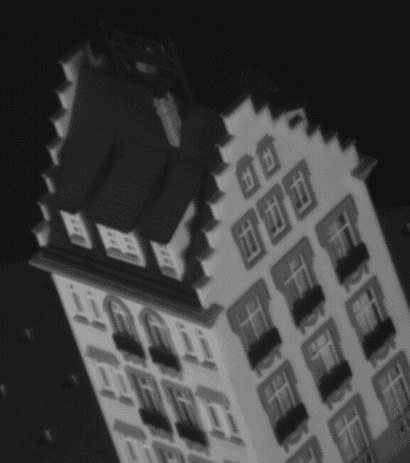}}%
    \qquad
    \subfigure[]{%
    \label{fig:first11}%
    \includegraphics[width=27mm,height=27mm]{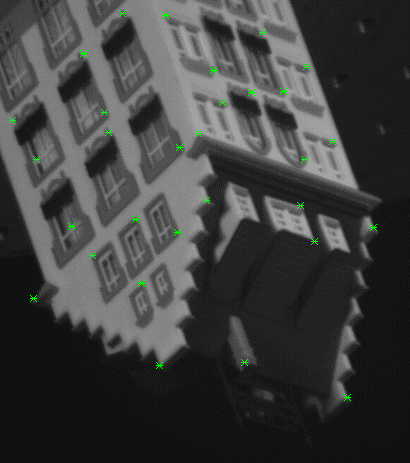}}%
    \qquad
    \subfigure[]{%
    \label{fig:second11}%
    \includegraphics[width=27mm,height=27mm]{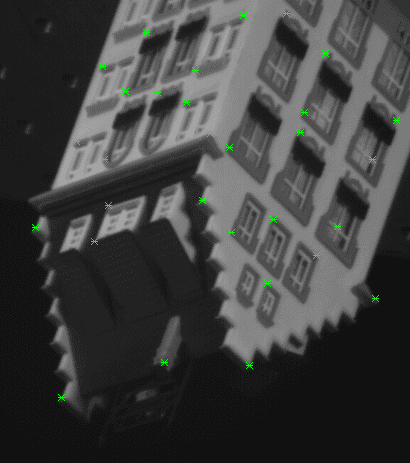}}%
    \qquad
    \subfigure[]{%
    \label{fig:third11}%
    \includegraphics[width=27mm,height=27mm]{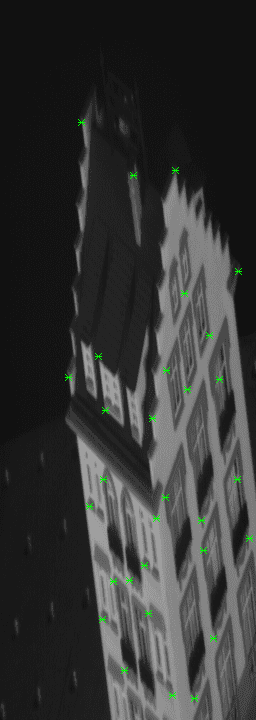}}%
    \qquad
    \subfigure[]{%
    \label{fig:fourth11}%
    \includegraphics[width=27mm,height=27mm]{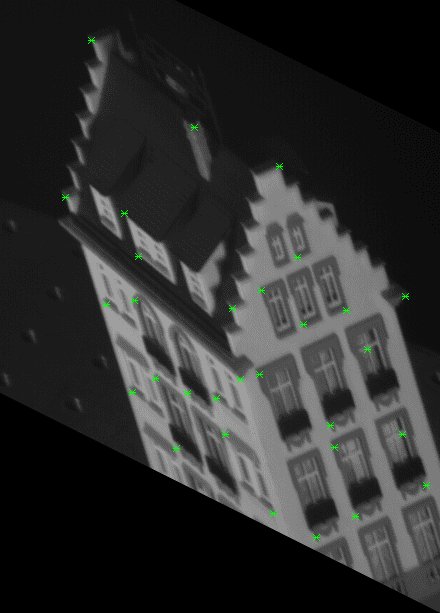}}%
    }
    \vspace{-5mm}
    \caption{(a) Original \emph{Hotel} frame, (b)--(e) four transformations on hotel frame: rotation, reflection, scaling, and shear (green markers show true matching case with the original frame (a)).}
    \label{fig:HotelTransform}

\end{figure*}

    \begin{figure*}
    \centering
    \makebox[\linewidth]{
    \subfigure[]{%
    \label{fig:trans1}%
    \includegraphics[width=52mm,height=42mm]{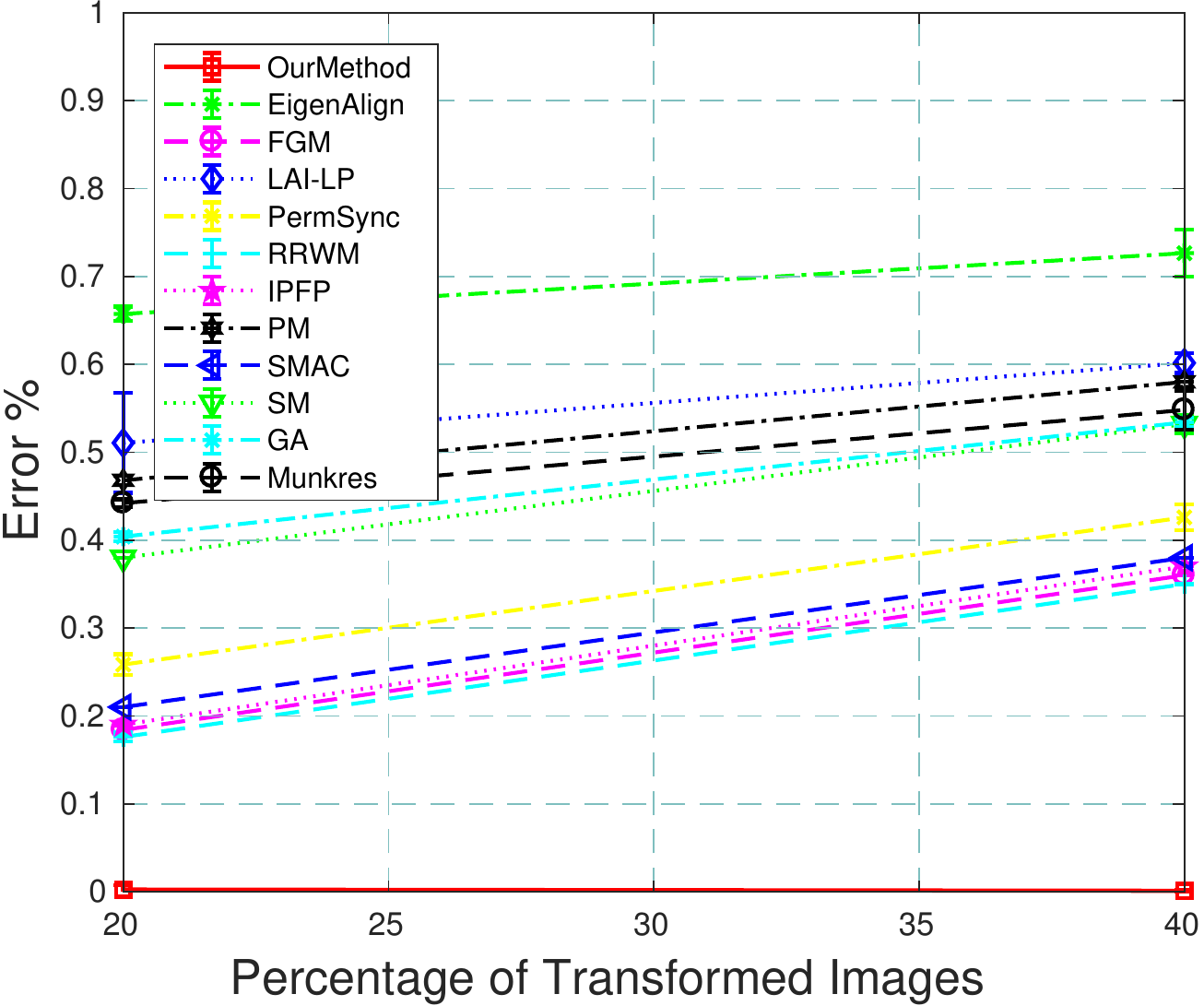}}%
    \qquad
    \subfigure[]{%
    \label{fig:trans2}%
    \includegraphics[width=52mm,height=42mm]{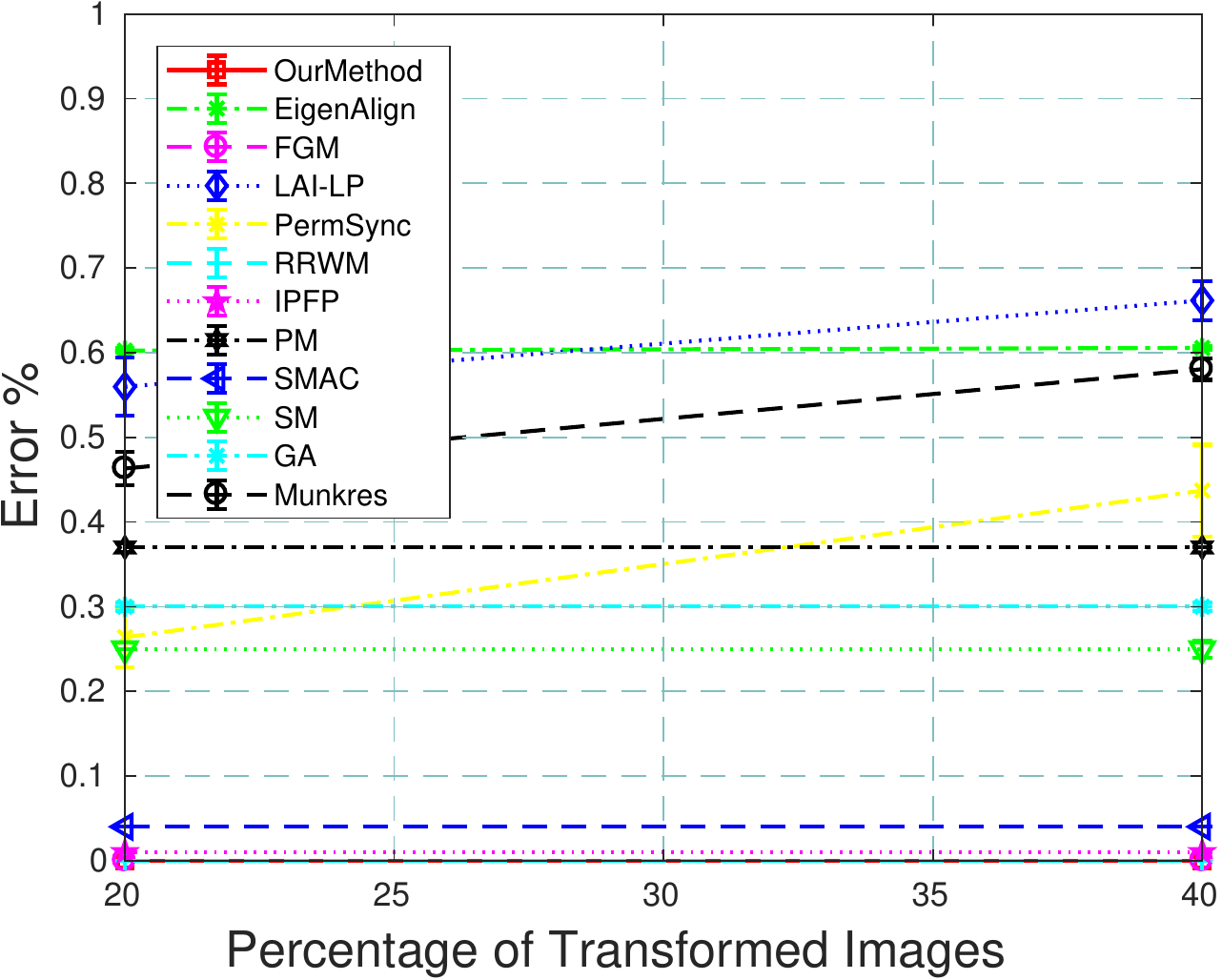}}%
    \qquad
    \subfigure[]{%
    \label{fig:trans3}%
    \includegraphics[width=52mm,height=42mm]{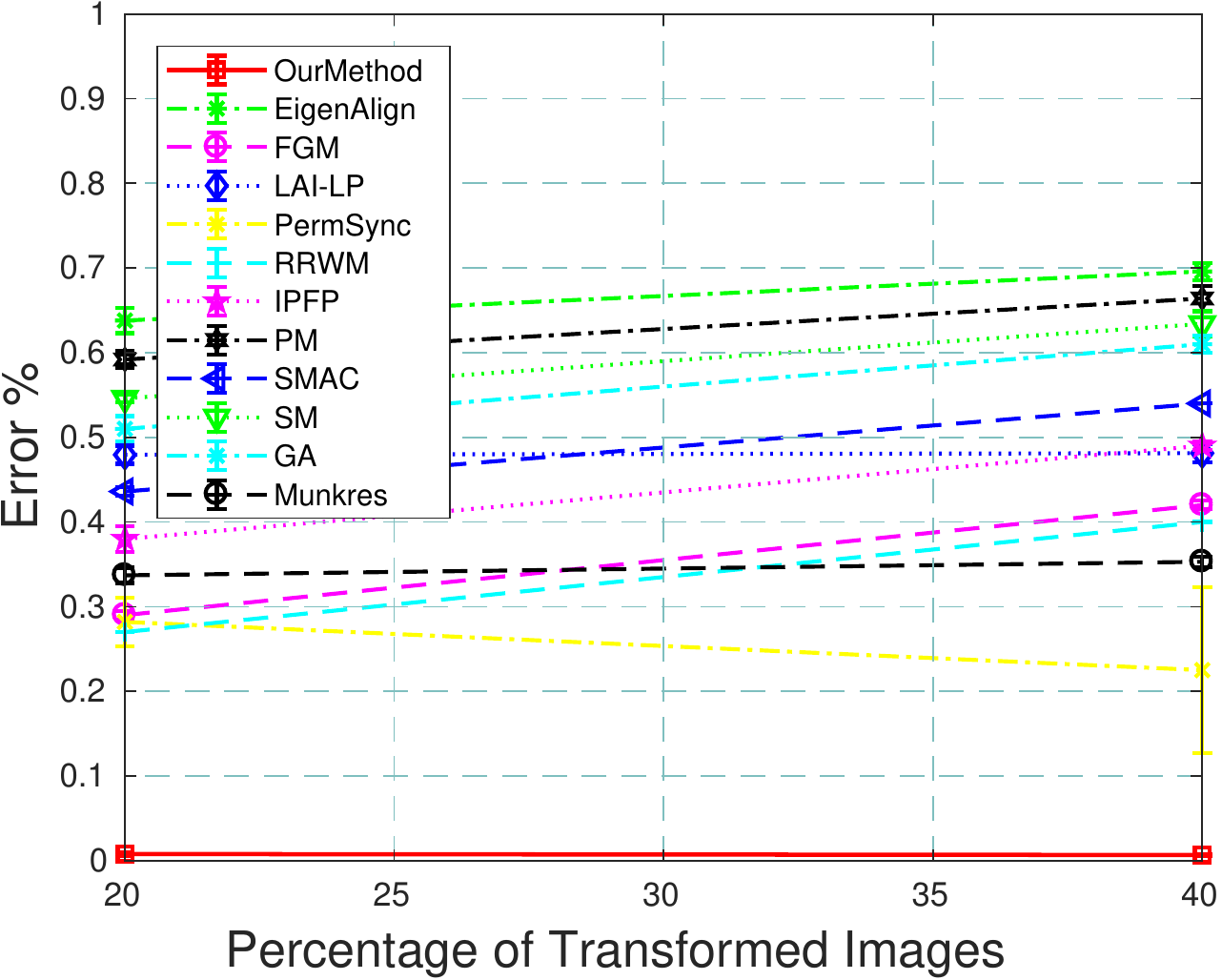}}%
    }
    \makebox[\linewidth]{
    \subfigure[]{%
    \label{fig:trans4}%
    \includegraphics[width=52mm,height=42mm]{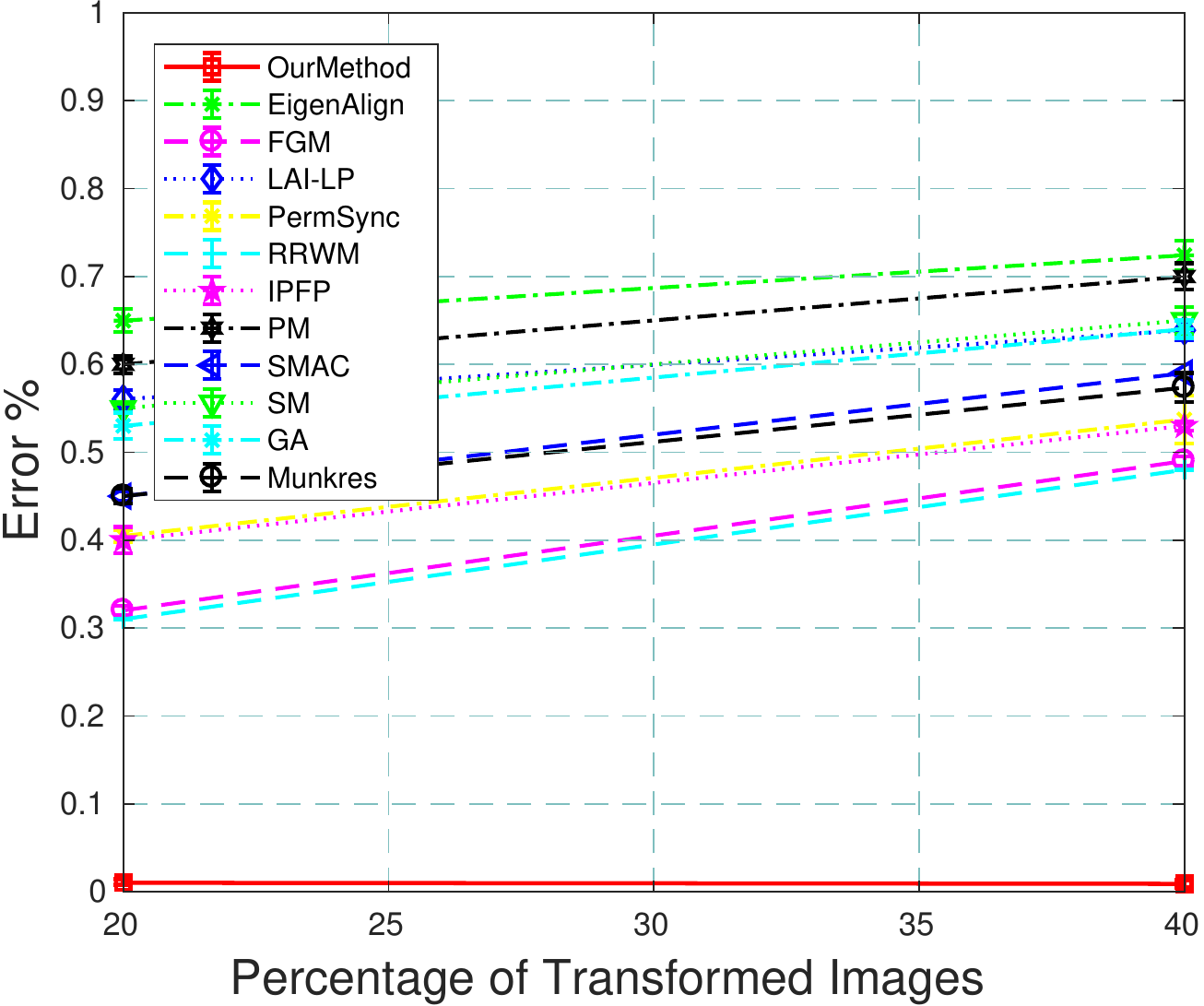}}%
        \qquad
    \subfigure[]{%
    \label{fig:trans5}%
    \includegraphics[width=52mm,height=42mm]{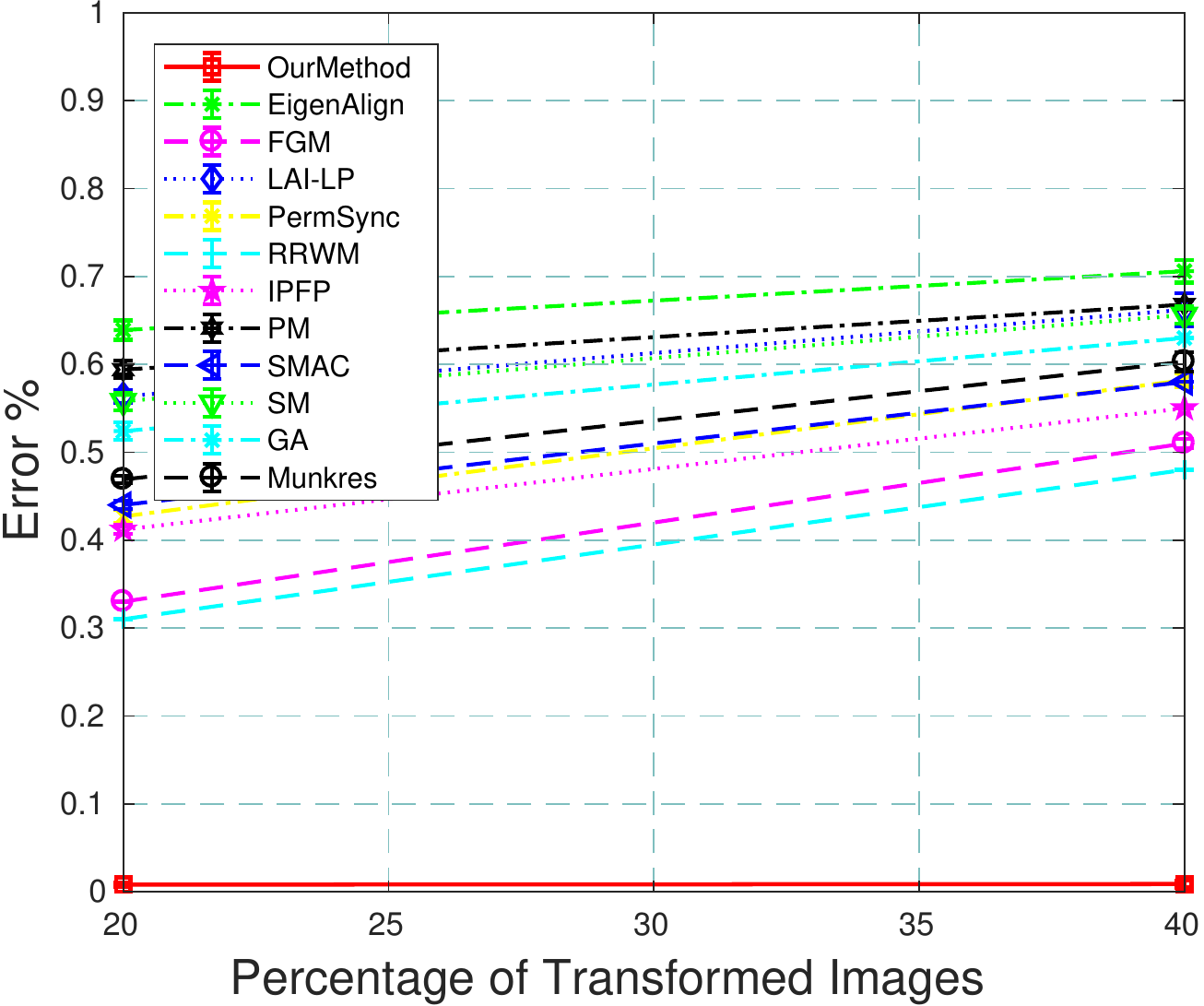}}%
        \qquad
    \subfigure[]{%
    \label{fig:trans6}%
    \includegraphics[width=52mm,height=42mm]{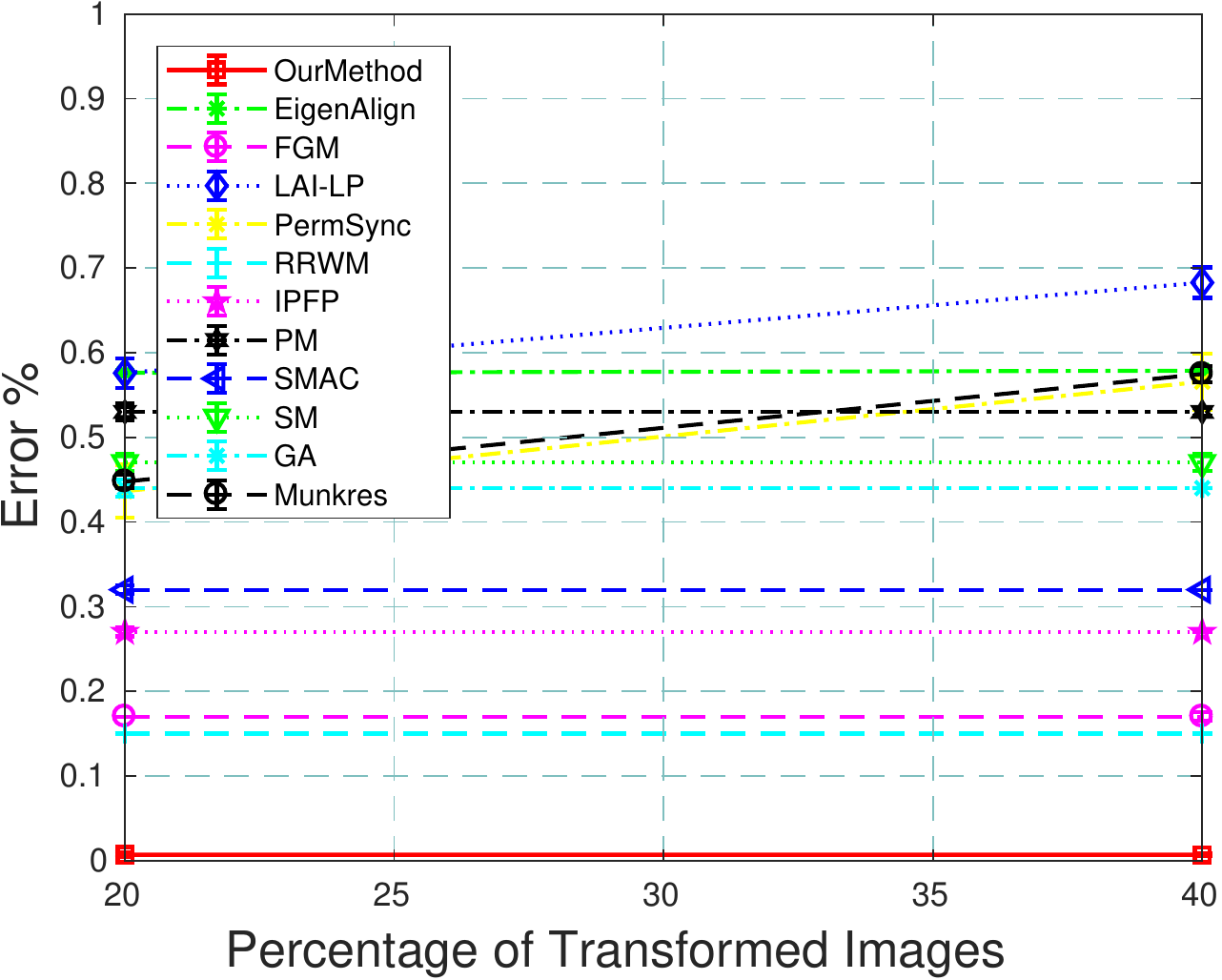}}%
    }
    \makebox[\linewidth]{
    \subfigure[]{%
    \label{fig:trans7}%
    \includegraphics[width=52mm,height=42mm]{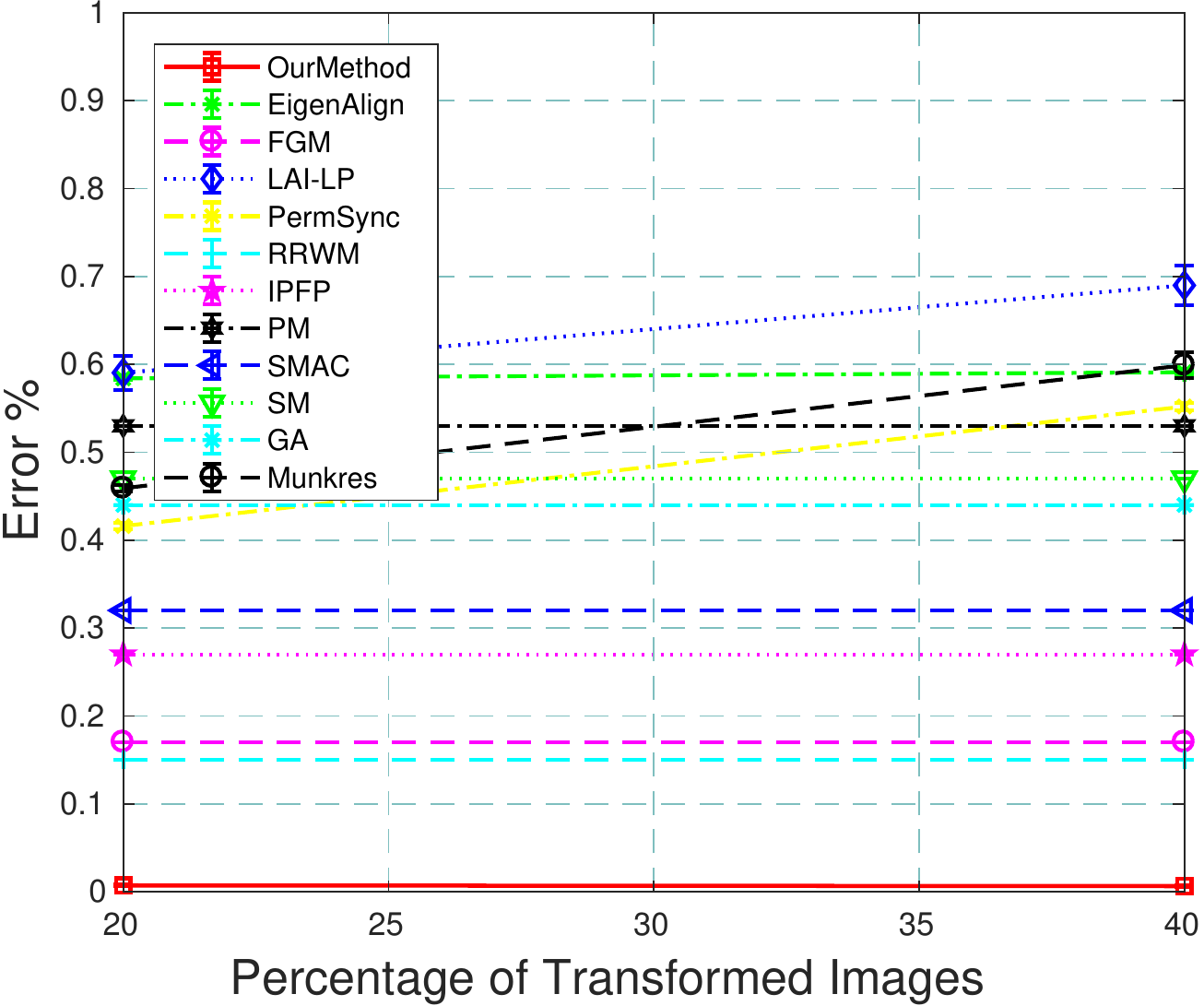}}%
        \qquad
    \subfigure[]{%
    \label{fig:trans8}%
    \includegraphics[width=52mm,height=42mm]{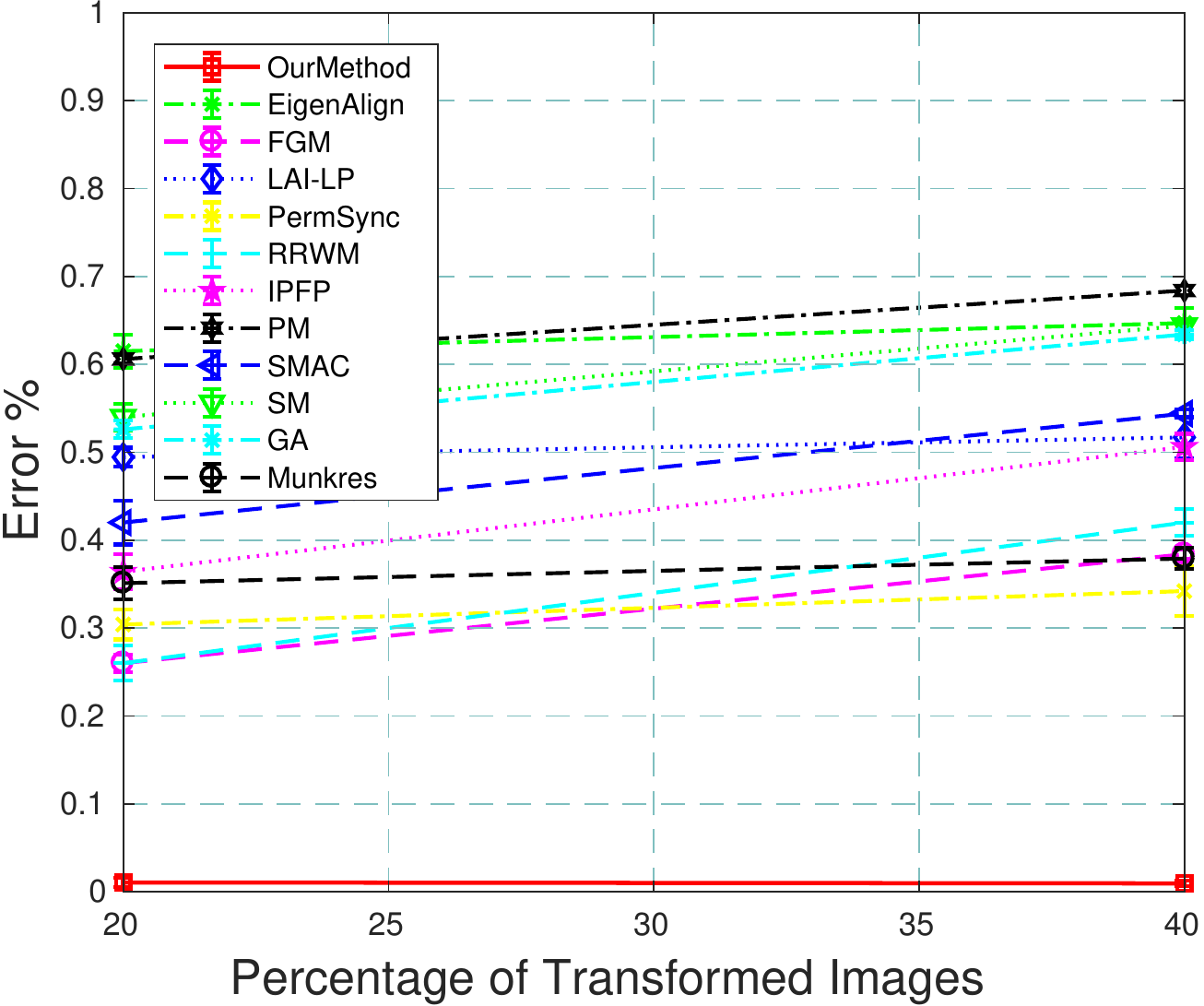}}%
        \qquad
    \subfigure[]{%
    \label{fig:trans9}%
    \includegraphics[width=52mm,height=42mm]{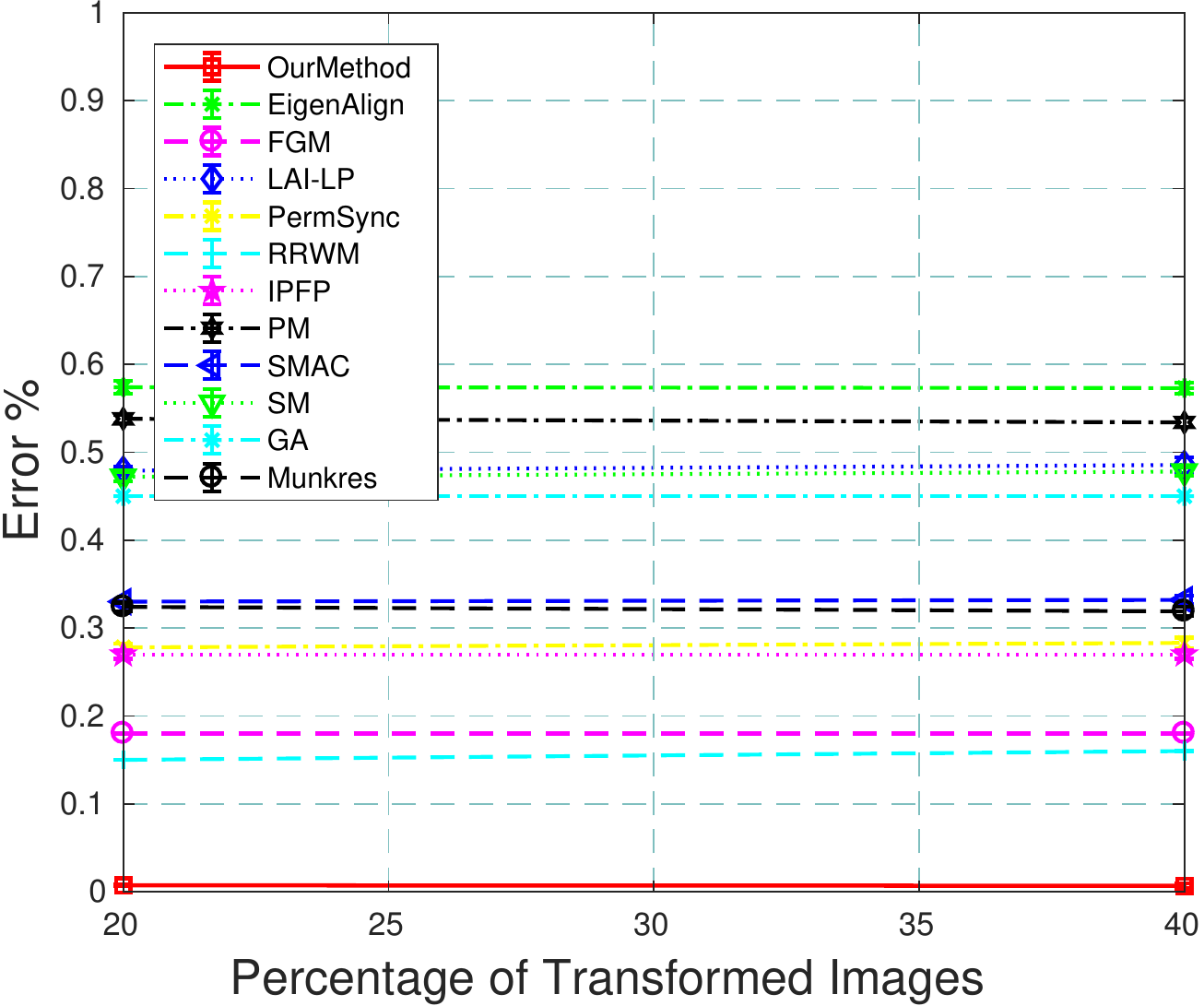}}%
    }
    \vspace{-5mm}
    \caption{Error(\%) in matching when varying the percentage ($20\%$ and $40\%$) of transformed images in the frame sequence of \emph{CMU House} (a)-(b) and \emph{CMU Hotel} (c)-(i) . (a) $40^\circ$ rotation, (b) $90^\circ$ rotation, (c) $20^\circ$ rotation, (d) $40^\circ$ rotation, (e) $60^\circ$ Degree rotation, (f) $90^\circ$  rotation, (g) reflection, (h) scaling, and (i) shearing.}
    \label{fig:Transformation}
    \end{figure*}

        \subsection{Effect of Affine Transformation}
    We created a synthetic dataset from CMU House and Hotel dataset by uniformly sampling $20\%$ and $40\%$ frames from a video sequence and performing affine transformations like rotation, 
    reflection, scaling, and shearing. We have explained the transformations we considered for this experiment which is similar to Figure (2)  and Table (1) in main paper. Table(1) in main paper shows the results on the CMU House dataset. 
    Affine transformations on Hotel frame are shown in Figure~\ref{fig:HotelTransform}. Figure \ref{fig:Transformation} shows the results of matching for the remaining House (fig. \ref{fig:trans1} and \ref{fig:trans2}) and Hotel synthetic dataset for all the algorithms. We observe that our method produces best results in all the cases, whereas the error for other 
    algorithms either remains stable or increases steeply with the increase in the percentage of transformed frames in the sequence.
    \subsection{Effect of Occlusion}
    We considered two datasets with grave occlusions, mentioned in Table \ref{tab:datasets}. Figures~\ref{fig:Bldg} and \ref{fig:Book} show the matching of two images for both the datasets, although the matching results are shown in Table (2) in the main paper. We also created a synthetic dataset by removing $2, 4, 6, 8,$ and $10$ ($6.66\%$, $13.33\%$, $20\%$, $26.66\%,$ and $33.33\%$) 
points out of total house landmark points (i.e., $30$ points) from $20\%$ and $60\%$ of frame sequences randomly. Figures~\ref{fig:second14} and \ref{fig:first31} show the increase in error as we remove more points from images. We also note the difference in both the results. Since we remove points from more percentage of frames in \ref{fig:first31}, there is more gradual increase in the error. This experimental setup is similar to Figure (4) in our main paper. It shows that affinity based methods like FGM and RRWM perform well but our method still consistently outperforms all the algorithms.
 
        \begin{figure*}
    \centering
    \makebox[\linewidth]{
    \subfigure[]{%
    \label{fig:second14}%
    \includegraphics[width=49mm,height=39mm]{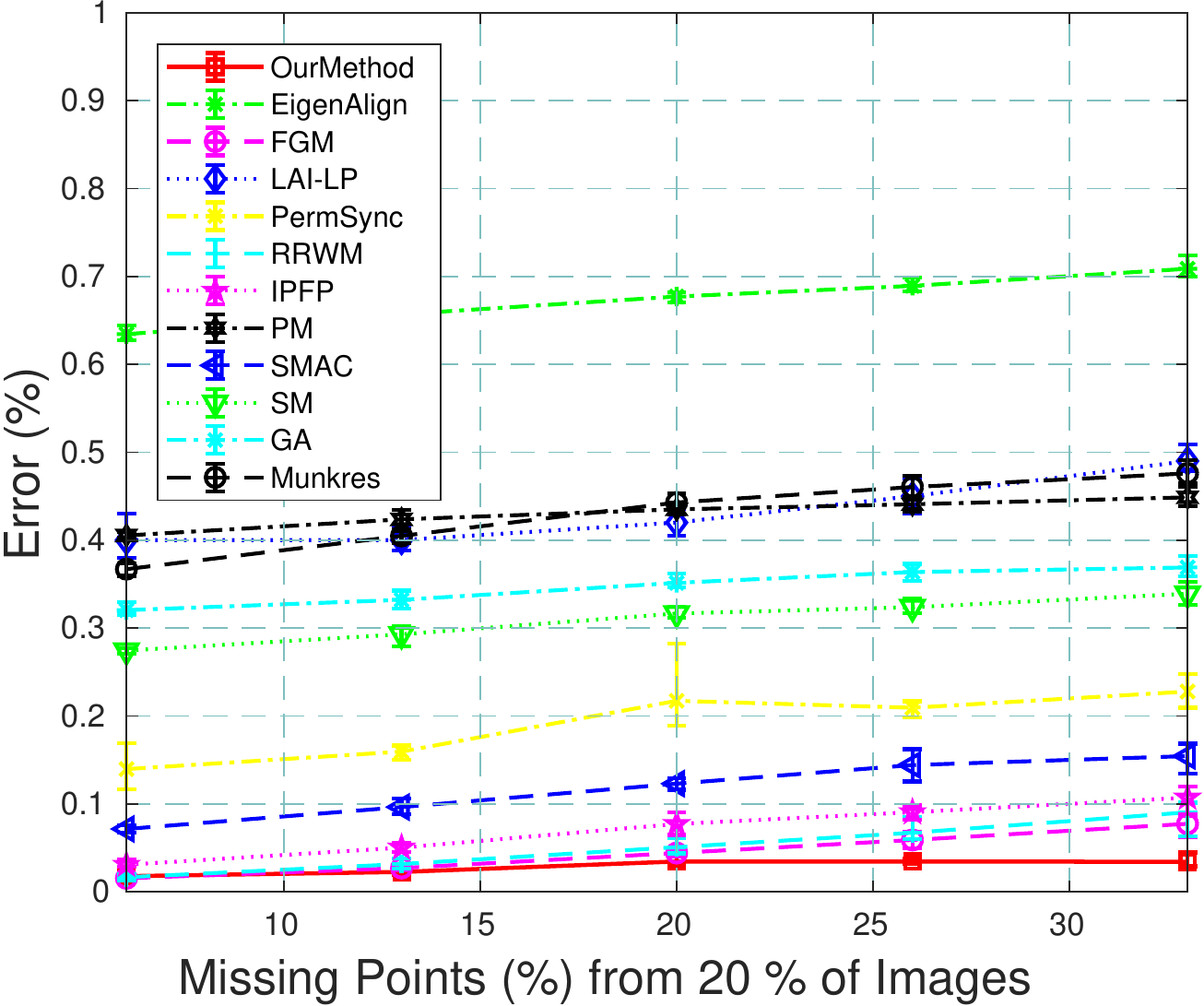}}%
    \qquad
    \subfigure[]{%
    \label{fig:first31}%
    \includegraphics[width=49mm,height=39mm]{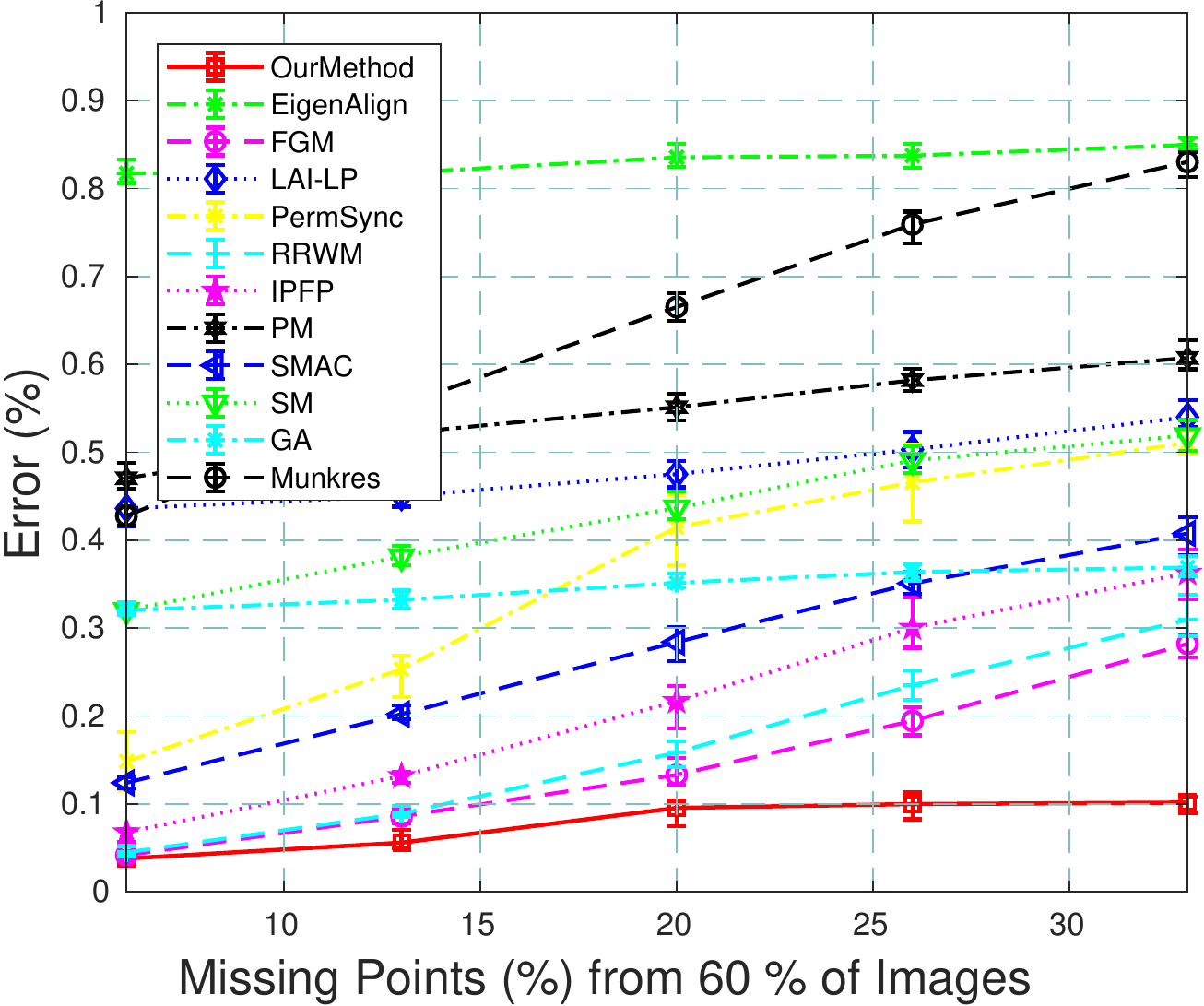}}%
    }
        \vspace{-5mm}
    \caption{Error (\%) in matching when varying the number of missing landmarks in (a) $20\%$ and (b) $60\%$ of the images in \emph{CMU House} frame sequence.}
    \label{fig:HouseMissing}
    \end{figure*}

    \begin{figure*}
    \centering
    \makebox[\linewidth]{
    \subfigure[]{%
    \label{fig:first3}%
    \includegraphics[width=49mm,height=39mm]{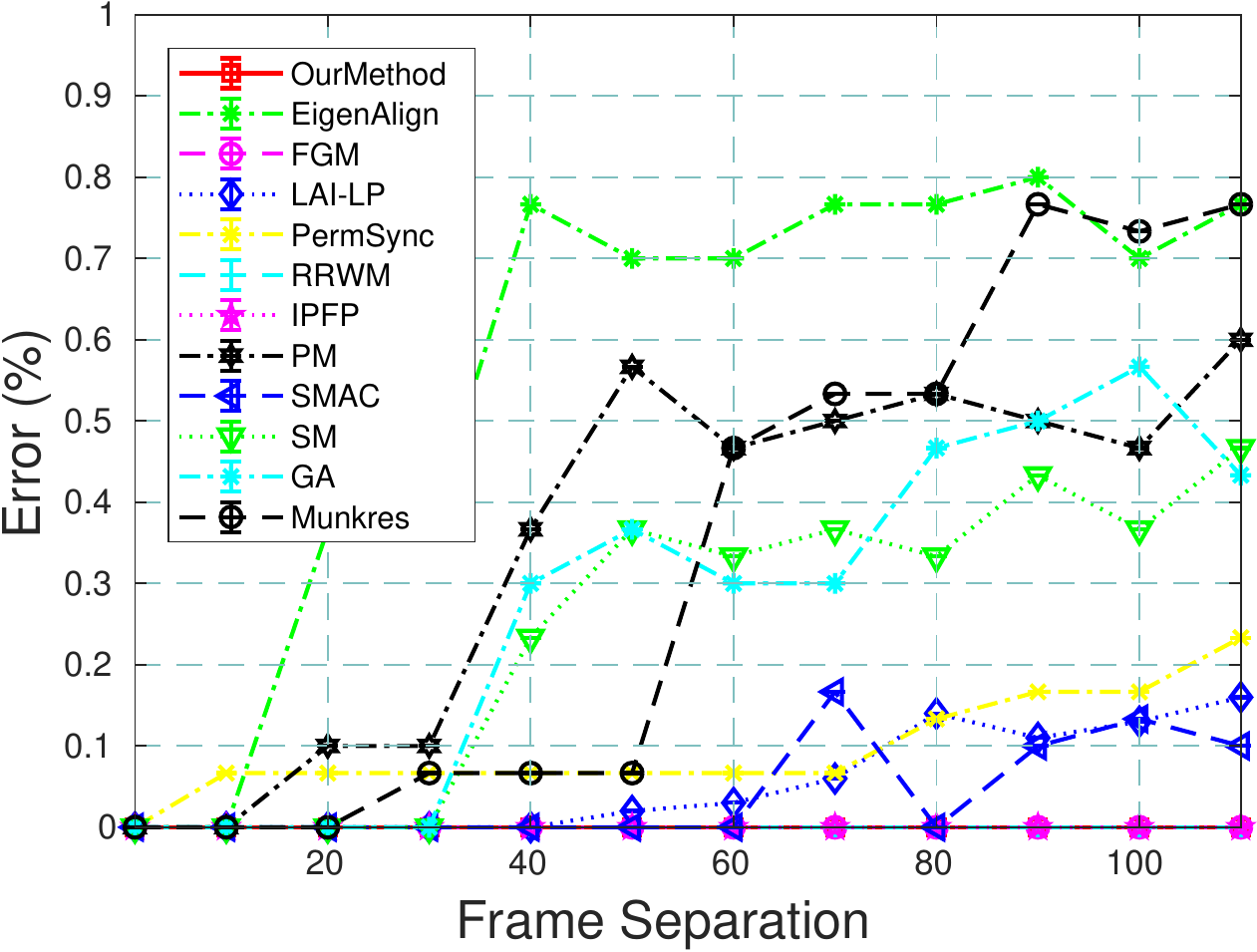}}%
    \qquad
    \subfigure[]{%
    \label{fig:second3}%
    \includegraphics[width=49mm,height=39mm]{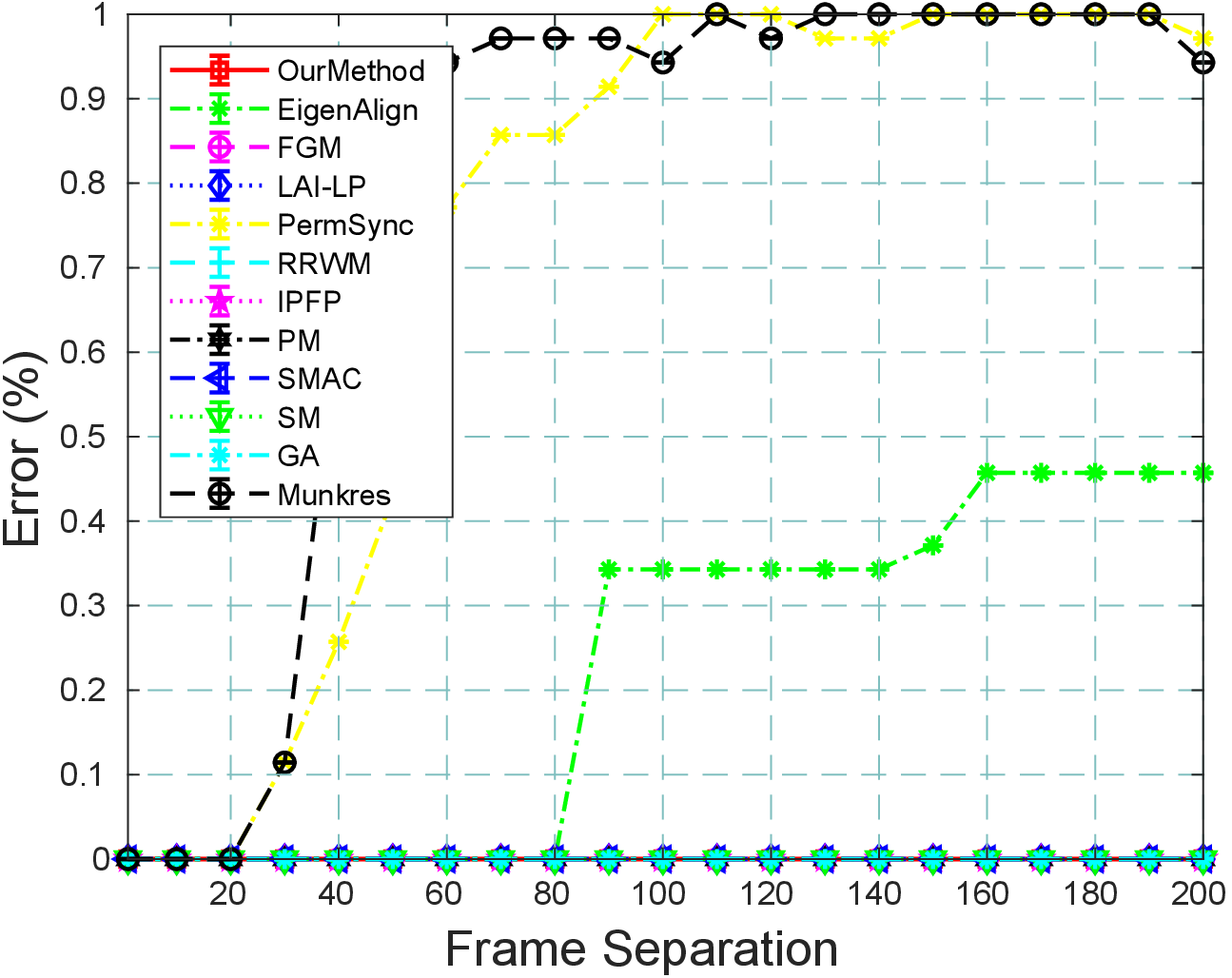}}%
    }
    \vspace{-5mm}
    \caption{Error (\%) in matching by various methods with different frame separation level for 
			(a) \emph{CMU House} and (b) \emph{Horse Rotate}.}
    \label{fig:FrameSeparation}
    \end{figure*}

   \subsection{Effect of Frame Separation}
   Figures~\ref{fig:first3} and \ref{fig:second3} show the frame separation level result of CMU House and Horse Rotate frame sequences. We select a pair of frames at a time with increase in
    their frame separation ($x$-axis). Here, the House dataset consists of 3D rotations of \emph{House} whereas \emph{Horse Rotate} dataset applies rotation with more degree of rotation as the frame separation level increases. We see that most of the algorithms performs well for both the datasets even with $0\%$ error.
   \subsection{Effect of k-Nearest Neighbour}
    In Figures~\ref{fig:first14} and \ref{fig:first13}, error and computation time of matching two frames of house are shown with different probability $p$ and nearest neighbor $k$ values. We observe that as the 
    value of $p$ and $k$ increases, the possibility of mismatching decreases which leads to correct matching. On the other hand, the computation time increases since it increases
    the number of edges in the underlying graph, which in turn leads to a larger number of $d$-cliques. This also causes a marked increase in the matching algorithm's runtime.
    The computation time of our algorithm considers the time of the \emph{Kuhn-Munkres} algorithm, which is used as a matching algorithm to match two random clique complexes, which takes $O(n^3)$ running time.

The overall time increases as we increase the value of $p$ and $k$, since it increases the probability of an edge occurrence between two landmark points. As the number of edges 
increase in a random graph, the number of $d$-cliques also increase. Due to this phenomenon, the runtime of the \emph{Kuhn-Munkres} algorithm also increases.

Figure~\ref{fig:second13} shows the computation time of matching two images with varying $k$-NN for different $n$ landmark points in the image. We can clearly see that the time increases with 
increasing $k$ and a larger number of landmark points. Here, $60$ landmark points take maximum time for the highest value of $k$. On the other hand, if we consider lower values of $k$, even $60$ 
landmark points take a reasonable amount of time to match, which is comparable to lower values of $n$. Thus, we set $k$ value as low as possible for matching, depending on the complexity of the dataset.

    \begin{figure*}
    \centering
    \makebox[\linewidth]{
    \subfigure[]{%
    \label{fig:first14}%
    \includegraphics[width=49mm,height=39mm]{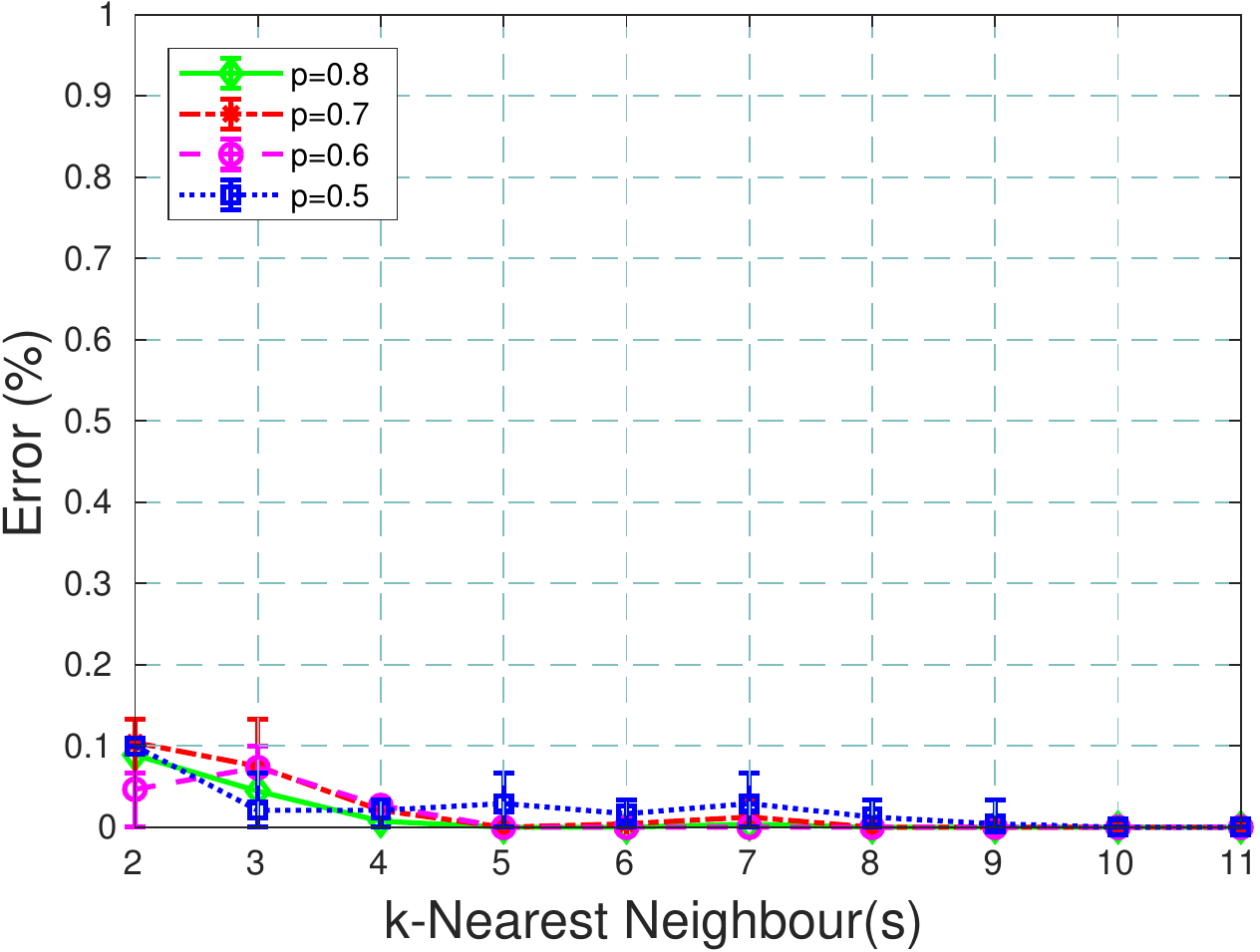}}%
    \qquad
    \subfigure[]{%
    \label{fig:first13}%
    \includegraphics[width=49mm,height=39mm]{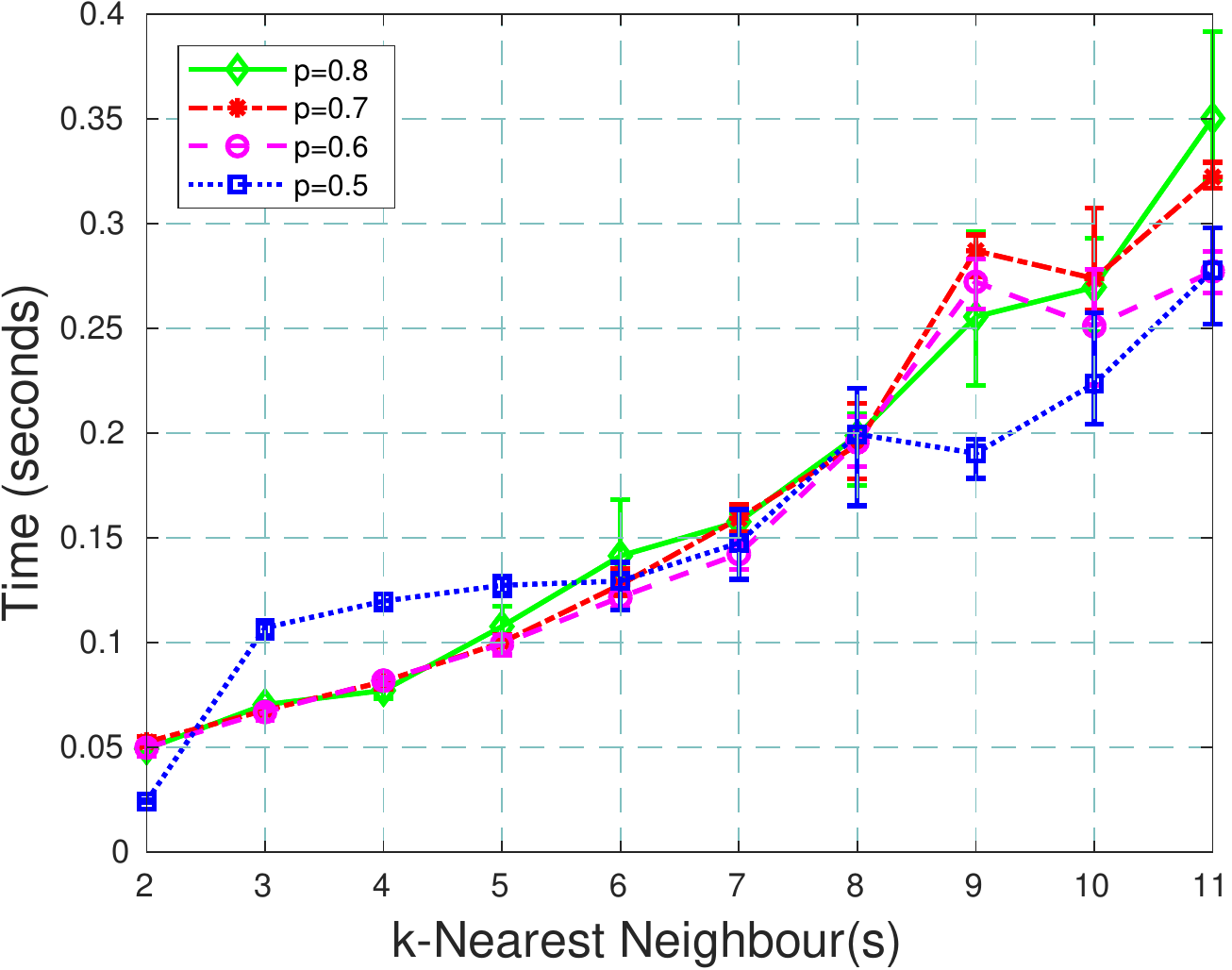}}%
    \qquad
    \subfigure[]{%
    \label{fig:second13}%
    \includegraphics[width=49mm,height=39mm]{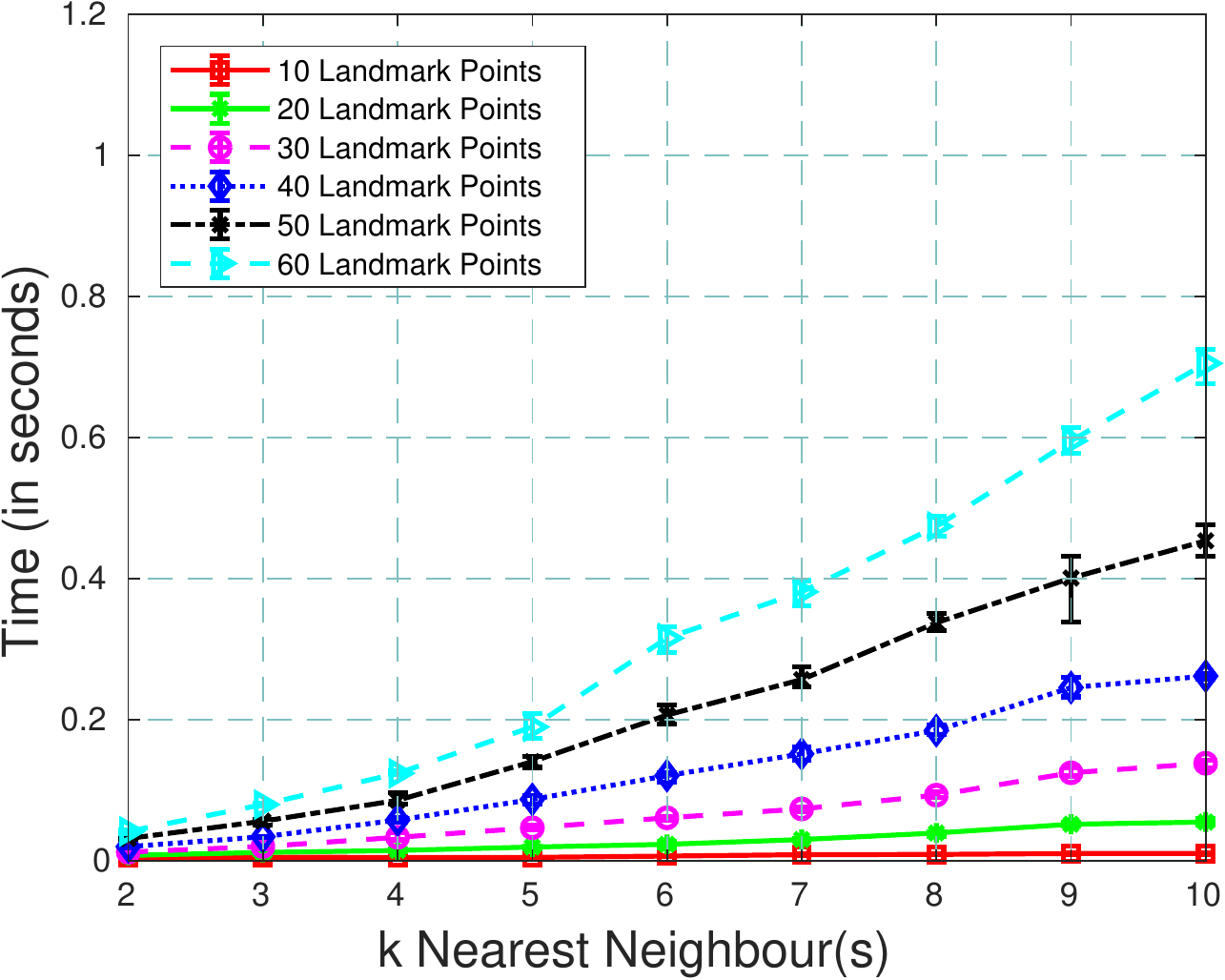}}%
    }
    \vspace{-5mm}
    \caption{(a) and (b) Error(\%) and runtime of matching two frames of \emph{House} with varying $p$ and $k$. Error decreases with $p$ and $k$, whereas computation time increases. (c) computation time of matching two images with varying $k$-NN for $n=10-60$ landmark points. $p$ is fixed as $0.6$ here for all the cases. }
    \label{fig:housekPlot}

    \end{figure*}

        \begin{figure*}
    \centering
    \makebox[\linewidth]{
    \subfigure[]{%
    \label{fig:noise1}%
    \includegraphics[width=50mm,height=40mm]{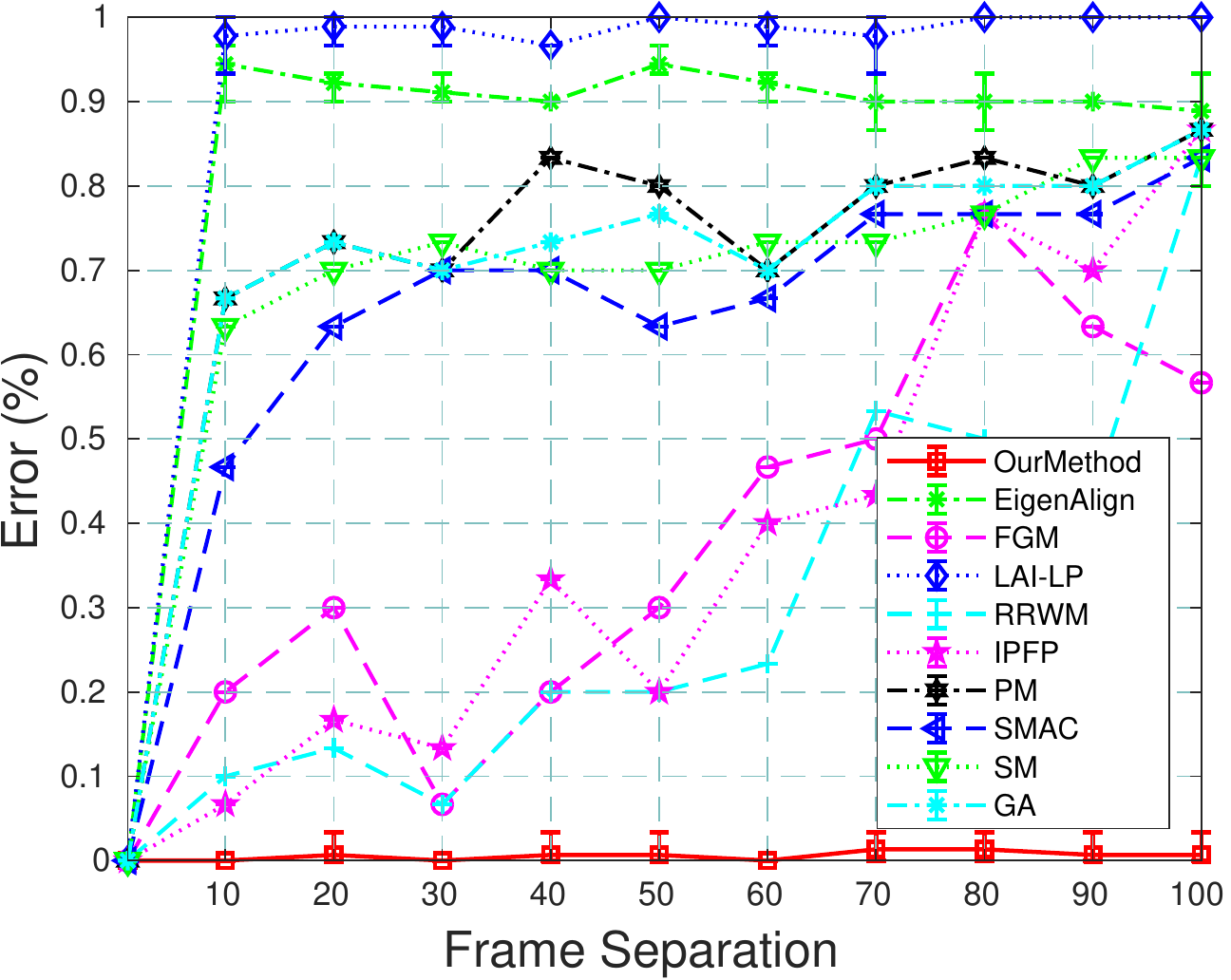}}%
    \qquad
    \subfigure[]{%
    \label{fig:noise2}%
    \includegraphics[width=50mm,height=40mm]{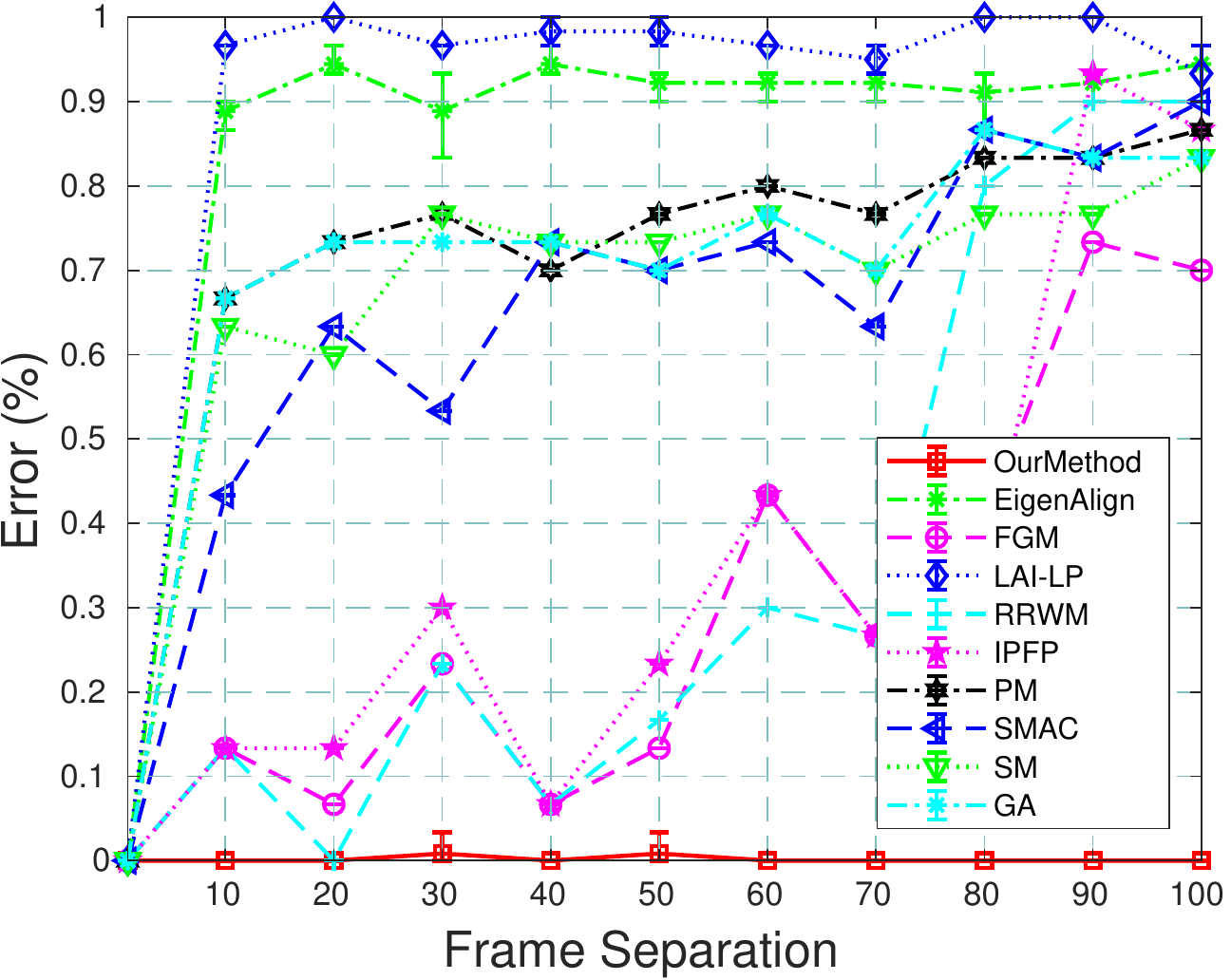}}%
    }
    \makebox[\linewidth]{
    \subfigure[]{%
    \label{fig:noise3}%
    \includegraphics[width=50mm,height=40mm]{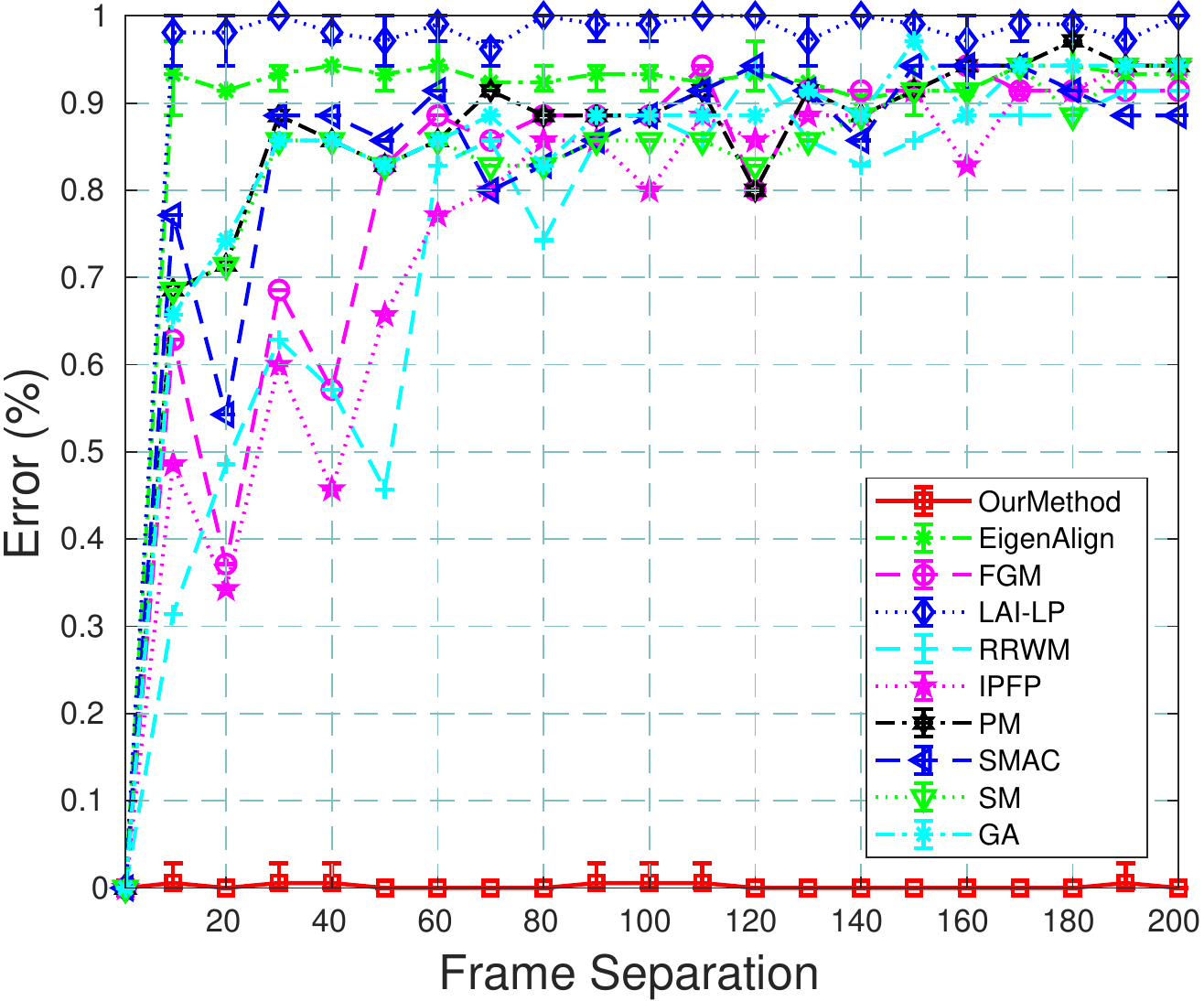}}%
        \qquad
    \subfigure[]{%
    \label{fig:noise4}%
    \includegraphics[width=50mm,height=40mm]{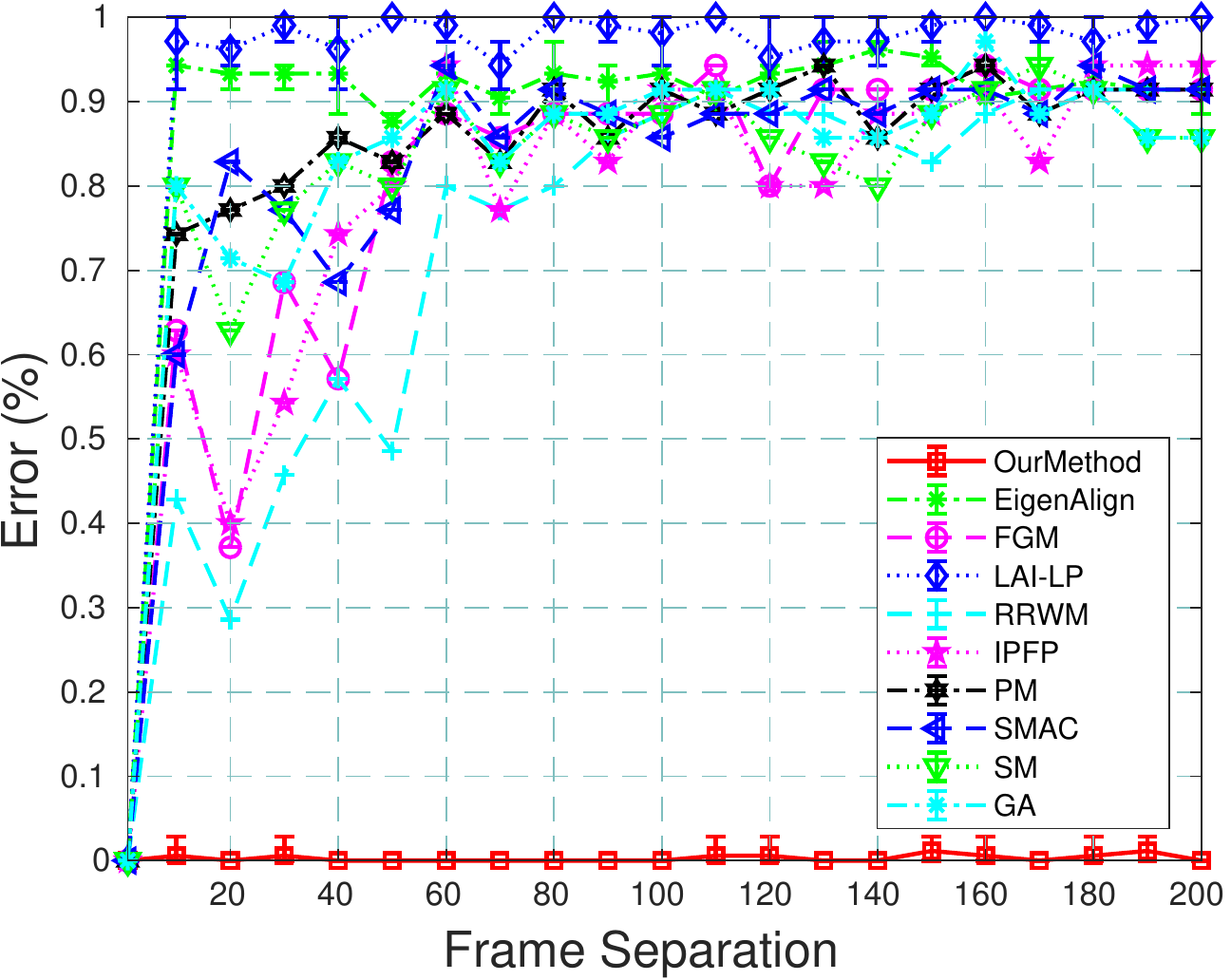}}%

    }

    \vspace{-5mm}
    \caption{Error (\%) in matching by various methods with different frame separation level for 
			Noise Model (a) I, (b) II for \emph{CMU Hotel} and (c) I, (d) II for \emph{Horse Shear}.}
    \label{fig:Noise}
    \end{figure*}
    
    \subsection{Noise Model}
  We analyze the performance of our method over other pairwise algorithms for two different noise models. We follow the noise model setup mentioned in ~\cite{feizi2016spectral}. We
  introduce noise in one random graph $G_1$ and generate a noisy version $\tilde{G}$ to be matched with $G_2$. $G_1$ is a random graph here which is created as $G_1(n,p)$ with $n$ nodes and $p$ probability. We describe two noise models as follows:

    \textbf{Noise Model I:}
    \begin{equation}
     \tilde{G} = G_1 \odot (1-A) + (1-G_1) \odot A
    \end{equation}
    $\tilde{G}$ is generated using the aforementioned equation where $A$ is a binary random symmetric matrix, whose entries are drawn from a \emph{Bernoulli distribution} as $A(n,q)$ with $n$ nodes and $q$ probability and $\odot$ represents the element-wise multiplication of matrices. This model flips the node-node adjacency of $G_1$ with probability $q$. 
    
    \textbf{Noise Model II:}
    \begin{equation}
     \tilde{G} = G_1 \odot (1-A) + (1-G_1) \odot B
    \end{equation}
    Again, $A$ and $B$ are binary random symmetric matrices, whose entries are drawn from the Bernoulli distribution as $A(n,q)$ and $B(n,r)$ with $n$ nodes and $q$ and $r$ probabilities, respectively. This model 
    flips node-node adjacency of $G_1$ with probability $q$, and in addition it also creates edges between non-connected nodes with probability $r$. 

    Results of noise model I and II on \emph{CMU Hotel} and \emph{Horse Shear} for frame separation level is shown in Figures~\ref{fig:noise1}, \ref{fig:noise3} and \ref{fig:noise2}, \ref{fig:noise4} respectively. We observe that our method is robust to noise for both the models as compared to other algorithms since there is a very small increase or no increase in error (\%) for all the cases.

\begin{figure*}
	\makebox[\linewidth]{
	\centering
	\vspace{-15mm}
	  \subfigure[]{%
	  \label{fig:houseMatch}%
	    \includegraphics[trim={2.5cm 0 3cm 0.35cm},clip,width=72mm,height=28mm]{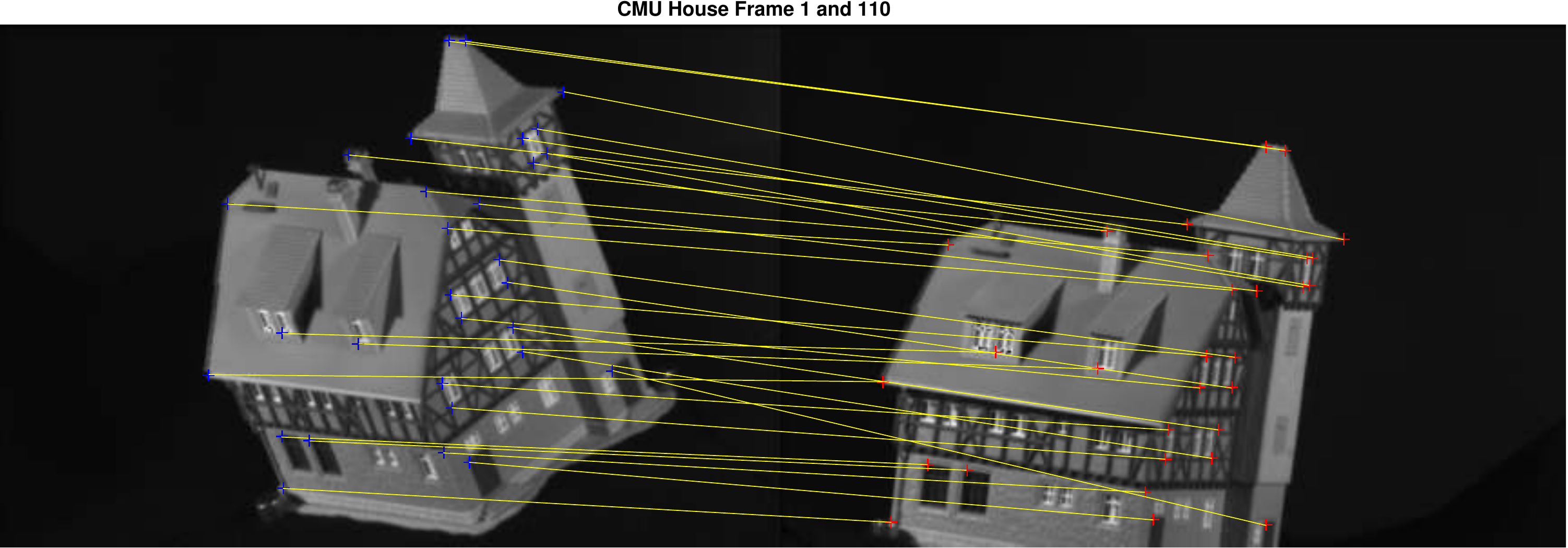}}%
	  \qquad
	  \subfigure[]{%
	  \label{fig:hotelMatch}%
	    \includegraphics[trim={2.1cm 0 0 0.35cm},clip,width=72mm,height=28mm]{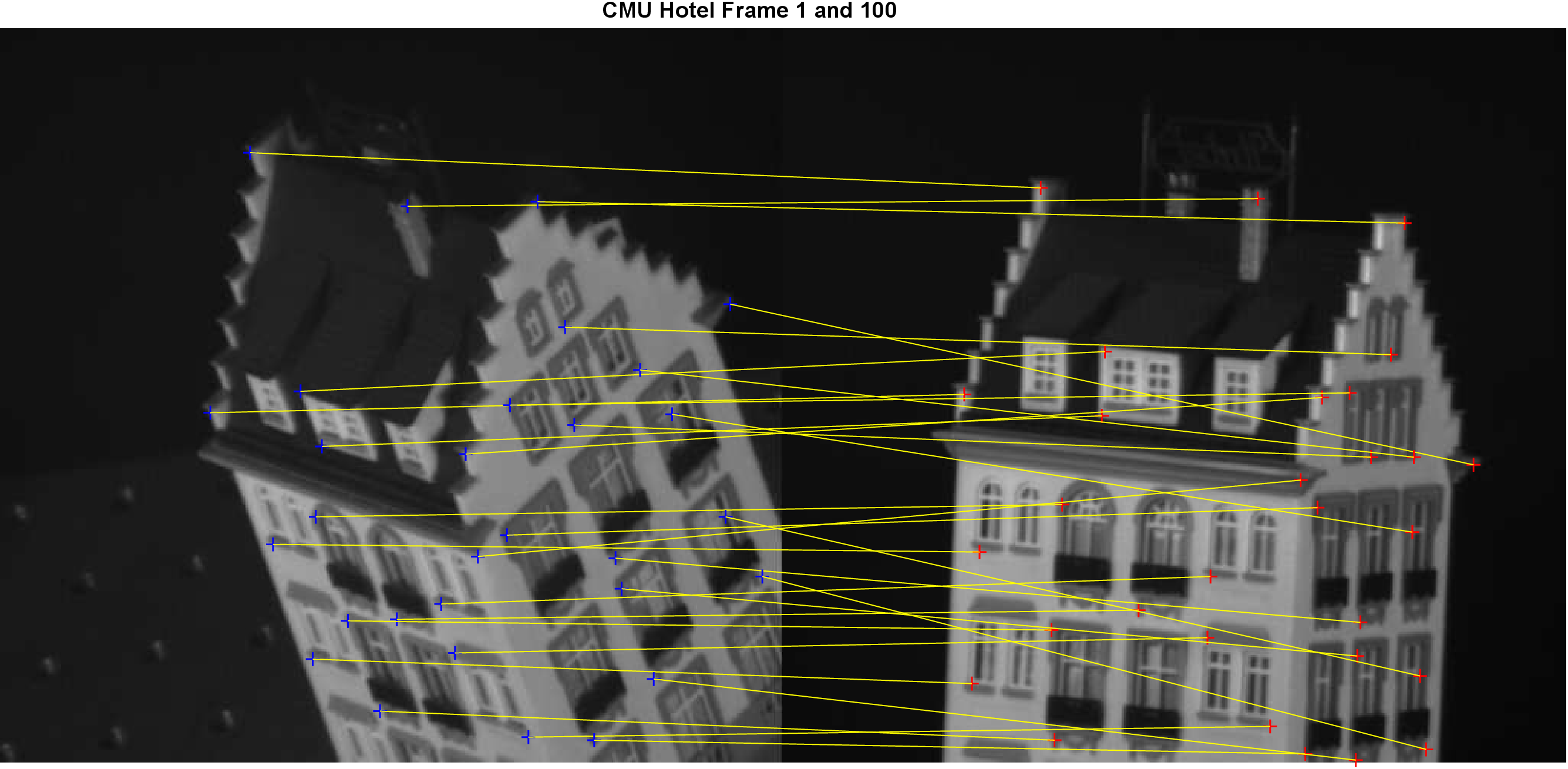}}%
	  }
	  \vspace{-5mm}
	  \qquad
	  \makebox[\linewidth]{
	  \hspace{-17mm}
	  \subfigure[]{%
	  \label{fig:HorseRotate}%
	    \includegraphics[trim={0 0 0 0.35cm},clip,width=72mm,height=28mm]{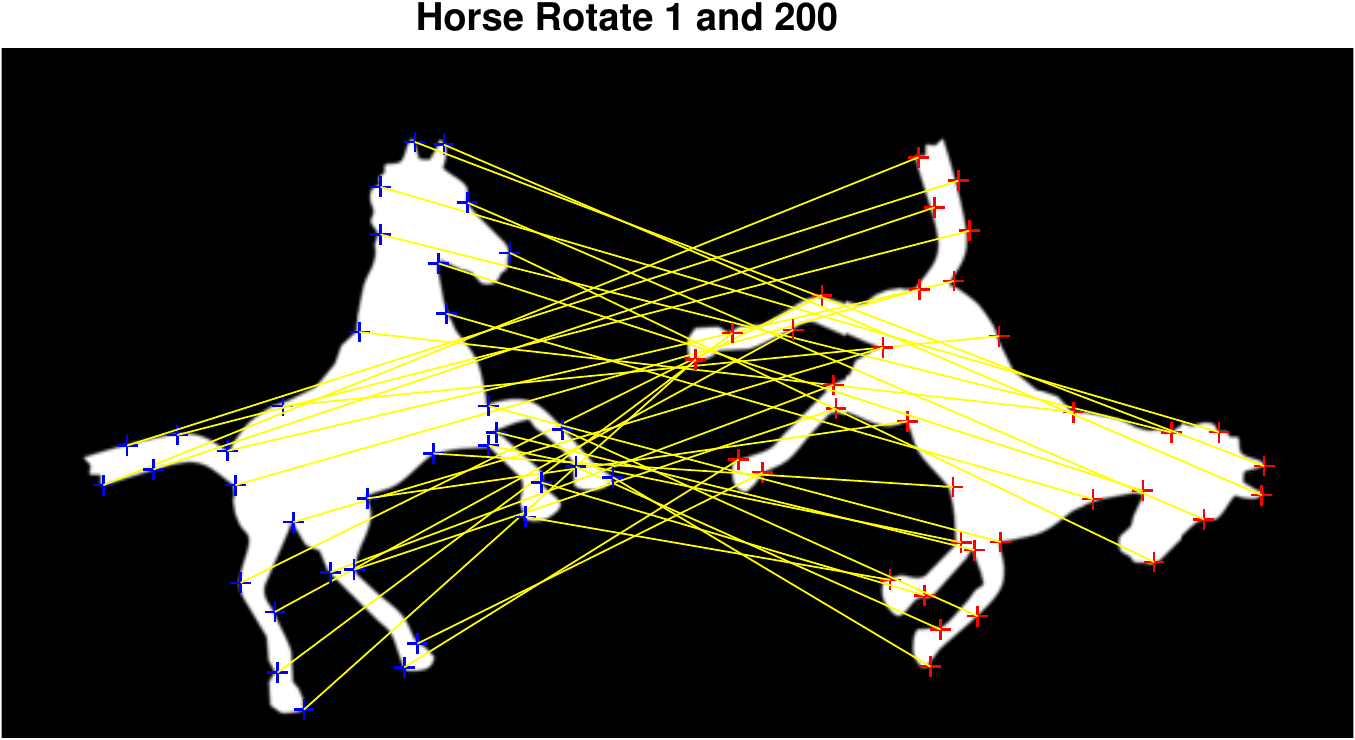}}%
	  \qquad
	  \subfigure[]{%
	  \label{fig:HorseShear}%
	    \includegraphics[trim={0 0 2.5cm 0.35cm},clip,width=72mm,height=28mm]{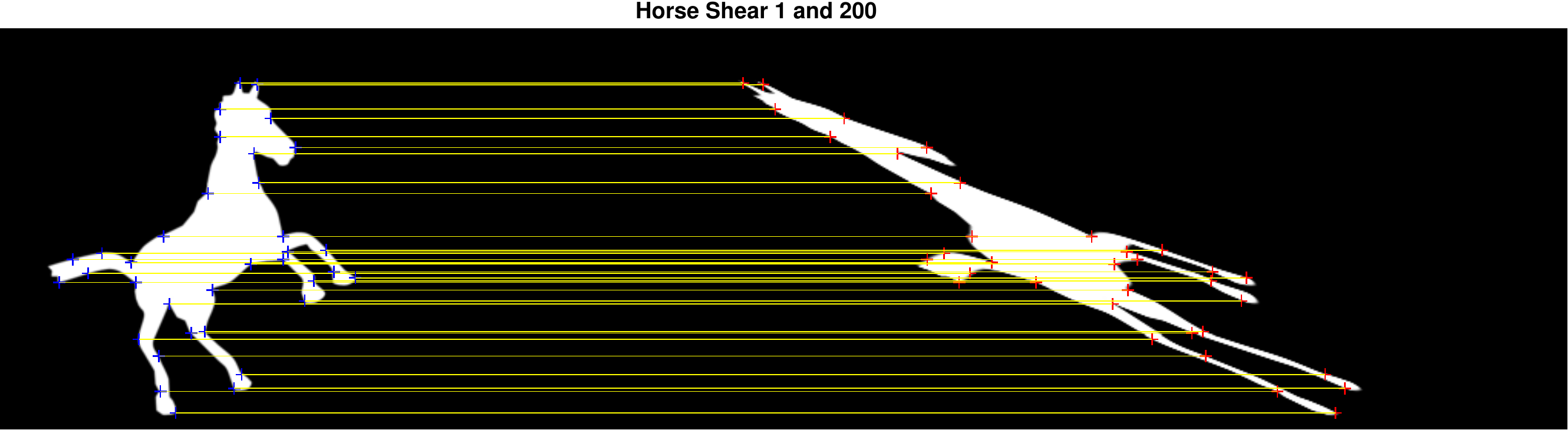}}%
	  }
	  \vspace{-5mm}
	  \qquad
	  \makebox[\linewidth]{
	  \hspace{-17mm}
	  \subfigure[]{%
	  \label{fig:Car}%
	    \includegraphics[width=72mm,height=28mm]{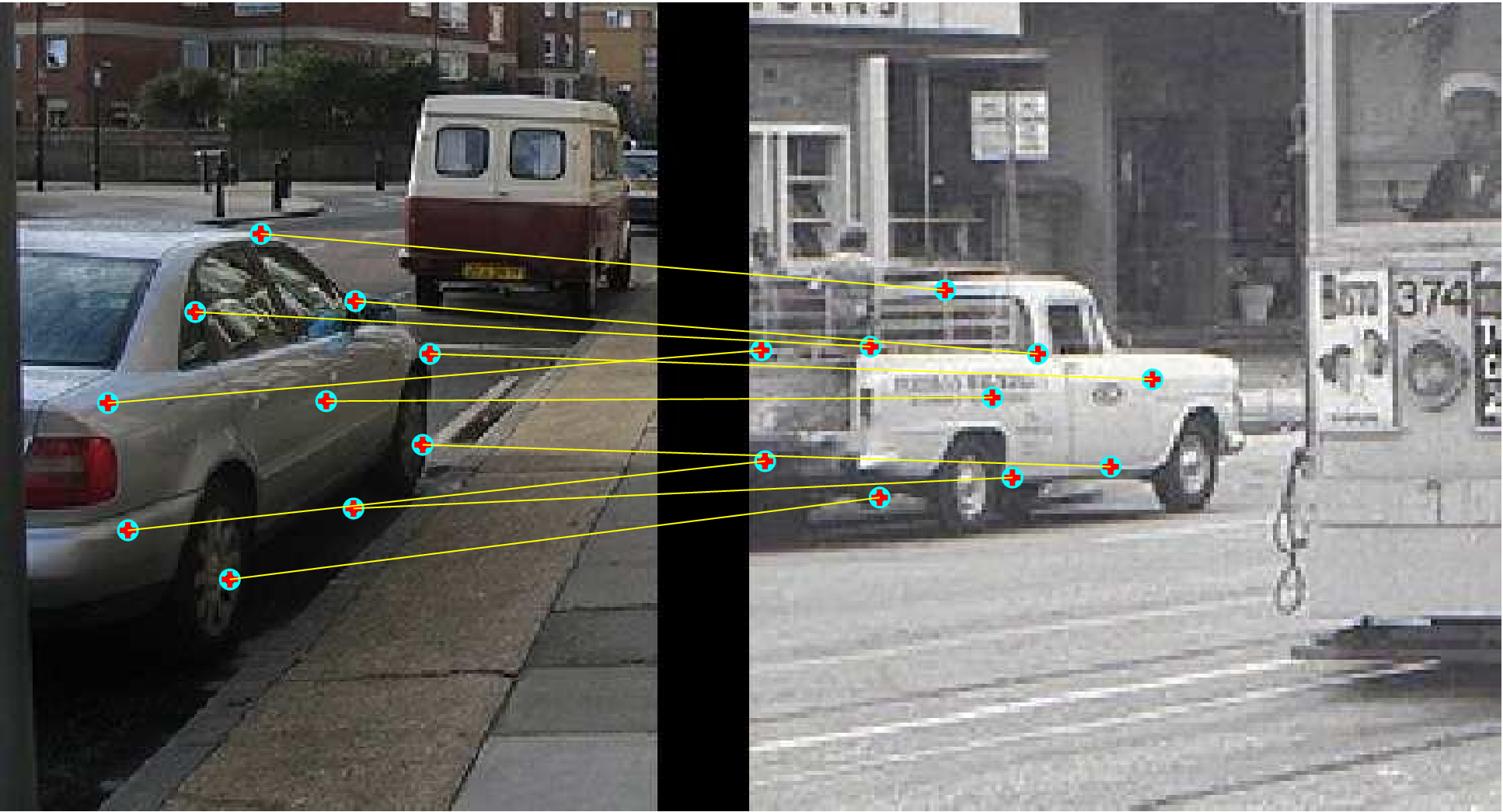}}%
	  \qquad
	  \subfigure[]{%
	  \label{fig:Bike}%
	    \includegraphics[width=72mm,height=28mm]{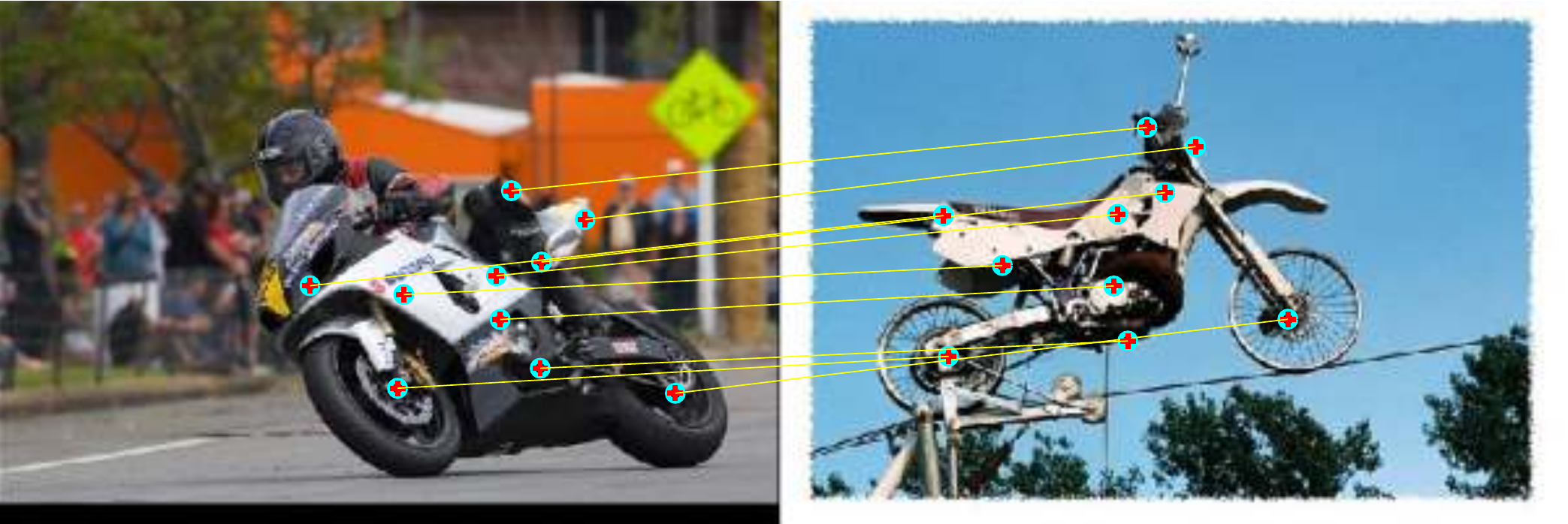}}%
	  }
	  \vspace{-4mm}
	  \qquad
	  \makebox[\linewidth]{
	  \hspace{-17mm}
	  \subfigure[]{%
	  \label{fig:Butterfly}%
	    \includegraphics[width=72mm,height=28mm]{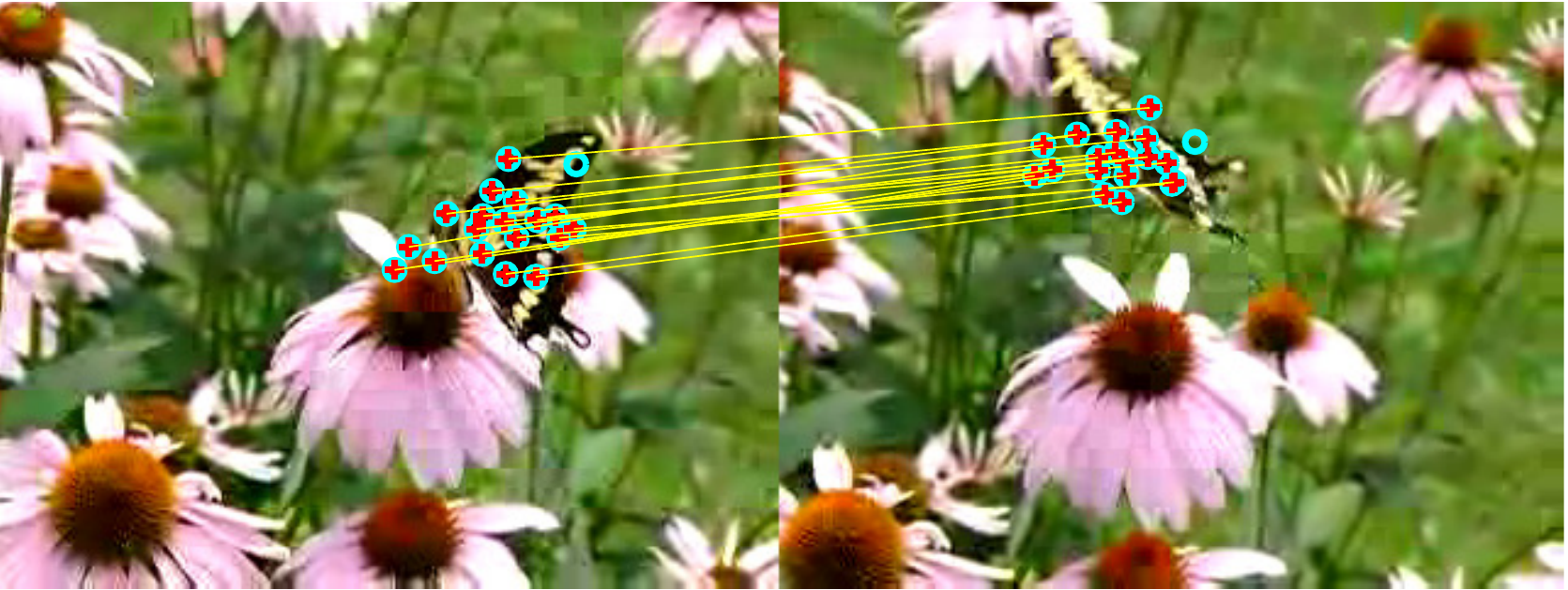}}%
	  \qquad
	  \subfigure[]{%
	  \label{fig:Spectrum}%
	    \includegraphics[width=72mm,height=28mm]{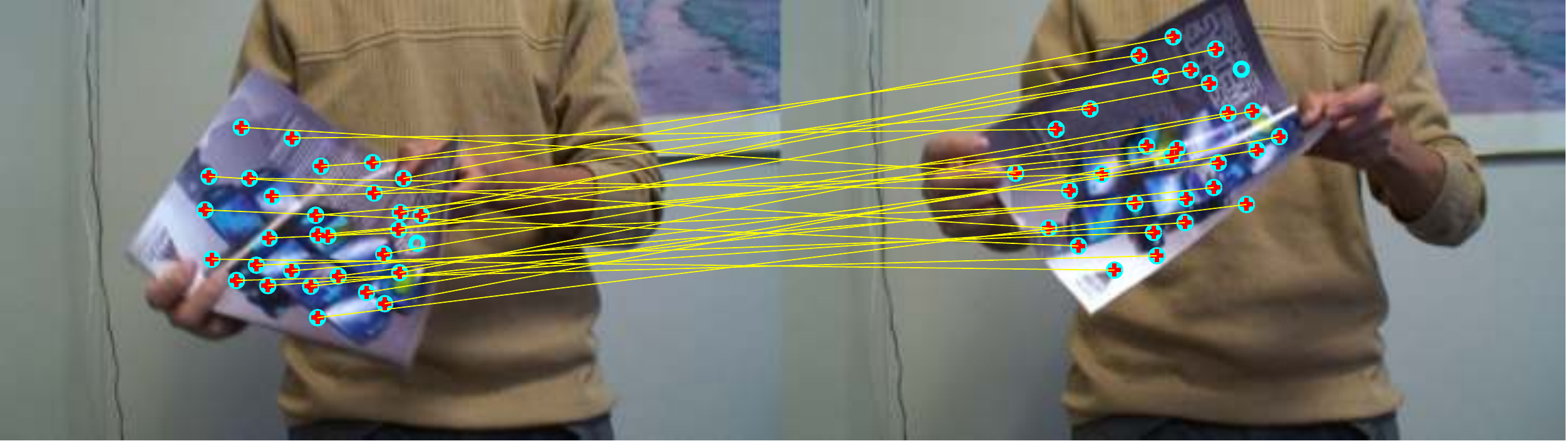}}%
	  }
	  \vspace{-5mm}
	  \qquad
	  \hspace{-9mm}
	  \makebox[\linewidth]{
	  \subfigure[]{%
	  \label{fig:Bldg}%
	    \includegraphics[trim={0 0 0 0.45cm},clip,width=72mm,height=28mm]{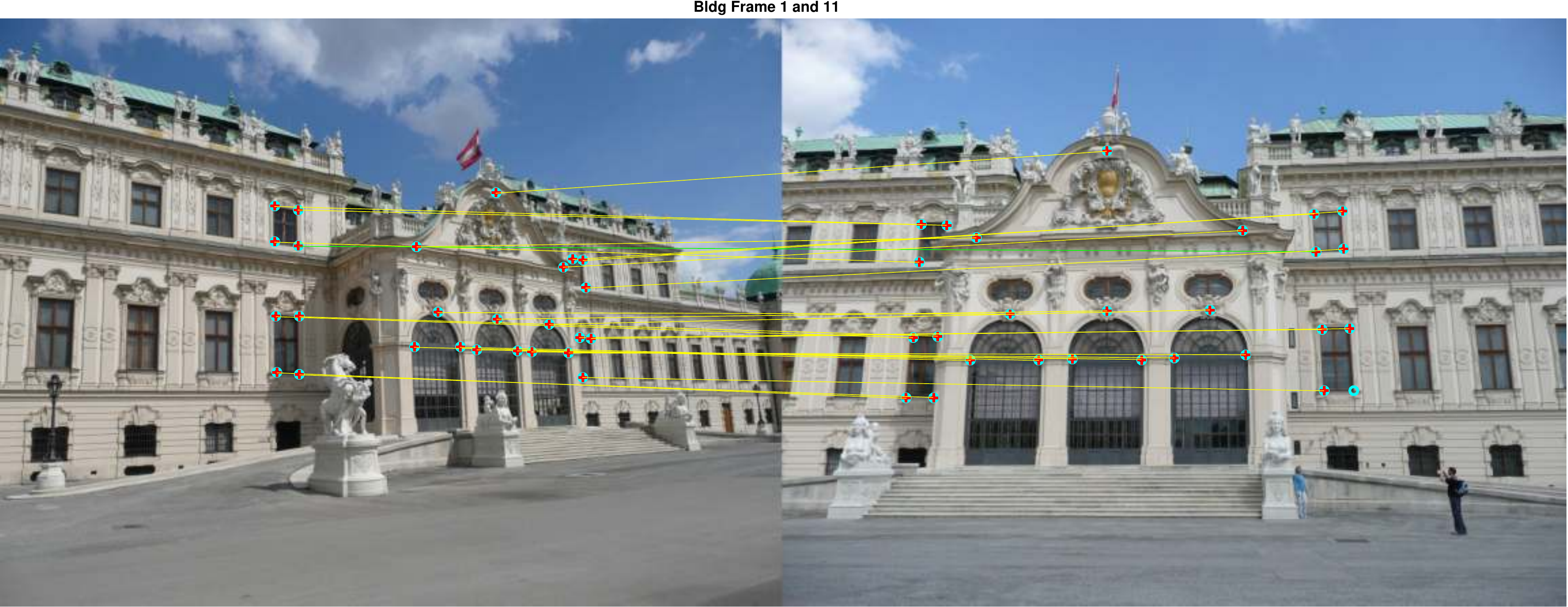}}%
	  \qquad
	  \subfigure[]{%
	  \label{fig:Book}%
	    \includegraphics[trim={0 0 0 0.4cm},clip,width=72mm,height=28mm]{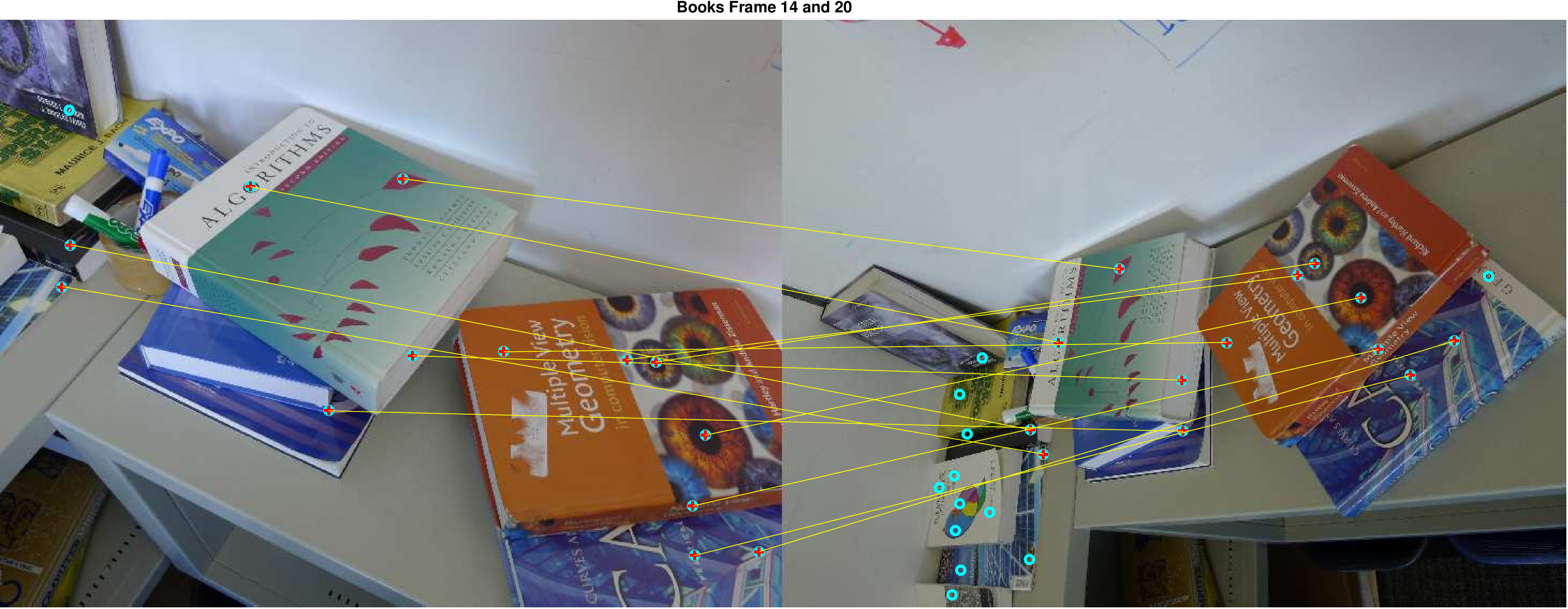}}%
	  }
	  \vspace{-2mm}
	  \caption{Instances of matchings in (a) House, (b) Hotel, (c) Horse Rotate, (d) Horse Shear frame sequences, (e) Car, (f) Bike, (g) Butterfly, (h) Magazine, (i) Building and (j) Books dataset. Yellow/green lines show correct/incorrect matches and isolated points show no matches.}
	  \label{fig:Matching}
	\end{figure*}


\end{document}